\documentclass[twoside,10pt]{article}

\usepackage{blindtext}

\usepackage[abbrvbib, preprint]{jmlr2e}

\usepackage{textcomp}%
\usepackage{manyfoot}%
\usepackage{booktabs}%
\usepackage{algorithm}%
\usepackage{algorithmicx}%
\usepackage{algpseudocode}%
\usepackage{listings}%

\usepackage{hyperref}
\usepackage{amsmath, tikz-cd}
\usepackage{amsfonts}
\usepackage{xcolor}         
\usepackage{mathrsfs}
\usepackage{commath}
\usepackage{accents}
\usepackage{geometry}
\usepackage{mathtools}
\usepackage[bb=dsserif]{mathalpha}
\usepackage{bm}
\usepackage[shortlabels]{enumitem}
\usepackage{breqn}
\usepackage{xparse}
\usepackage{dsfont}
\usepackage{bbm}
\usepackage{caption}
\usepackage{subcaption}
\usepackage{wrapfig}
\usepackage{float}
\usepackage{threeparttable}
\usepackage{booktabs}
\usepackage{multirow}
\usepackage{pifont}
\usepackage{cite}

\newtheorem{assumption}{Assumption}

\definecolor{norm}{HTML}{d9ead3}
\definecolor{local}{HTML}{cfe2f3}
\definecolor{plan}{HTML}{fff2cc}

\DeclarePairedDelimiter\floor{\lfloor}{\rfloor}

\usepackage{lastpage}

\begin{document}

\title{On Robust Cross Domain Alignment}

\author{\name Anish Chakrabarty \email \\
       \addr Statistics and Mathematics Unit\\
        Indian Statistical Institute, Kolkata
       \AND
       \name Arkaprabha Basu \email \\
       \addr Electronics and Communication Sciences Unit \\
        Indian Statistical Institute, Kolkata
        \AND
       \name Swagatam Das \email swagatam.das@isical.ac.in\\
       \addr Electronics and Communication Sciences Unit \\
        Indian Statistical Institute, Kolkata}

\editor{My editor}

\maketitle

\begin{abstract}
The Gromov-Wasserstein (GW) distance is an effective measure of alignment between distributions supported on distinct ambient spaces. Calculating essentially the mutual departure from isometry, it has found vast usage in domain translation and network analysis. It has long been shown to be vulnerable to contamination in the underlying measures. All efforts to introduce robustness in GW have been inspired by similar techniques in optimal transport (OT), which predominantly advocate partial mass transport or unbalancing. In contrast, the cross-domain alignment problem being fundamentally different from OT, demands specific solutions to tackle diverse applications and contamination regimes. Deriving from robust statistics, we discuss three contextually novel techniques to robustify GW and its variants. For each method, we explore metric properties and robustness guarantees along with their co-dependencies and individual relations with the GW distance. For a comprehensive view, we empirically validate their superior resilience to contamination under real machine learning tasks against state-of-the-art methods.

\end{abstract}


\section{Introduction}
Aligning unalike objects (images, networks, point clouds, etc.) based on their geometry remains the crux of machine learning challenges such as style transfer, graph correspondence, and shape matching. The first hint of a statistical measure of discrepancy between two such distinct distributions came in the form of Gromov-Wasserstein distance \citep{memoli2011gromov}, quickly finding continual application in data alignment \citep{demetci2022scot}, clustering \citep{chowdhury2021generalized, gong2022gromov}, and dimensionality reduction \citep{clark2024generalized}. Emerging as a $L^{p}$-relaxation of the Gromov-Hausdorff distance, it calculates the minimal distortion between replicates from distributions $\mu$ and $\nu$, themselves defined on spaces $\mathcal{X}$ and $\mathcal{Y}$ respectively. In other words,
\begin{equation*}
    \inf_{\pi} \norm{d_{X}(x,x') - d_{Y}(y,y')}_{L^{p}(\pi \otimes \pi)},
\end{equation*}
where $\pi$ denotes a coupling between distributions $(\mu,\nu)$ and the spaces are endowed with the respective metrics $d_{X}$ and $d_{Y}$\footnote{The metric space $(\mathcal{X},d_{X})$ coupled with the measure $\mu$ defines a \textit{metric measure} (mm) space.}, $p\geq 1$. Resembling the Kantorovich formulation in OT, it immediately inspires a Monge-like upper bound to the distance (namely, Gromov-Monge (GM)), given by
\begin{equation*}
    \inf_{\phi} \norm{d_{X}(x,x') - d_{Y}(\phi(x),\phi(x'))}_{L^{p}(\mu \otimes \mu)},
\end{equation*}
where the infimum is rather over measure preserving maps $\phi : \textrm{supp}(\mu) \rightarrow \textrm{supp}(\nu)$. In both cases, the underlying cost, measuring the extent of departure from strong isometry, differentiates the problem from mass transportation. It is rather the susceptibility to contamination that unites alignment and OT. The value of GW can be arbitrarily perturbed only by implanting an arbitrarily `outlying' observation. However, the defense against such outliers in the context of alignment turns out to be much more nuanced compared to OT (see, Section \ref{rob_GW_OT}). Its diverse applications, coupled with the objective of aligning geometries, demand unique solutions in different contexts. The very formulation of GW also hints towards several avenues to search for a robust formulation. Existing approaches to robustify GW and its progenies all draw insight from similar techniques in OT. While relaxing the optimization following \textit{partial} \citep{chapel2020partial} or \textit{unbalanced} \citep{sejourne2021unbalanced} OT fosters capable solutions, it is perhaps unfounded to expect them to serve every context. For example, relieving the set of feasible couplings from meeting the marginal constraints also takes away metric properties. Moreover, despite showing that image-to-image (I2I) translation architectures such as CycleGAN (Cycle-consistent Generative Adversarial Networks) are indeed special cases of GM-like distances \citep{zhang2022cycle}, current literature does not provide a pathway to accurate generation under contaminated source data. This work analyzes three principal means of robustifying the cross-domain alignment problem. This way, besides meeting diverse requirements of tasks related to alignment, we address the larger landscape of contamination. We refer the reader to Section \ref{rob_GW_OT} for a detailed discussion outline.

\textbf{Contributions: } The key takeaways of our study are as follows.  
\begin{itemize}
    \item The first method introduces penalization to large distortions while calculating GW in the spirit of Tukey and Huber. In context, it gives rise to relaxed GW distances that preserve topologies and usual metric properties (Proposition \ref{met}). We show that GW, under Tukey's penalization, becomes robust to Huber contamination (Proposition \ref{Prop_up}) and promotes resilience to underlying distributions (Corollary \ref{corrniet}). Provably, it extends the Robust OT (ROBOT) to distributions supported on distinct mm spaces. We provide algorithms to calculate the Tukey and Huber GW distances, which in applications such as shape matching under contamination exhibit superior performance compared to existing techniques. We also suggest data-dependent parameter tuning schemes that produce precise levels of robustness. 
    \item Offering a finer control over extreme pairwise distances from either space, the second method rather deploys relaxed metrics that preserve topology. The resultant \textit{locally} robust distance, surrogate to GW, becomes a lower bound to the first formulation (lemma \ref{ineqtukloc}). We show that solving the same boils down to calculating an OT between truncated observations from $\mu$ and $\nu$ (Proposition \ref{propIGW}). We also show that the notion can be generalized to define robust distances over probabilistic mm spaces. This eventually leads to a framework that offers denoising capability to Image translation models assuming a shared latent space.
    \item The third approach regularizes the optimization based on `clean' proxy distributions to achieve robust measure-preserving maps. We show its connection to robust OT formulations (lemma \ref{dual_3}) and the sample complexity such optimizations demand under contamination. The resultant optimization generalizes the notion of partial alignment, as plans corresponding to the latter can be shown to be an amenable candidate of ours. Based on the same, we propose RRGM, a novel image-to-image translation architecture that exhibits superior denoising capacity while generating hand-written digit images under contamination.
\end{itemize}

\section{Related Work}
Recovering unperturbed transport plans under contamination poses a significant challenge in cross-domain alignment. In most treatments based on GW (and Sturm's GW), the \textit{unbalanced} relaxation to the class of underlying couplings is used to ensure robustness \citep{sejourne2021unbalanced, de2022entropy}. As a result, the `denoised' solutions in both spaces become merely positive Radon measures. In case the marginal constraints are imposed using the TV norm (instead of Csisz\'{a}r or $\phi$-divergence), the idea boils down to transporting only a fraction of the mass under the distributions \citep{chapel2020partial, bai2024partialgromovwassersteinmetric}. UCOOT's \citep{tran2023unbalanced} robust formulation to deter Huber contamination utilizes a similar relaxation additionally on the feature spaces of the domains. While such a mass-trimming approach penalizes outliers, the resultant distance suffers significant deviations from its balanced counterpart \citep{nguyen2023unbalanced}. Moreover, the alignment problem, fundamentally different from mass transportation, raises more unanswered questions. For example, in most image-to-image translation problems, only one domain runs the risk of contamination. Unbalancing turns out ill-posed to handle such a semi-constrained robustification \citep{le2021robust}. Also, the landscape of contamination models stretches way beyond that of Huber's, which the current unbalanced techniques are solely equipped to deal with. The most recent technique offering robustness in GW alignment \citep{kong2024outlier} reinforces unbalancing, based on \textit{inlying} surrogate distributions over graphs. This essentially being an upper bound to UGW, carries all the aforementioned issues. As such, a detailed exploration of robust alignment between distinct domains subjective of diverse underlying tasks remains overdue.     

\textbf{Notations: }Given a Polish space $\mathcal{X}$, we denote by $\mathcal{P}(\mathcal{X})$ and $\mathcal{M}(\mathcal{X})$ the set of Borel probability measures and signed Radon measures defined on it respectively. For $p \in [1, \infty)$, measures $\rho \in \mathcal{P}(\mathcal{X})$ with finite $p$-th absolute moment, $M_{p}(\rho) \coloneqq \int \norm{x}^{p} \rho(dx) < \infty$ form the space $\mathcal{P}_{p}(\mathcal{X})$. The Total Variation (TV) norm of $\rho \in \mathcal{M}(\mathcal{X})$ is denoted as $\norm{\rho}_{\textrm{TV}} \coloneqq  \frac{1}{2}|\rho| (\mathcal{X)}$. The space of measurable functions $f: \mathcal{X} \rightarrow \mathbb{R}$ satisfying $\norm{f}_{L^{p}(\rho)} \coloneqq (\int |f|^{p} d\rho)^{1/p} < \infty$ is denoted by $L^{p}(\rho)$. The pushforward of $\rho \in \mathcal{P}(\mathcal{X})$ by a measurable map $f$ is defined as $f_{\#}\mu = \mu(f^{-1})$. We define the uniform norm as $\norm{f}_{\infty} \coloneqq \sup_{x \in \mathcal{X}} |f(x)|$. The notation $\circledcirc$ denotes the tensor-matrix multiplication, whereas $\odot$ and $\oslash$ signify element-wise product and division in matrices respectively. The notation used for the Frobenius norm is $\norm{\cdot}_{\textrm{F}}$. Given $a,b \in \mathbb{R}$, we write $a \vee b = \max\{a,b\}$ and $a \wedge b = \min\{a,b\}$. The uniform $\varepsilon$-covering number of a class of functions $\mathcal{F}$, based on $n$ points $\{x_{i}\}^{n}_{i=1}$, with respect to (w.r.t.) the metric $d(f,f') \coloneqq \max_{i \in [n]}|f(x_{i}) - f'(x_{i})|$ is denoted as $N_{\infty}(\varepsilon, \mathcal{F},n)$. We also write inequalities, suppressing absolute constants, as $\lesssim$ and $\gtrsim$. In case $a \lesssim b$, we equivalently write $a = O(b)$. Given that the previous relation holds for a logarithmic function of $b$, we write $a = \Tilde{O}(b)$. If there exists an (strong) isometry between the spaces $\mathcal{X}$ and $\mathcal{Y}$, we write $\mathcal{X} \cong \mathcal{Y}$.  

\section{Preliminaries}
Before introducing our robust formulations, we review the basics of transportation and alignment between metric measure spaces. 

\textbf{Optimal Transport and Entropic Regularization: } Given a Polish space $\mathcal{X}$ endowed with a metric $d(\cdot, \cdot)$, the OT problem between $\mu, \nu \in \mathcal{P}(\mathcal{X})$ is defined as 
\begin{equation} \label{OT}
    \textrm{OT}_{c}(\mu, \nu) \coloneqq \inf_{\pi \in \Pi(\mu, \nu)} \int_{\mathcal{X} \times \mathcal{X}} c(x,y) d\pi(x,y),
\end{equation}
where $c : \mathcal{X} \times \mathcal{X} \rightarrow \mathbb{R_{+}}$ is the lower semi-continuous transportation cost and $\Pi(\mu, \nu) = \{\pi \in \mathcal{P}(\mathcal{X} \times \mathcal{X}) : \pi(\cdot \times \mathcal{X}) = \mu, \pi(\mathcal{X} \times \cdot) = \nu\}$ is the set of couplings between $\mu$ and $\nu$. (\ref{OT}) is the Kantorovich formulation and can be shown to possess a minimizer. Given $c(x,y) = d(x,y)^{p}$, $p \geq 1$ it defines the $p$-Wasserstein metric $W_{p}(\mu, \nu) \coloneqq [\textrm{OT}_{c}(\mu, \nu)]^{1/p}$ on the space $\mathcal{P}_{p}(\mathcal{X})$ and metrizes weak convergence \citep{villani2009optimal}. OT is a typical linear program and it is an entropic regularization that makes it strictly convex. Given a parameter $\varepsilon > 0$, reinforcing the marginal constraint under the Kullback-Leibler (KL) divergence yields the primal Entropic OT (EOT) problem:
\begin{equation} \label{EOT}
    \textrm{EOT}^{\varepsilon}_{c}(\mu, \nu) \coloneqq \textrm{OT}_{c}(\mu, \nu) + \varepsilon d_{\textrm{KL}}(\pi | \mu \otimes \nu).
\end{equation}
Unlike its unregularized counterpart, the convergence rate corresponding to the empirical EOT cost (towards the population limit) becomes devoid of $\textrm{dim}(\mathcal{X})$\citep{mena2019statistical}. Entropic regularization also enables computing $\delta$-approximate estimates of the transport cost in $\Tilde{O}(n^{2}/\delta)$ time \citep{blanchet2024towards}. Despite computational and theoretical prowess, observe that the EOT cost (also OT) and corresponding potentials can be arbitrarily perturbed if either $\mu$ or $\nu$ (or both) is perturbed the slightest in TV.

\textbf{The Gromov-Wasserstein distance: } As mentioned before, we call the triplet $(\mathcal{X}, d, \mu)$ a metric measure space, where $\mu$ has full support, i.e. $\textrm{supp}(\mu) = \mathcal{X}$. While it is technically convenient to define GW as an extension of OT between two distinct mm spaces, we differentiate them based on their origins in mass transportation and object alignment. Given two Polish mm spaces $(\mathcal{X}, d_{X}, \mu)$ and $(\mathcal{Y}, d_{Y}, \nu)$, the GW distance in all its generality is defined as 
\begin{equation} \label{GW}
     d_{\textrm{GW}}(\mu, \nu) \coloneqq \Big(\inf_{\pi \in \Pi(\mu, \nu)} \int_{\mathcal{X} \times \mathcal{Y}} \int_{\mathcal{X} \times \mathcal{Y}} [\Lambda(d_{X}(x,x'),d_{Y}(y,y'))]^{p} d \pi \otimes \pi (x,y,x',y')\Big)^{\frac{1}{p}},
\end{equation}
where $\Lambda: \mathbb{R}_{+} \times \mathbb{R}_{+} \rightarrow \mathbb{R}_{+}$ is a pseudometric measuring the extent of distortion, $1 \leq p < \infty$. We also sparingly write $d_{\textrm{GW}}(X,Y)$. Observe that (\ref{GW}) is essentially the $L^{p}$-relaxation of the Gromov-Hausdorff distance (\citet{memoli2011gromov}, Section 4.1), a similar operation to what leads to the Kantorovich-Rubinstein formulation in OT ($W_{p}$). Now, considering $\Lambda = \Lambda_{q}(a,b) \coloneqq \frac{1}{2}\abs{a^{q}-b^{q}}^{1/q}$, $q < \infty$ one can recover the $(p,q)$-GW distance \citep{arya2024gromov}, which induces a metric over the class of strongly isomorphic\footnote{$(\mathcal{X}, d_{X}, \mu)$ and $(\mathcal{Y}, d_{Y}, \nu)$ are said to be \textit{strongly isomorphic} if there exists a measure preserving isometry $\phi : \mathcal{X} \rightarrow \mathcal{Y}$ (i.e. $\phi_{\#}\mu = \nu$ and $d_{Y}(\phi(x),\phi(x')) = d_{X}(x,x')$) which is also a bijection.} mm spaces with finite $p$-diameter, i.e. $\int_{\mathcal{X} \times \mathcal{X}} [d_{X}(x,x')]^{p} \mu_{X}(dx) \mu_{X}(dx') < \infty$. Different choices of $d_{X}, d_{Y}$ also lead to interesting variants of the GW distance, e.g. considering $d_{X} = \langle \cdot, \cdot \rangle$ (with $p=2, q=1$) and $\norm{\cdot -\cdot}$ (with $p=4, q=2$) makes the corresponding distances invariant to orthogonal transformations and translations respectively. \citet{bauer2024z} proposes $\mathcal{Z}$-GW distances by further generalizing $d_{X}: \mathcal{X} \times \mathcal{X} \rightarrow \mathcal{Z}$ (also $d_{Y}$) as network kernels, given any complete and separable metric space $\mathcal{Z}$. Despite enjoying structural maneuverability, unlike OT, GW distances pose a quadratic assignment problem (QAP) and are in general NP-hard to compute. While it is still feasible to determine the exact value of $(4,2)$-GW between spheres \citep{arya2024gromov}, given samples from arbitrary distributions one must resort to entropic regularization to ensure computational tractability \citep{scetbon2022linear, rioux2023entropic}. Following the setup in (\ref{GW}), the Entropic GW (EGW) distance is defined as
\begin{equation} \label{EGW}
    \textrm{EGW}^{\varepsilon}(\mu, \nu) \coloneqq d_{\textrm{GW}}(\mu, \nu) + \varepsilon d_{\textrm{KL}}(\pi | \mu \otimes \nu).
\end{equation}
This becomes particularly useful in case both the mm spaces are Euclidean with $(\mu, \nu) \in \mathcal{P}_{4}(\mathcal{X}) \times \mathcal{P}_{4}(\mathcal{Y})$, as it ties the underlying $(4,2)$-EGW\footnote{Essentially the square root of the $(4,2)$-EGW distance. A more convenient way of realizing it is to assume $\Lambda_{q}(a,b) \coloneqq \frac{1}{2}\abs{a^{q}-b^{q}}$ instead, under which the parameters become $p=2,q=2$ \citep{rioux2023entropic}.} optimization to EOT with an altered cost. However, the issue regarding uncontrolled perturbation under contamination still persists. 

\section{Robustifying Gromov-Wasserstein} \label{rob_GW_OT}
Formulating a mechanism that forestalls the effects of contamination in GW is more elusive compared to OT. Firstly, there is the context of the underlying optimization itself. In OT, the treatment ensuring robustness differs based on the task at hand. For example, cases that prioritize a divergence (e.g. generative models requiring a robust loss) usually call for a robust surrogate to $W_{p}$ only. As a result, relaxations such as unbalancing or mass truncation (equivalently, addition) are often appropriate \citep{nietert2022outlier, nietert2023robust}. The goal in such cases lies mainly to recover $W_{p}(\mu, \nu)$ based on a robust proxy $W^{\epsilon}_{p}(\hat{\mu}_{n}, \hat{\nu}_{n})$, i.e. $|{W_{p}(\mu, \nu) - W^{\epsilon}_{p}(\hat{\mu}_{n}, \hat{\nu}_{n})}| \rightarrow 0$ in probability, where $\epsilon > 0$ denotes the radius of robustness. This can also be achieved by defining a margin on the extent of allowable perturbation while choosing the surrogate \citep{raghvendra2024robpar}. While such formulations preserve sample complexity, the resultant transport plans ($\pi^{*}_{\epsilon}$) do not carry robust marginals that are also necessarily probability distributions. This becomes crucial when one is also interested in finding a robust measure preserving map ($T_{\epsilon} : \textrm{supp}(\mu) \rightarrow \textrm{supp}(\nu)$) between the two distributions in the sense of Monge. A surrogate loss ignoring the marginal constraints is bound to result in a map whose deviation from the oracle ($T^{*}$) has a non-vanishing lower bound (i.e. there exists $\tau_{\epsilon}>0$ such that $\norm{T_{\epsilon} - T^{*}} \gtrsim \tau_{\epsilon}$). In this regard, \citet{balaji2020robust, le2021robust} (ROT) maintains a balanced transport by optimizing over proxy distributions instead. While statistical properties of the resulting plans remain unexamined, KL-enforced regularization makes them tractable with comparable efficacy ($\Tilde{O}(n^{2}/\delta)$, where $\delta >0$ is the error margin and $n$ is the sample size.). 

Due to its role in alignment (e.g. shape matching) and the involvement of two distinct mm spaces, one needs to be more cautious in approaching the GW problem using similar techniques. Observe that, $d_{\textrm{GW}}(\mu, \nu)$ calculates the optimal $p$-distortion of a coupling between $\mu$ and $\nu$ (i.e. ${||\Lambda(d_{X}(x,x'),d_{Y}(y,y'))||}_{L^{p}( \pi \otimes \pi)}$). As such, it may become \textit{extremely fragile} (\citet{blumberg2014robust}, Proposition 4.3) and sustain uncontrolled fluctuation if a single observation from either space is perturbed heavily. Unbalancing readily limits mass allocation to such outliers, resulting in a robust surrogate to $d_{\textrm{GW}}$. However, unlike OT, it also risks sacrificing geometric information contained solely in pairwise distances. This creates significant misalignment between the resultant plan and the isometric benchmark. The later technique of penalizing the distributions ($\mu$ and $\nu$) themselves based on robust proxies also needs additional consideration. For example, when only one of them is contaminated (semi-constrained), the focus must lie on robustifying the pairwise distances, which is not the same as making the law robust. The problem is further confounded if near-isometric robust Monge maps \citep{dumont2024existence}(bidirectional in case of Reverse Gromov-Monge \citep{hur2024reversible}) are sought.


\begin{figure}[ht]
  \begin{minipage}[c]{0.6\textwidth}
    \includegraphics[width=\textwidth]{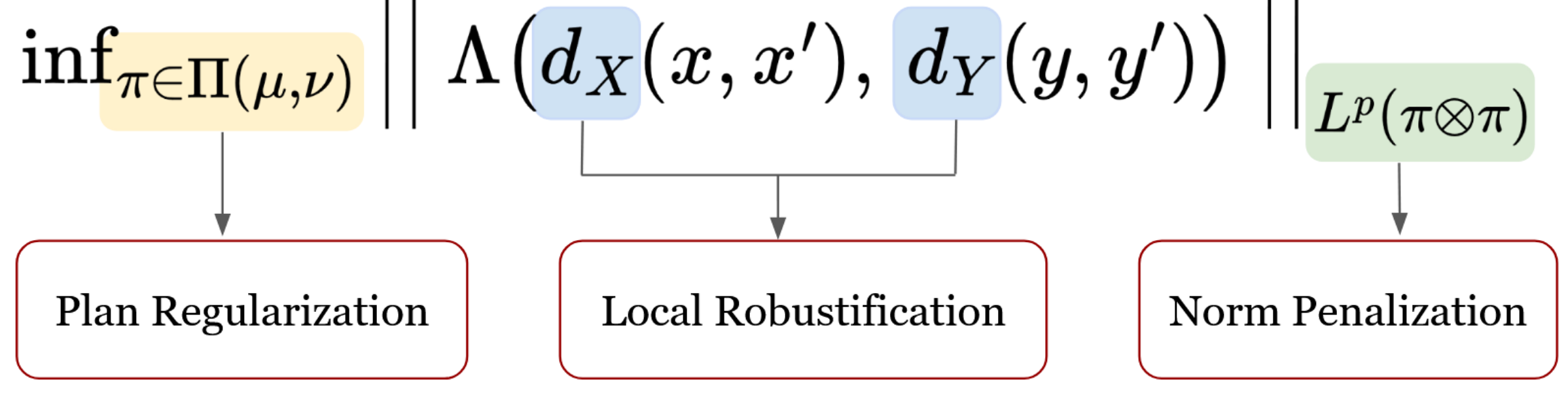}
  \end{minipage}\hfill
  \begin{minipage}[c]{0.39\textwidth}
    \caption{Three disjoint approaches leading to outlier-robustness of different degrees in Gromov-Wasserstein formulations. The forthcoming discussion follows the course: Section \ref{Part1} (\textcolor{norm}{$\blacksquare$}), Section \ref{Part2} (\textcolor{local}{$\blacksquare$}), and Section \ref{Part3} (\textcolor{plan}{$\blacksquare$}).} 
    \label{Fig:GW_rob}
  \end{minipage}
\end{figure}

On top of these existing ambiguities, the participating mm spaces might suffer different types of contamination with varying intensity. As such, the spectrum of contamination models (Wasserstein, Huber's $\epsilon$-contamination, etc.) needs to be kept in mind. Based on the varied demands of cross-domain alignment problems, we identify \textit{three} solutions to the robustification problem in GW. The \textit{primary} and most immediate way is to arrest the extreme distortions, given pairwise observations $(x,y), (x',y')$. In the process, we introduce Tukey's and Huber's relaxation into the GW metric (Section \ref{Part1}). The \textit{second} solution stems from robust surrogates to the metrics $d_{X}, d_{Y}$ that limit fluctuations at their nascency while calculating pairwise distances (Section \ref{Part2}). Based on the nature of the mm spaces, this approach may propose either structural or optimization-based robustness. In search of robust translation maps, the \textit{third} method advocates relaxing the optimization itself by regularizing the set of plans $\Pi(\mu,\nu)$ (Section \ref{Part3}). Intriguing relations to robust OT formulations and duality emerge from this approach.

\subsection{Norm Penalization: Towards Huber's Gromov-Wasserstein} \label{Part1}
The most recognized contamination model in robust statistics \citep{huber1964robust} assumes the existence of an arbitrary distribution $\mu_{c} \in \mathcal{P}(\mathcal{X})$ from which outliers originate. Under the same, it is equivalent to tossing a coin with $\epsilon \in (0,1)$ probability in favor of $\mu_{c}$ during each independent draw from $\mu$. However, given a clear separation between the two distributions, it implies vastly outlying observations in the sample (see Figure \ref{fig:cat_heart_data}(c)). Now, let us recall the definition (\ref{GW}),
\begin{equation*}
    d_{\textrm{GW}}(\mu, \nu) = \inf_{\pi \in \Pi(\mu, \nu)} \frac{1}{2} \norm{\Lambda_{X,Y}}_{L^{p}(\pi \otimes \pi)},
\end{equation*}
where $\Lambda_{X,Y} = 2\Lambda_{1}(d_{X},d_{Y})$ in particular. In a sample problem, the $l_{p}$ norm calculating the departure form isometry is sensitive to unusually large (or small) observations. We employ relaxed $l_{p}$ norms to curb the effects of such outliers in distortion. Tukey's relaxation \citep{clarkson2019dimensionality} is the most intuitive in this regard. 
\begin{definition}[Tukey loss function]
    Given a threshold $\tau \geq 0$, the $p$-Tukey loss function for $p \in [1, \infty)$ is defined as 
    \begin{equation*}
        \mathcal{T}_{p}(x) \coloneqq \begin{cases}
                \abs{x}^{p} & \textrm{if }\abs{x} \leq \tau\\
                {\tau}^{p} & \textrm{otherwise}.
            \end{cases}
    \end{equation*}
\end{definition}
Observe that, it is polynomially bounded above\footnote{An increasing function $f: \mathbb{R}_{+} \rightarrow \mathbb{R}_{+}$ is said to be polynomially bounded above with degree $p$ if $\frac{f(b)}{f(a)} \lesssim \big(\frac{b}{a}\big)^{p}$.} with degree $p$. It also induces the corresponding `norm', $\norm{f}_{\mathcal{T}_{p}(\mu)} = \big(\int_{\mathcal{X}} \mathcal{T}_{p}(f(x)) d\mu(x)\big)^{1/p}$, given $f \in L^{p}(\mu)$. Though dependant on the parameter $\tau$, it enables one to define
\begin{equation} \label{TGW}
    \norm{\Lambda_{X,Y}}_{\mathcal{T}_{p}(\pi \otimes \pi)} \coloneqq \Big(\int_{\mathcal{X} \times \mathcal{Y}} \int_{\mathcal{X} \times \mathcal{Y}} \mathcal{T}_{p}\abs{d_{X}(x,x') - d_{Y}(y,y')} d \pi \otimes \pi (x,y,x',y')\Big)^{\frac{1}{p}}.
\end{equation}
We call the quantity $d_{\textrm{TGW}}(\mu, \nu) \coloneqq \inf_{\pi \in \Pi(\mu, \nu)} \frac{1}{2}\norm{\Lambda_{X,Y}}_{\mathcal{T}_{p}(\pi \otimes \pi)}$, \textit{Tukey's GW} (TGW, or specifically $(p,\tau)$-TGW). Clearly, it is a lower bound to the corresponding $d_{\textrm{GW}}(\mu, \nu)$ and $d_{\textrm{TGW}} \nearrow d_{\textrm{GW}}$ as $\tau \nearrow \infty$. It also carries some major properties of the original GW distance (\citet{memoli2011gromov}, Theorem 5.1). The non-negativity and symmetry are obvious and given $\mathcal{X} \cong \mathcal{Y}$, it becomes $0$. Conversely, given a threshold $\tau > 0$, $d_{\textrm{TGW}}(\mu, \nu)=0$ implies that the mm spaces $\mathcal{X}$ and $\mathcal{Y}$ are isometric, for $p \in [1,\infty)$. Here, we only define the loss for $p < \infty$ since given $\tau > 1$, the essential norm at $p = \infty$ spoils the thresholding and eventually the robustness. Our goal also lies in avoiding large deviations between TGW from its perturbed GW benchmark, caused solely due to large thresholds. As such, it is not in our interest to check the continuity of $\norm{\cdot}_{\mathcal{T}_{p}(\pi)}$ w.r.t. the weak convergence of $\pi$. Hence, a feasible coupling $\pi \in \Pi(\mu, \nu)$ that realizes the infimum only exists for $p \in [1,\infty)$ (\citet{memoli2011gromov}, Corollary 10.1\footnote{The only modification required in the first part of the proof of Corollary 10.1 is the following: Define $f: [0,M] \rightarrow \mathbb{R}_{+}$ as $t \mapsto t^{p}$ if $t \leq \tau$, and $\tau^{p}$ if $\tau < t \leq M$, which also becomes Lipschitz with constant $pM^{p-1}$. Observe that, $\tau \geq M = \textrm{diam}(\mathcal{X}) \vee \textrm{diam}(\mathcal{Y})$ implies the exact function as in \citet{memoli2011gromov}.}). Also, the triangle inequality exists owing to the same property for $\norm{\cdot}_{\mathcal{T}_{p}}$. 

\begin{proposition}[Metric properties] \label{met}
    Let $(\mathcal{X}, d_{X}, \mu_{X})$, $(\mathcal{Y}, d_{Y}, \mu_{Y})$ and $(\mathcal{Z}, d_{Z}, \mu_{Z})$ denote arbitrary Polish mm spaces. Then,
    \begin{enumerate}[(i)]
        \item (Triangle inequality)
        \begin{equation*}
            d_{\textrm{TGW}}(X,Y) \leq d_{\textrm{TGW}}(X,Z) + d_{\textrm{TGW}}(Z,Y). 
        \end{equation*}
        \item Given $\gamma \in \Pi(\mu_{X}, \mu_{Y})$, for all $0 \leq p \leq p' < \infty$
        \begin{equation*}
            \norm{\Lambda_{X,Y}}_{\mathcal{T}_{p}(\gamma \otimes \gamma)} \leq \norm{\Lambda_{X,Y}}_{\mathcal{T}_{p'}(\gamma \otimes \gamma)}.
        \end{equation*}
        \item For $\tau' > \tau \geq 0$, we have $(p,\tau)$-TGW $\leq (p,\tau')$-TGW.
    \end{enumerate}
\end{proposition}

\begin{proof}
    \textit{(i)} First, let us prove the triangle inequality for the `norm' $\norm{\cdot}_{\mathcal{T}_{p}}$ (not a norm following the formal definition). Assume $f, g \in L^{p}(\mu)$, where $\mu \in \mathcal{P}(\mathcal{X})$. Now,
    \begin{align}
        \norm{f+g}^{p}_{\mathcal{T}_{p}(\mu)} &= \int_{\mathcal{X}} \mathcal{T}_{p}(f(x) + g(x)) d\mu(x) \nonumber \\
        &= \int_{\mathcal{X}} \Big(\mathcal{T}_{p}(f + g)^{\frac{1}{p}}\Big) \Big(\mathcal{T}_{p}(f + g)^{\frac{p-1}{p}}\Big) d\mu \nonumber \\ &\leq \int_{\mathcal{X}} \Big(\mathcal{T}^{\frac{1}{p}}_{p}(f) + \mathcal{T}^{\frac{1}{p}}_{p}(g)\Big) \Big(\mathcal{T}_{p}(f + g)^{\frac{p-1}{p}}\Big) d\mu \label{subadd} \\ &= \int_{\mathcal{X}} \mathcal{T}^{\frac{1}{p}}_{p}(f)\Big(\mathcal{T}_{p}(f + g)^{\frac{p-1}{p}}\Big) d\mu + \int_{\mathcal{X}} \mathcal{T}^{\frac{1}{p}}_{p}(g)\Big(\mathcal{T}_{p}(f + g)^{\frac{p-1}{p}}\Big) d\mu \nonumber \\ &\leq  \Big[\Big(\int_{\mathcal{X}} \mathcal{T}_{p}(f) d\mu\Big)^{\frac{1}{p}} + \Big(\int_{\mathcal{X}} \mathcal{T}_{p}(g) d\mu\Big)^{\frac{1}{p}}\Big]\Big(\int_{\mathcal{X}} \mathcal{T}_{p}(f + g) d\mu\Big)^{\frac{p-1}{p}} \label{Holder} \\ &= \Big(\norm{f}_{\mathcal{T}_{p}(\mu)} + \norm{g}_{\mathcal{T}_{p}(\mu)}\Big) \norm{f+g}^{p-1}_{\mathcal{T}_{p}(\mu)}, \nonumber
    \end{align}
    where inequality (\ref{subadd}) is due to the subadditivity of $\mathcal{T}^{\frac{1}{p}}_{p}$ (\citet{musco2021active}, Lemma C.12) and H\"{o}lder inequality implies (\ref{Holder}). In the process, we assume that the norm itself is not $0$.
    
    Given arbitrary $\varepsilon > 0$, one can obtain feasible optimal couplings $\pi_{XZ} \in \Pi(\mu_{X}, \mu_{Z})$ and $\pi_{ZY} \in \Pi(\mu_{Z}, \mu_{Y})$ that satisfy the following
    \begin{equation} \label{opt}
        \frac{1}{2}\norm{\Lambda_{X,Z}}_{\mathcal{T}_{p}(\pi_{XZ} \otimes \pi_{XZ})} + \frac{1}{2}\norm{\Lambda_{Z,Y}}_{\mathcal{T}_{p}(\pi_{ZY} \otimes \pi_{ZY})} = d_{\textrm{TGW}}(Z, Y) + d_{\textrm{TGW}}(X, Z) + 2\varepsilon.
    \end{equation}
    The Gluing lemma ensures the existence of $\pi \in \mathcal{P}(\mathcal{X} \times \mathcal{Y} \times \mathcal{Z})$ with marginals $\pi_{XZ}$ on $\mathcal{X} \times \mathcal{Z}$ and $\pi_{ZY}$ on $\mathcal{Z} \times \mathcal{Y}$. Also, let $\pi_{XY}$ be the marginal of $\pi$ on $\mathcal{X} \times \mathcal{Y}$. Observe that, given $x,x' \in \mathcal{X}$, $y,y' \in \mathcal{Y}$ and $z,z' \in \mathcal{Z}$, the triangle inequality of $\Lambda_{1}$ implies
    \begin{equation*}
        \Lambda_{1}(d_{X}(x,x'),d_{Y}(y,y')) \leq \Lambda_{1}(d_{X}(x,x'),d_{Z}(z,z')) + \Lambda_{1}(d_{Z}(z,z'),d_{Y}(y,y'))
    \end{equation*}
    $(\pi \otimes \pi)$-a.e. Now,
    \begin{align}
        d_{\textrm{TGW}}(X,Y) &\leq \frac{1}{2}\norm{\Lambda_{X,Y}}_{\mathcal{T}_{p}(\pi_{XY} \otimes \pi_{XY})} = \frac{1}{2}\norm{\Lambda_{X,Y}}_{\mathcal{T}_{p}(\pi \otimes \pi)} \nonumber \\ &\leq \frac{1}{2}\norm{\Lambda_{X,Z} + \Lambda_{Z,Y}}_{\mathcal{T}_{p}(\pi \otimes \pi)} \nonumber \\ &\leq \frac{1}{2}\norm{\Lambda_{X,Z}}_{\mathcal{T}_{p}(\pi \otimes \pi)} + \frac{1}{2}\norm{\Lambda_{Z,Y}}_{\mathcal{T}_{p}(\pi \otimes \pi)} \label{tpineq}\\ &= \frac{1}{2}\norm{\Lambda_{X,Z}}_{\mathcal{T}_{p}(\pi_{XZ} \otimes \pi_{XZ})} + \frac{1}{2}\norm{\Lambda_{Z,Y}}_{\mathcal{T}_{p}(\pi_{ZY} \otimes \pi_{ZY})} \nonumber \\ &= d_{\textrm{TGW}}(Z, Y) + d_{\textrm{TGW}}(X, Z) + 2\varepsilon \nonumber,
    \end{align}
    where (\ref{tpineq}) is due to the triangle inequality of the Tukey norm. Since the choice of $\varepsilon > 0$ is arbitrary, this completes the proof.
    
    \textit{(ii)} The proof of this part follows from the monotonicity of $L^{p}$ norms. Observe that
    \begin{align*}
        \norm{\Lambda_{X,Y}}_{\mathcal{T}_{p}(\gamma \otimes \gamma)} = \norm{\mathcal{T}^{\frac{1}{p}}_{p}(\Lambda_{X,Y})}_{L^{p}(\gamma \otimes \gamma)} \leq \norm{\mathcal{T}^{\frac{1}{p'}}_{p'}(\Lambda_{X,Y})}_{L^{p'}(\gamma \otimes \gamma)} = \norm{\Lambda_{X,Y}}_{\mathcal{T}_{p'}(\gamma \otimes \gamma)}.
    \end{align*}
\end{proof}
The properties altogether make $d_{\textrm{TGW}}$ a pseudometric over the collection of isomorphism classes of mm spaces. It also limits the corruption due to Huber's $\epsilon$-contamination. For simplicity, let us assume only one of the distributions is contaminated, say $\mu$. As such, one now has observations from $\mu' = (1-\epsilon)\mu + \epsilon \mu_{c}$ instead, where $\mu_{c} \in \mathcal{P}_{p}(\mathcal{X})$. Under this setup, the following result gives the extent to which the population level loss can propagate.

\begin{proposition} \label{Prop_up}
    Given that the two distributions $\mu, \nu$ belong to the same mm space (i.e. they are namely $(\mathcal{X}, d_{X}, \mu)$ and $(\mathcal{X}, d_{X}, \nu)$), if $\mu$ suffers Huber's $\epsilon$-contamination, we have
    \begin{equation*}
        d_{\textrm{TGW}}(\mu',\nu) \leq \tau \epsilon^{\frac{1}{p}} + W_{\mathcal{T}_{p}}(\mu, \nu),
    \end{equation*}
    where $W_{\mathcal{T}_{p}} \coloneqq \inf_{\Pi} \norm{d_{X}}_{\mathcal{T}_{p}}$ is the $\textrm{OT}_{d_{X}}$ distance under the $p$-Tukey norm. 
\end{proposition}

The result also implies that TGW can pose as a provably robust estimate to GW, i.e. $d_{\textrm{TGW}}(\mu',\nu) - d_{\textrm{GW}}(\mu,\nu) \leq \tau\epsilon^{1/p}$, if the distributions are supported on the same space. This observation is instrumental in robust shape-matching and generation problems. Proposition \ref{Prop_up} also has interesting consequences under specific assumptions on the contamination model. For example, if there exists $k \geq 1$ such that $W_{p}(\mu, \mu_{c}) = k W_{p}(\mu, \nu)$ \citep{balaji2020robust}, we have
\begin{equation*}
    d_{\textrm{TGW}}(\mu',\nu) \leq (1 + k\epsilon^{\frac{1}{p}})W_{p}(\mu, \nu),
\end{equation*}
due to the trivial upper bound on $W_{\mathcal{T}_{p}}$ by the $p$-Wasserstein distance.

\begin{remark}[Resilience under TGW]
    The inequality (\ref{tri_norm}) also enables one to discuss the `resilience' \footnote{$\mu \in \mathcal{P}(\mathcal{X})$ is said to be $(\rho, \varepsilon)$-resilient w.r.t. the divergence $d$ if $\forall \tilde{\mu} \in \mathcal{P}(\mathcal{X})$ such that $\tilde{\mu} \leq \frac{1}{1-\varepsilon}\mu$, we have $d(\mu, \tilde{\mu}) \leq \rho$, where $0 \leq \varepsilon < 1$ and $\rho \geq 0$.} of a distribution $\mu$ under $d_{\textrm{TGW}}$. It boils down to checking the maximum change in $W_{\mathcal{T}_{p}}$ if a $\varepsilon$-fraction of mass under $\mu$ is deleted and renormalized to form $\tilde{\mu} \leq \frac{1}{1-\varepsilon}\mu$. \citet{nietert2023robust} show that $\abs{\mathbb{E}_{\tilde{\mu}}[d_{X}(Y, x_{0})^{p}] - \mathbb{E}_{\mu}[d_{X}(Z, x_{0})^{p}]} \leq \rho$ implies resilience under $W_{p}$, given $x_{0} \in \mathcal{X}$ (Lemma 11). In the process, it is sufficient to assume that $\mu$ has finite $p$-moment. Observe that, 
    \begin{align*}
        &\abs{\mathbb{E}_{\tilde{\mu}}[\mathcal{T}_{p}(d_{X}(Y, x_{0}))] - \mathbb{E}_{\mu}[\mathcal{T}_{p}(d_{X}(Z, x_{0}))]} \leq \mathbb{E}_{Y \sim \tilde{\mu}, Z \sim \mu}\abs{\mathcal{T}_{p}(d_{X}(Y, x_{0})) - \mathcal{T}_{p}(d_{X}(Z, x_{0}))} \\ &\leq \mathbb{E}_{Y \sim \tilde{\mu}, Z \sim \mu}\abs{d_{X}(Y, x_{0})^{p} - d_{X}(Z, x_{0})^{p}} \\ &\leq \sqrt{\textrm{Var}_{\tilde{\mu}}[d_{X}(Y, x_{0})^{p}]} + \sqrt{\textrm{Var}_{\mu}[d_{X}(Z, x_{0})^{p}]} + \abs{\mathbb{E}_{\tilde{\mu}}[d_{X}(Y, x_{0})^{p}] - \mathbb{E}_{\mu}[d_{X}(Z, x_{0})^{p}]},
    \end{align*}
    where the first inequality is due to Jensen's inequality. As such, given that the variances are finite, mean resilience also implies the same under $\mathcal{T}_{p}$. However, in case $\tilde{\mu}$ results in a vastly distinct variance, the associated resilience bound on $W_{\mathcal{T}_{p}}$ becomes weak. 
    \begin{corollary}[\citet{nietert2023robust}] \label{corrniet}
        Given $Z \sim \mu$ and $x_{0} \in \mathcal{X}$, let $\mathbb{E}_{\mu}[d_{X}(Z, x_{0})^{p}] \leq \sigma^{p}$ for some $\sigma \geq 0$. If $\mathcal{T}_{p}(d_{X}(Z, x_{0}))$ is $(\rho, \varepsilon)$-resilient in mean, then $\mu$ is $\Big(2 \Big((\rho^{\frac{1}{p}} + \varepsilon^{\frac{1}{p}}(\sigma \wedge \tau)) \wedge \varepsilon^{\frac{1}{p}}\tau\Big), \varepsilon\Big)$-resilient w.r.t. $W_{\mathcal{T}_{p}}$.
    \end{corollary}
    Observe that, the bound is non-trivial only when $\sigma \leq \tau$. While a user-defined $\tau$ ensures resilience for distributions having thicker tails, the result hints towards distributions ($\mu$) that imply sharper resilience bounds. One immediate example is the class of sub-Gaussian distributions. 

    \begin{proof}
    For any $\tilde{\mu} \leq \frac{1}{1-\varepsilon}\mu$, by choosing an appropriate $\beta \in \mathcal{P}(\mathcal{X})$ we can write $\mu = (1-\varepsilon)\tilde{\mu} + \varepsilon \beta$. Now,
    \begin{align}
        W_{\mathcal{T}_{p}}(\mu, \tilde{\mu}) &\leq \varepsilon^{\frac{1}{p}} (W_{\mathcal{T}_{p}}(\beta, \delta_{x_{0}}) + W_{\mathcal{T}_{p}}(\delta_{x_{0}}, \tilde{\mu})) \nonumber \\ & = \varepsilon^{\frac{1}{p}} \Big[\norm{d_{X}(Y, x_{0})}_{\mathcal{T}_{p}(\beta)} + \norm{d_{X}(Y, x_{0})}_{\mathcal{T}_{p}(\tilde{\mu})} \Big] \nonumber \\ &\leq 2\varepsilon^{\frac{1}{p}} \sup_{\alpha \in \mathcal{P}(\mathcal{X}), \alpha \leq \frac{1}{1-\varepsilon'}\mu} \norm{d_{X}(Y, x_{0})}_{\mathcal{T}_{p}(\alpha)} \label{last1},
    \end{align}
    where $\varepsilon' \coloneqq \varepsilon \vee (1-\varepsilon)$. Given that the expectation is finite, the definition of $\mathcal{T}_{p}$ implies $\mathbb{E}_{\alpha}[\mathcal{T}_{p}(d_{X}(Y, x_{0}))] \leq \tau^{p}$. Moreover,
    \begin{align}
        \mathbb{E}_{\alpha}[\mathcal{T}_{p}(d_{X}(Y, x_{0}))] &\leq \abs{\mathbb{E}_{\alpha}[\mathcal{T}_{p}(d_{X}(Y, x_{0}))] - \mathbb{E}_{\mu}[\mathcal{T}_{p}(d_{X}(Z, x_{0}))]} + \mathbb{E}_{\mu}[\mathcal{T}_{p}(d_{X}(Z, x_{0}))] \nonumber \\ &\leq \Big(1 \vee \frac{1-\varepsilon}{\varepsilon}\Big)\rho + (\sigma^{p} \wedge \tau^{p}) \label{last},
    \end{align}
    where the first inequality is due to triangle inequality. As such, combining inequality (\ref{last}) with the bound (\ref{last1}) yields,
    \begin{equation*}
        W_{\mathcal{T}_{p}}(\mu, \tilde{\mu}) \leq 2 \Big((\rho^{\frac{1}{p}} + \varepsilon^{\frac{1}{p}}(\sigma \wedge \tau)) \wedge \varepsilon^{\frac{1}{p}}\tau\Big).
    \end{equation*}
    \end{proof}
\end{remark}
\begin{remark}[Lower bound to ROBOT] \label{remROB}
    TGW has a surprising relation to existing robust efforts in OT, following from Proposition \ref{Prop_up}. Given a transportation cost $c : \mathcal{X} \times \mathcal{X} \rightarrow \mathbb{R_{+}}$, \citet{mukherjee2021outlier} (formulation 2) define the $\lambda$-ROBOT distance between $\mu, \nu \in \mathcal{P}(\mathcal{X})$ as $\textrm{OT}_{c_{\lambda}}(\mu, \nu)$, where $c_{\lambda}(x,y) \coloneqq l_{2\lambda}(c(x,y)) = \min\{c(x,y), 2\lambda\}$. Specifically for $c=d_{X}$, we write $W_{1,2\lambda}(\mu,\nu)$. It metrizes the underlying class of distributions and was shown earlier to lead to faster cost computation in tasks such as image retrieval \citep{pele2009fast}. On the other hand, if $p=1$, the Tukey loss boils down to the exact functional $l_{\lambda}(x) = \min\{x, \tau\}$, $x > 0$. As such, for any coupling $\pi$, by assuming $\tau = 2\lambda$ we have 
   \begin{equation*}
       \norm{\min\{2\Lambda_{1}, \tau\}}_{L^{1}(\pi \otimes \pi)} \leq \norm{\min\{d_{X}(x,y) + d_{X}(x',y'), \tau\}}_{L^{1}(\pi \otimes \pi)} \leq 2\norm{\min\{d_{X}(x,y), \tau\}}_{L^{1}(\pi)}.
   \end{equation*}
    Taking infimum over $\Pi(\mu, \nu)$ we conclude $(1,2\lambda)$-$d_{\textrm{TGW}} \leq W_{1,2\lambda}$. 
\end{remark}
\begin{remark}[Concentration]
    In reality, often outlying observations find their way in the pool of samples, which is unlike drawing i.i.d. replicates from a contaminated distribution $\mu' = (1-\epsilon)\mu + \epsilon \mu_{c}$. Rather, in a set of samples $\{x_{i}\}_{i=1}^{m}$, we are left with $|\mathcal{I}|$ i.i.d. observations from $\mu$ and the rest, $|\mathcal{O}| \coloneqq n - |\mathcal{I}|$ drawn independently from adversaries. If the outliers also follow $\mu_{c}$ identically and $|\mathcal{O}| = m\epsilon$, it becomes equivalent to Huber's contamination regime in an empirical setup. \citet{lecue2020mom} call this the $\mathcal{O} \cup \mathcal{I}$ framework. This is crucial since it allows one to comment on the concentration of the empirical $d_{\textrm{TGW}}$. In such a setup, given samples $\{(x_{i},y_{j})\}^{m,n}$, as a corollary to Proposition \ref{Prop_up}, we get for $p=2$
    \begin{equation}
        d_{\textrm{TGW}}(\hat{\mu}_{m}, \hat{\nu}_{n}) - d_{\textrm{GW}}(\hat{\mu}^{\mathcal{I}}_{m}, \hat{\nu}^{\mathcal{I}}_{n}) \leq \tau \left(\sqrt\frac{{\abs{\mathcal{O}^{X}}}}{m} + \sqrt\frac{{\abs{\mathcal{O}^{Y}}}}{n}\right),
    \end{equation}
    where $\hat{\mu}_{m}, \hat{\nu}_{n}$ are the usual empirical distributions based on $\{(x_{i},y_{j})\}^{m,n}$, and $\hat{\mu}^{\mathcal{I}}_{m} \coloneqq \abs{\mathcal{I}^{X}}^{-1}\sum_{i \in \mathcal{I}^{X}} \delta_{x_{i}}$ is the same based on inliers. The same goes for $\hat{\nu}^{\mathcal{I}}_{n}$ in the other space. Moreover, $\mathcal{O}^{X}$ and $\mathcal{O}^{Y}$ denote the set of outliers. As such, given $\abs{\mathcal{O}^{X}} \vee \abs{\mathcal{O}^{Y}} = o(m \wedge n)$ one may equivalently calculate TGW based on only the inliers. Moreover, due to \citet{zhang2024gromov}, Theorem 3
    \begin{equation*}
        \abs{\mathbb{E}[d^{2}_{\textrm{GW}}(\hat{\mu}^{\mathcal{I}}_{m}, \hat{\nu}^{\mathcal{I}}_{n})] - d^{2}_{\textrm{GW}}(\mu,\nu)} \lesssim \frac{M^{4}}{\sqrt{\abs{\mathcal{I}^{X}} \wedge \abs{\mathcal{I}^{Y}}}} + (1+M^{4})\bigvee_{\mathcal{I}^{X}, \mathcal{I}^{Y}}\abs{\mathcal{I}}^{-\frac{2}{(d \wedge d')\vee 4}}(\log \abs{\mathcal{I}})^{\mathbb{1}_{\{d \wedge d' =4\}}}.
    \end{equation*}
\end{remark}
The theoretical richness of TGW gives us a solid foundation to search for better approximations using data-dependent thresholding. It is also quiet intuitive that a misspecified $\tau > 0$ may lead to heavier penalization than required, generating a large deviation from GW. In practice, even minute fine-tuning errors lead to a significant loss in tail information. A smoother thresholding may achieve a nearer robust approximation without sacrificing favorable properties. This leads us to the Huber `norm'.
\begin{definition}[Huber loss function] The Huber loss with threshold $\tau > 0$ is defined as 
    \begin{equation*}
        \mathcal{H}(x) \coloneqq \begin{cases}
                x^{2}/2\tau & \textrm{if } \abs{x} \leq \tau \\
                \abs{x} - \tau/2 & \textrm{otherwise},
            \end{cases}
    \end{equation*}
    which induces the corresponding norm, $\norm{f}_{\mathcal{H}(\mu)} = \big(\int_{\mathcal{X}} \mathcal{H}(f(x)) d\mu(x)\big)^{1/2}$.
\end{definition}
Observe that, $\mathcal{H}(\cdot)$ is continuously differentiable and given $\tau \simeq 0$, closely approximates the $l_{1}$ loss. The robust penalization also becomes data-dependant, making it essential in robust M-estimation and regression \citep{loh2017robust}. Following our previous formulation, for $(\mu, \nu) \in \mathcal{P}_{2}(\mathcal{X}) \times \mathcal{P}_{2}(\mathcal{Y})$, we define $d_{\textrm{HGW}}(\mu, \nu) \coloneqq \inf_{\pi \in \Pi(\mu, \nu)} \frac{1}{2}\norm{\Lambda_{X,Y}}_{\mathcal{H}(\pi \otimes \pi)}$, namely the \textit{Huber's GW} (HGW). Addressing the robustness of GW for $p=2$ in particular, $\norm{\Lambda_{X,Y}}_{\mathcal{H}}$ does not admit a monotonic property. However, based on the fact that $\mathcal{H}^{\frac{1}{2}}$ is subadditive \citep{clarkson2014sketching}, one can recover a Minkowski-like inequality as in TGW. Moreover, HGW poses as a robust estimate of the corresponding GW value as it follows a property similar to Proposition \ref{Prop_up}. The result involves defining the Huber version of the modified Wasserstein `distance' $W_{\mathcal{H}} \coloneqq \inf_{\Pi} \norm{d_{X}}_{\mathcal{H}}$. Consequently, it promotes resilience to a distribution under it upon mass truncation. 

To empirically demonstrate HGW's robustness, we devote the rest of the section to building an algorithm to solve a sample HGW. Given $m$ and $n \in \mathbb{N}_{+}$ i.i.d. samples from $\mu$ and $\nu$ respectively, let us denote the two pairwise distance matrics (based on $d_{X}$ and $d_{Y}$) as $C^{X} \in \mathbb{R}^{m \times m}_{+}$ and $C^{Y} \in \mathbb{R}^{n \times n}_{+}$. Then, $d^{2}_{\textrm{HGW}}$ boils down to solving the familiar non-convex optimization \citep{peyre2016gromov}
\begin{equation} \label{Alg}
    \min_{\pi \in \Pi(\hat{\mu}_{m}, \hat{\nu}_{n})} \sum_{i,i',j,j'}\mathcal{H}(C^{X}_{ii'} - C^{Y}_{jj'})\pi_{ij} \pi_{i'j'} = \min_{\pi \in \Pi(\hat{\mu}_{m}, \hat{\nu}_{n})} \langle \mathcal{H}(C^{X} - C^{Y}) \circledcirc \pi, \pi\rangle,
\end{equation}
where $\hat{\mu}_{m}, \hat{\nu}_{n}$ are empirical distributions or rather simplexes and $\Pi \in \mathbb{R}^{m \times n}_{+}$. To adapt to the Sinkhorn scaling framework \citep{cuturi2013sinkhorn}, we additionally impose an entropic regularization $d_{\textrm{KL}}(\pi)$ to (\ref{Alg}), which at the $k$-th iteration calculates $d_{\textrm{KL}}(\pi | \pi^{k}) = \langle \pi, \log \pi - \log \pi^{k}\rangle$. The resultant Huber's EGW formulation follows the Algorithm \ref{alg:HEGW}. It is immediately beneficial for computing a robust loss, compared to Unbalanced GW (UGW) or Partial GW (PGW), since it results in marginal distributions and computationally scales with the EGW ($O(m^{2}n^{2})$) exactly. 

\begin{algorithm}
\caption{Huber's Entropic Gromov-Wasserstein}\label{alg:HEGW}
\begin{algorithmic}
\State \textbf{Input:} Initialised distributions $p,q$, regularization parameter $\varepsilon$, number of inner and outer iterations $N_{2}, N_{1}$.
\State \textbf{Output:} HEGW
\let \oldnoalign \noalign
\let \noalign \relax
\midrule
\let \noalign \oldnoalign
\State Compute pairwise distance matrices $C^{X}, C^{Y}$
\State Initialise $\pi^{(0)} = pq^{T}$, $a^{(0)} = 1_{m}$, and $b^{(0)} = 1_{n}$
\For{$i \in \{0, 1, \cdots , N_{1}-1\}$}
            \State $\mathscr{C}(\pi^{(i)}) \leftarrow \mathcal{H}(C^{X} - C^{Y}) \circledcirc \pi^{(i)}$ \Comment{Compute cost matrix}
            \State $K^{(i)} \leftarrow \exp{\big(-\frac{\mathscr{C}(\pi^{(i)})}{\varepsilon}\big)} \odot \pi^{(i)}$  \Comment{Compute kernel}
            \For{$j \in \{0, 1, \cdots , N_{2}-1\}$} 
            \State $\big\{a^{(j+1)}, b^{(j+1)}\big\} \leftarrow \big\{p \oslash (K^{(i)}b^{(j)}), q \oslash ({K^{(i)}}^{T}a^{(j+1)})\big\}$  \Comment{Sinkhorn scaling}
            \EndFor
            \State $\pi^{(i+1)} \leftarrow \textrm{diag}(a^{(N_{2})})K^{(i)}\textrm{diag}(b^{(N_{2})})$
        \EndFor
        \State Return $\big\langle \mathscr{C}(\pi^{(N_{1})}), \pi^{(N_{1})}\big\rangle$
\end{algorithmic}
\end{algorithm}

While we present a simple working algorithm\footnote{This allows seamless integration into existing libraries: \url{https://pythonot.github.io/}}, HEGW also adapts to lower-complexity approximations. In Algorithm \ref{alg:HEGW}, computing the cost $\mathcal{C}(\pi)$ alone incurs the high complexity $O(m^{2}n^{2})$. However, we can write $\mathcal{H}(a-b) = f_{1}(a) + f_{2}(b) - h_{1}(a)h_{2}(b)$, where given $\abs{a-b} \leq \tau$, $f_{1}(a) = a^{2}/2\tau, f_{2}(b) = b^{2}/2\tau, h_{1}(a) = a/\tau, h_{2}(b) = b$ and if $a-b > \tau$, we have $f_{1}(a) = a, f_{2}(b) = - b, h_{1}(a) = \tau, h_{2}(b) = 1/2$. As such, the cost computation can be eased down to $O(m^{2}n + mn^{2})$ (\citet{peyre2016gromov}, Proposition 1). This is the best complexity achievable if $\mu, \nu$ are not sliced first or no additional constraint on the matrices satisfying lower ranks is imposed. However, if along with the robust penalization, we identify a set of indices $\mathcal{S} = \{(i,j)\}$ with $\abs{\mathcal{S}} = s$, such that
\begin{equation*}
    \Tilde{\mathcal{H}}(C^{X}_{ii'} - C^{Y}_{jj'}) = \begin{cases}
        \mathcal{H}(C^{X}_{ii'} - C^{Y}_{jj'}) & \textrm{if } (i',j') \in \mathcal{S} \\
        0 & otherwise,
    \end{cases}
\end{equation*}
where $(i,j) \in \mathcal{S}$, the modified HEGW problem can be solved with accompanying complexity $O(mn + s^{2})$ \citep{li2023efficient}. While this results in a truncated plan ($\pi$), it effectively thwarts outliers in graphs.

\subsection*{Experiment: Shape Matching with Outliers} 
We deploy HGW for robust 2D shape matching based on point cloud data \citep{mroueh2020unbalanced}. Observe that the underlying mm spaces are essentially $(\mathbb{R}^{2}, \norm{\:\cdot\:}_{2})$, endowed with measures $\hat{\mu}_{m}$ and $\hat{\nu}_{n}$ respectively. Given two shapes (e.g. cat and heart), we identify one as the target and the other as the source. The contamination regime we follow is the following: for $\alpha \in (0,1)$, we randomly sample $m\alpha$ observations from the source point cloud and replace them with replicates from an adversary $\mu_{c}$ (e.g. standard bi-variate Cauchy). 
\begin{figure}[H]
    \centering
    \includegraphics[width=0.9\linewidth]{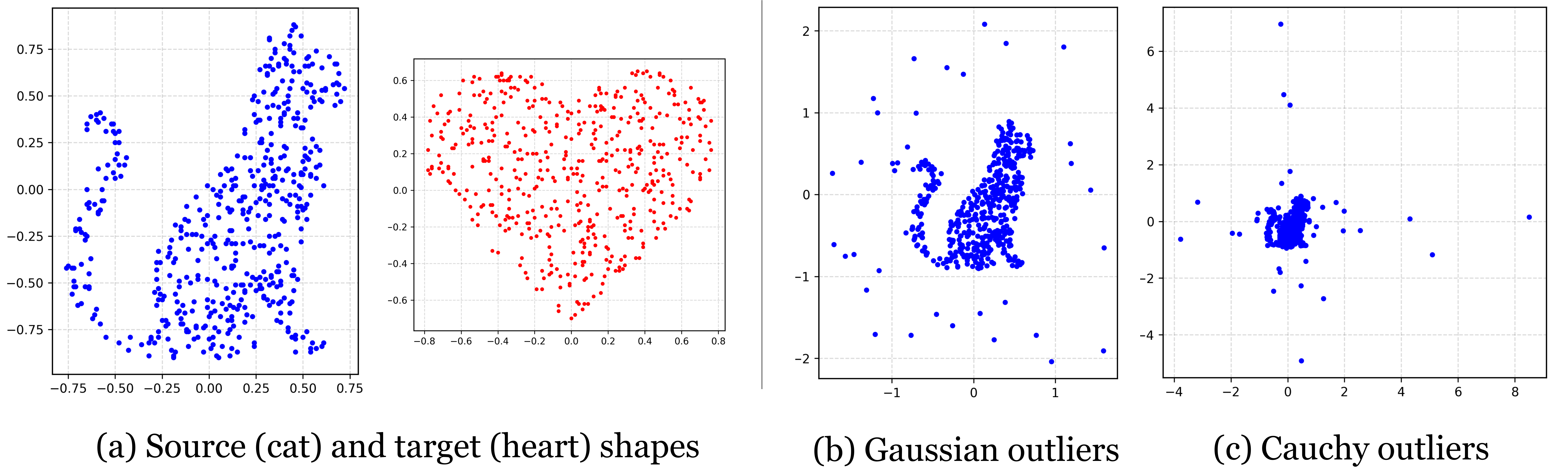}
    \caption{(a) Point clouds ($m = n = 500$) corresponding to shapes of cat (source) and heart (target). Contaminated source with 20 outliers drawn independently from a standard (b) bivariate Gaussian and (c) bivariate Cauchy.}
    \label{fig:cat_heart_data}
\end{figure}
For comparison, we use the vanilla GW, FGW \citep{vayer2020fused}, PGW \citep{chapel2020partial}, and UGW \citep{sejourne2021unbalanced} as baselines under $p=2$. For the unbalanced methods, we allocate unit mass to each point and for the rest, we normalize. Each method, at each level of contamination, is repeated 100 times to cover for any variation generated due to optimization. The parameters for each distance are selected based on the recommendations by the respective authors, e.g. in the case of FGW, the mixing coefficient of GW and the OT is kept at $0.5$. The regularization parameter for PGW is taken as $0.001$. We find that only such small values, chosen judiciously, can strike a balance between adequate penalization and a low enough value of the corresponding loss. In TGW (and HGW), the method's accuracy hinges on the selection of $\tau$. Very low values lead to over-penalization and large deviations from the actual robust benchmark. In our study, we devise a data-driven scheme for selecting $\tau$. Given observations $\{(x_{i},y_{j})\}^{m,n} \sim \mu \otimes \nu$, we scrutinize the distribution of sample distortion values (say, $J_{X,Y}$) $|\:{\norm{x_{i} - x_{i'}}^{2} - \norm{y_{j} - y_{j'}}^{2}}\:|$. An immediate estimate for a threshold that trims outlying $J_{X,Y}$ values is a higher percentile, e.g. $98\%,95\%$, which we use as a reference. Our choice of an appropriate $\tau$ becomes $\tilde{m} + 3\tilde{\sigma}$, where $\tilde{m}$ and $\tilde{\sigma}$ are the median and mean absolute deviation about median of $J_{X,Y}$. Ideally, for a standard folded Normal distribution, the value turns out $\approx 2.04$ (see, Appendix \ref{param_sel}). The method enables a dynamic parameter selection that adjusts according to the proportion of outliers. We present a detailed discussion in Appendix \ref{param_sel}. 



\begin{figure}[H]
    \centering
    \begin{subfigure}{0.425\linewidth}
        \includegraphics[width=\linewidth]{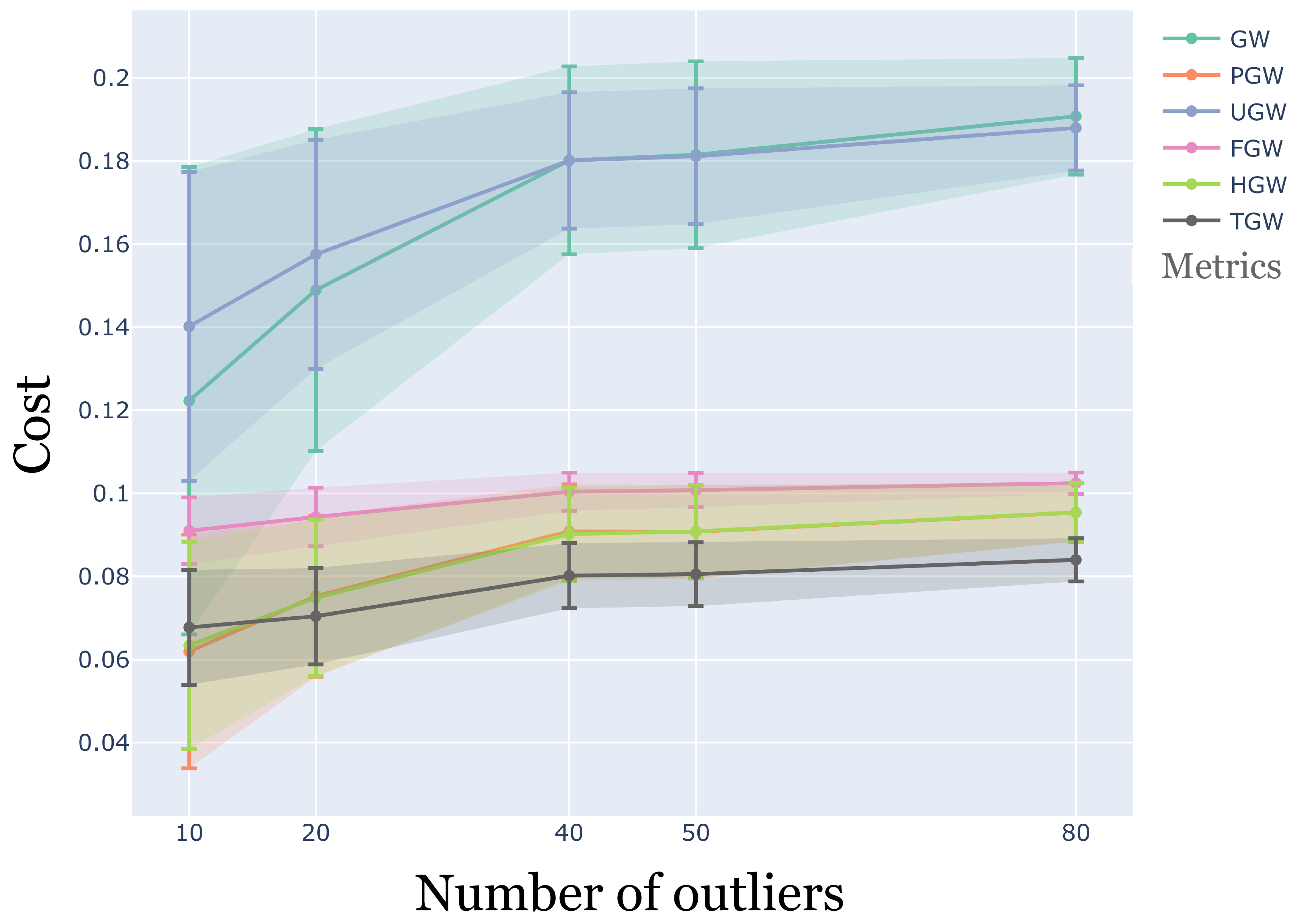}
        \caption{}
        \label{fig:cauchy_plot}
    \end{subfigure}
    \begin{subfigure}{0.565\linewidth}
      \includegraphics[width=\linewidth]{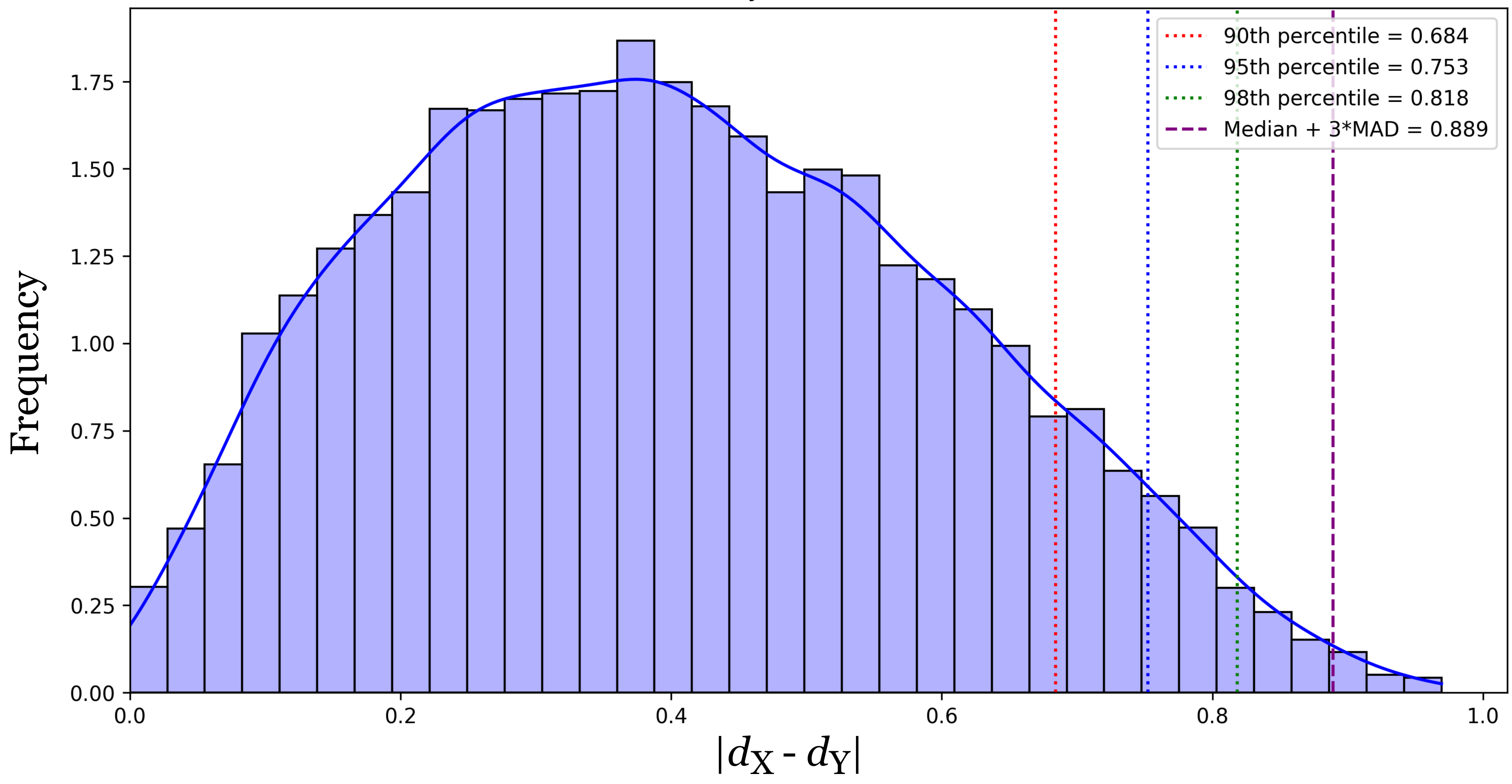}
      \caption{}
      \label{fig:cauchy_dist}
    \end{subfigure}
    \caption{(a) Average loss values under increasing proportion of bi-variate Cauchy outliers $(0.02,0.04,0.08,0.1,0.16)$ in the source domain. (b) Empirical distribution of deviations between pairwise distances under $80$ Cauchy outliers. Realized $95$-percentile and $\tilde{m} + 3\tilde{\sigma}$ are $0.753$ and $0.889$ respectively.} 
    \label{fig:cat_heart_Cauchy}
\end{figure}

Based on the proposed scheme, the value of $\tau$ is chosen dynamically at each level $\alpha$ for HGW. As a reference, we use the $95$-percentile of $J_{X,Y}$ in TGW. The immediate observation is that TGW remains the most stable under increasing contamination. On the other hand, HGW exhibits performance comparable to that of partial mass allocation, as in PGW. Since the threshold only penalizes extreme values of $J_{X,Y}$, pairwise distances between outliers that are similar to that between inliers contribute to the overall loss. This implies a minute increase in HGW. The effect is much pronounced in Gaussian outliers (see, Appendix \ref{param_sel}). The stability of PGW is intuitive since its optimal plan ignores the outliers altogether. Remarkably, HGW simulates the same effect without altering the plan. FGW (with a mixing ratio $0.5$) shows elevated fluctuation as its OT component is also vulnerable to contamination. The instability of UGW might stem from mass-splitting and the partial dependence of its plan on outliers (as also pointed out by \citet{bai2024partialgromovwassersteinmetric}). However, despite relying on a full plan connecting outliers to points in the target shape, HGW results in robust estimates of the corresponding loss. In Appendix \ref{param_sel}, we show the results corresponding to Gaussian contamination and compare the effects due to varying $\tau$.

\subsection{Local Robustification} \label{Part2}
Techniques that make the distortion $\Lambda(d_{X}(x,x'),d_{Y}(y,y'))$, given $(x,y), (x',y') \in \mu \otimes \nu$, robust to outliers, offer a robust estimate to the extent of deviation from isometry. It is equivalent to selecting a different measurement to compare the two local geometries, however perturbed; for example, TGW. In other words, the penalization applies over the discrepancy in pairwise distances, rather than $d_{X}(x,x')$ and $d_{Y}(y,y')$ themselves. This may suffer from inefficiency in information retention, as given only an outlying observation $x' \in \mathcal{X}$, it depreciates the contribution of a `clean' $d_{Y}(y,y')$. An earlier-stage robustification of the distances solves this issue. In search of nearer GW approximates, we turn to robust surrogates of $d_{X}$ and $d_{Y}$. 

For typical choices of metric spaces (for example, $(\mathbb{R}^{d}, \norm{\:\cdot\:})$), it is quite straightforward to retain metric properties of $d_{X}$ under thresholding (e.g., Tukey) or Winsorization. For example, the \textit{truncated} surrogate $l_{\lambda}(d_{X}) = \min\{d_{X}, \lambda\}$, $\lambda \geq 0$ satisfies non-negativity, symmetry and the triangle inequality. Based on the fact that $\mathcal{T}_{2}(a-b) \geq \abs{l_{\lambda}(a) - l_{\lambda}(b)}^{2}$, for $a,b \geq 0$, it immediately improves the corresponding robust GW formulation, evoking a lower bound to Tukey-type distances.  

\begin{lemma} \label{ineqtukloc}
    Given $(\mu, \nu) \in \mathcal{P}_{4}(\mathbb{R}^{d}) \times \mathcal{P}_{4}(\mathbb{R}^{d'})$, the $(2,\lambda)$-TGW between them satisfies
    \begin{align*}
        d_{\textrm{TGW}}(\mu, \nu) \geq \frac{1}{2} \Big(\inf_{\pi \in \Pi(\mu,\nu)} \int \int \abs{l_{\lambda}\Big(\norm{x - x'}^{2}\Big) - l_{\lambda}\Big(\norm{y - y'}^{2}\Big)}^{2} d \pi \otimes \pi (x,y,x',y')\Big)^{\frac{1}{2}}.
    \end{align*}
\end{lemma}
We call such formulations, as in the lower bound, \textit{locally robust GW} ($(p,\lambda)$-LRGW, in general). Infima of the corresponding optimization are always realized at relaxed optimal couplings (see, Appendix \ref{exist}). The modification gives greater control over the extent of robustification based on distinct choices of $\lambda,\lambda' \geq 0$ for the two mm spaces. LRGW also follows most metric properties of GW, particularly non-negativity, symmetry, and triangle inequality. The proofs become similar to showing the same for GW under altered mm spaces of the form $(\mathcal{X}, l_{\lambda}(d_{X}), \mu)$, $\lambda > 0$. For completeness, we mention some properties of LRGW, including its dependence on the threshold $\lambda$.

\begin{lemma}[Properties] \label{properLR}
    Given Polish mm spaces $(\mathcal{X}, d_{X}, \mu)$, $(\mathcal{Y}, d_{Y}, \nu)$; and $p \in [0,\infty]$
    \begin{enumerate}[(i)] 
        \item for $\lambda' > \lambda \geq 0$, we have $(p,\lambda)$-LRGW($\mu,\nu$) $\leq (p,\lambda')$-LRGW($\mu,\nu$). In fact,
        \begin{equation*}
            (p,\lambda')\textrm{-LRGW}(\mu,\nu) - (p,\lambda)\textrm{-LRGW}(\mu,\nu) \leq \inf_{\pi \in \Pi^{\lambda}(\mu,\nu)} \norm{l_{\lambda'}(\bar{l}_{\lambda}(d_{X})) - l_{\lambda'}(\bar{l}_{\lambda}(d_{Y}))}_{L^{p}(\pi \times \pi)},
        \end{equation*}
        where $\bar{l}_{\lambda}(x) = \max\{x,\lambda\}$ and $\Pi^{\lambda}$ is the set of couplings optimal for $(p,\lambda)$-LRGW.
        \item $(p,\infty)$-LRGW($\mu,\nu$) = $p$-GW($\mu,\nu$).
    \end{enumerate}
\end{lemma}

Observe that, lemma \ref{properLR}\textit{(ii)} rather holds for any $\lambda \geq M = \textrm{diam}(\mathcal{X}) \vee \textrm{diam}(\mathcal{Y})$, given $\textrm{diam}(\mathcal{X}) = \max_{x,x' \in \mathcal{X}}d_{X}(x,x')$, which however may become arbitrarily large in the presence of outliers in a sample problem. As a consequence of lemma \ref{properLR}\textit{(ii)}, $\mathcal{X} \cong \mathcal{Y}$ implies that the associated LRGW nullifies. However, the converse does not hold necessarily since, $l_{\lambda}(d_{X}(x,x')) = l_{\lambda}(d_{Y}(y,y'))$ a.s. does not imply $d_{X}(x,x')=d_{Y}(y,y')$ a.s. While this formulation sacrifices non-degeneracy, it preserves geometric sensitivity under appropriately tuned $\lambda$. It delineates an estimated support of the distribution of distances based on inlying observations\footnote{Given $x \in \mathcal{X}$, define the map $u^{\lambda}_{x} : \mathcal{X} \rightarrow [0,\lambda)$ by $u^{\lambda}_{x}(x') = l_{\lambda}(d_{X}(x,x'))$. Then, ${u^{\lambda}_{x}}_{\#}\mu \in \mathcal{P}([0,\lambda))$ denotes the distribution of distances supported on the trimmed interval.}. Though finer than TGW, such a filtration allows distances corresponding to outlying samples within $\lambda$-radius to each other to pass through. This hints towards addressing contamination due to distribution shifts and mass reallocation as a result. Later in the section, we discuss the origin of local penalization ($l_{\lambda}$) in a generalized setup. To motivate, we mention that similar costs are often used to devise robust OT-based divergences metrizing $\mathcal{P}(\mathcal{X})$ (see Remark \ref{remROB}). Remarkably, impartially trimmed-$W_{p}$ due to \citet{czado1998assessing} becomes equivalent to trimming the underlying univariate measures \citep{alvarez2008trimmed}. In this line, our next result shows that variational representations of LRGW formulations link the alignment problem to certain robust OT costs, relying on trimmed observations instead. Given two mm spaces $(\mathbb{R}^{d}_{\geq 0}, \norm{\:\cdot\:}, \mu)$ and $(\mathbb{R}^{d'}_{\geq 0}, \norm{\:\cdot\:}, \nu)$\footnote{Such subspaces remain Polish equipped with the inherited metric (\citet{fristedt2013modern}, Chapter 18.1). We may equivalently choose $[0,1]^d$, which is sufficient for the IGW formulation. In both cases, the corresponding problem boils down to assessing the alignment between distributions corresponding to non-negative multivariate random variables.}, let us consider the locally robust inner product GW distance \citep{memoli2011gromov}
\begin{equation*}
    d_{\textrm{LRIGW}}(\mu, \nu; \lambda) \coloneqq \Big(\inf_{\pi \in \Pi(\mu,\nu)} \int\int \abs{l_{\lambda}(\langle x, x'\rangle) - l_{\lambda}(\langle y , y'\rangle)}^{2} d \pi \otimes \pi (x,y,x',y')\Big)^{\frac{1}{2}}.
\end{equation*}
Here, the robustification translates to limiting extreme angular deviation. At $\lambda \rightarrow \infty$, the cost circles back to IGW. To state the result, let us first decompose the squared LRIGW cost in the following way:
\begin{equation*}
    d^{2}_{\textrm{LRIGW}}(\mu, \nu; \lambda) = F_{1} + F_{2},
\end{equation*}
where
\begin{align*}
        &F_{1}(\mu, \nu; \lambda) = \int \abs{l_{\lambda}(\langle x, x'\rangle)}^{2} d \mu \otimes \mu (x,x') + \int \abs{l_{\lambda}(\langle y, y'\rangle)}^{2} d \nu \otimes \nu (y,y'), \\ &F_{2}(\mu, \nu; \lambda) = \inf_{\pi \in \Pi(\mu,\nu)} -2 \int l_{\lambda}(\langle x, x'\rangle) \:l_{\lambda}(\langle y , y'\rangle) d \pi \otimes \pi (x,y,x',y').
    \end{align*}
\begin{proposition}[Locally robust IGW duality] \label{propIGW}
    Given $(\mu, \nu) \in \mathcal{P}_{4}(\mathbb{R}^{d}_{\geq 0}) \times \mathcal{P}_{4}(\mathbb{R}^{d'}_{\geq 0})$, define $M^{\lambda}_{\mu,\nu} \coloneqq \sqrt{M_{2}(\mu;\lambda)M_{2}(\nu;\lambda)}$, where for any distribution $\rho$, $M_{2}(\rho;\lambda) = \int \norm{l_{\lambda}(x)}^{2} d\rho(x)$, $\lambda \geq 0$. Then, there exists an upper bound to $F_{2}$, say $\bar{F}_{2}$, satisfying
    \begin{equation*}
        \bar{F}_{2}(\mu, \nu; (d \vee d')\lambda^{2}) = \inf_{\textbf{A} \in \mathcal{D}_{M^{\lambda}_{\mu,\nu}}} 8 \norm{\textbf{A}}^{2}_{F} + \textrm{OT}_{c^{\lambda}_{\textbf{A}}}(\mu, \nu),
    \end{equation*}
    where $\mathcal{D}_{M^{\lambda}_{\mu,\nu}} \coloneqq [0, M^{\lambda}_{\mu,\nu}/2]^{d \times d'}$ and $c^{\lambda}_{\textbf{A}}: (x,y) \in \mathbb{R}^{d}_{\geq 0} \times \mathbb{R}^{d'}_{\geq 0} \mapsto -8l_{\lambda}(x)^{T}\textbf{A}l_{\lambda}(y)$ denotes the cost function deployed under $\textrm{OT}_{c^{\lambda}_{\textbf{A}}}$.
\end{proposition}

\begin{proof}
    The proof follows the decomposition of the GW cost due to \citet{zhang2024gromov}. Recall the decomposition of the squared LRIGW cost: $d^{2}_{\textrm{LRIGW}}(\mu, \nu) = F_{1} + F_{2}$, where $F_{2}(\mu, \nu; \lambda) = \inf_{\pi \in \Pi(\mu,\nu)} -2 \int l_{\lambda}(\langle x, x'\rangle) \:l_{\lambda}(\langle y , y'\rangle) d \pi \otimes \pi (x,y,x',y')$.
    
    Now, for all $x,x' \sim \mu \in \mathcal{P}_{4}(\mathbb{R}^{d}_{\geq 0})$
    \begin{equation*}
        l_{\lambda}(\langle x, x'\rangle) = l_{\lambda}\Big(\sum_{i=1}^{d}x_{i}x'_{i}\Big) \geq \sum_{i=1}^{d} l_{\frac{\lambda}{d}}(x_{i}x'_{i}) \geq \sum_{i=1}^{d} l_{\sqrt{\frac{\lambda}{d}}}(x_{i})l_{\sqrt{\frac{\lambda}{d}}}(x'_{i}) = \Big\langle l_{\sqrt{\frac{\lambda}{d}}}(x), l_{\sqrt{\frac{\lambda}{d}}}(x')\Big\rangle.
    \end{equation*}
    In the last step, the function $l_{\lambda}$ applies componentwise. Similarly, the inequality $l_{\lambda}(\langle y, y'\rangle) \geq \sum_{j=1}^{d'} l_{\sqrt{\lambda/d'}}(y_{j})l_{\sqrt{\lambda/d'}}(y'_{j})$ holds for all $y,y' \sim \nu \in \mathcal{P}_{4}(\mathbb{R}^{d'}_{\geq 0})$. Let us generalize by defining $M^{(\lambda, \lambda')}_{\mu,\nu} \coloneqq \sqrt{M_{2}(\mu;\lambda)M_{2}(\nu;\lambda')}$, where $M_{2}(\rho;\lambda) = \int \norm{l_{\lambda}(x)}^{2} d\rho(x)$ for any $\rho$. Also, let $\mathcal{D}_{M^{(\lambda,\lambda')}_{\mu,\nu}} \coloneqq [0, M^{(\lambda, \lambda')}_{\mu,\nu}/2]^{d \times d'}$. Hence,
    \begin{align}
        F_{2} &\leq \inf_{\pi \in \Pi(\mu,\nu)} -2 \sum_{{1\leq i \leq d} \atop {1\leq j \leq d'}} \Big(\int l_{\sqrt{\frac{\lambda}{d}}}(x_{i}) l_{\sqrt{\frac{\lambda}{d'}}}(y_{j}) d \pi (x,y) \Big)^{2} \label{F2bar} \\ & = \inf_{\pi \in \Pi(\mu,\nu)} \sum_{{1\leq i \leq d} \atop {1\leq j \leq d'}} \inf_{0 \leq a_{ij} \leq \frac{M^{\bar{\lambda}}_{\mu,\nu}}{2}} 8\Big(a^{2}_{ij} - \int a_{ij}l_{\sqrt{\frac{\lambda}{d}}}(x_{i}) l_{\sqrt{\frac{\lambda}{d'}}}(y_{j}) d \pi (x,y) \Big) \label{eq2opti} \\ & = \inf_{\textbf{A} \in \mathcal{D}_{M^{\bar{\lambda}}_{\mu,\nu}}} \inf_{\pi \in \Pi(\mu,\nu)} \sum_{{1\leq i \leq d} \atop {1\leq j \leq d'}}  8\Big(a^{2}_{ij} - \int a_{ij}l_{\sqrt{\frac{\lambda}{d}}}(x_{i}) l_{\sqrt{\frac{\lambda}{d'}}}(y_{j}) d \pi (x,y) \Big) \nonumber \\ & = \inf_{\textbf{A} \in \mathcal{D}_{M^{\bar{\lambda}}_{\mu,\nu}}} 8 \norm{\textbf{A}}^{2}_{F} + \inf_{\pi \in \Pi(\mu,\nu)} \int c^{\bar{\lambda}}_{\textbf{A}}(x,y) d \pi (x,y), \nonumber
    \end{align}
    where $\bar{\lambda} = (\sqrt{\lambda/d}, \sqrt{\lambda/d'})$ and $c^{\bar{\lambda}}_{\textbf{A}}: (x,y) \in \mathbb{R}^{d}_{\geq 0} \times \mathbb{R}^{d'}_{\geq 0} \mapsto -8l_{\sqrt{\lambda/d}}(x)^{T}\textbf{A}l_{\sqrt{\lambda/d'}}(y)$. The optimization in $\textbf{A}$ can be made unconstrained, however, the optimal $a_{ij}$ in (\ref{eq2opti}) is achieved at $\frac{1}{2}\int l_{\sqrt{\lambda/d}}(x_{i}) l_{\sqrt{\lambda/d'}}(y_{j}) d \pi (x,y) \in [0, M^{\bar{\lambda}}_{\mu,\nu}/2]$ (due to Cauchy-Schwarz inequality), which enables us to restrict the optimization to $\mathcal{D}_{M^{\bar{\lambda}}_{\mu,\nu}}$.

    A simple parametrization can result in uniform $\lambda$-thresholding over the two spaces. Specifically, for the threshold $(d \vee d')\lambda^{2}$, we achieve the desired upper bound to $F_{2}$, as in (\ref{F2bar}), satisfying 
    \begin{equation}
        \bar{F}_{2}(\mu, \nu; (d \vee d')\lambda^{2}) = \inf_{\textbf{A} \in \mathcal{D}_{M^{\lambda}_{\mu,\nu}}} 8 \norm{\textbf{A}}^{2}_{F} + \textrm{OT}_{c^{\lambda}_{\textbf{A}}}(\mu, \nu). \label{optlamb}
    \end{equation}
    Given an optimal coupling $\pi^{*}_{\textbf{A}}$ for $\textrm{OT}_{c^{\lambda}_{\textbf{A}}}$, a solution $\textbf{A}^{*}$ achieving the infimum in (\ref{optlamb}) can be expressed as $\textbf{A}^{*} = \frac{1}{2}\int l_{\lambda}(x) l_{\lambda}(y)^{T} d \pi^{*}_{\textbf{A}^{*}} (x,y)$. The associated optimal value of the upper bound becomes
    \begin{equation*}
        \bar{F}_{2} = -2 \int \Big\langle l_{\lambda}(x), l_{\lambda}(x')\Big\rangle \Big\langle l_{\lambda}(y) , l_{\lambda}(y')\Big\rangle d \pi^{*}_{\textbf{A}^{*}} \otimes \pi^{*}_{\textbf{A}^{*}} (x,y,x',y').
    \end{equation*}
\end{proof}
While it is intuitive that a sample-level robustification is sufficient for LR, Proposition \ref{propIGW} additionally provides the deterministic extent to which the associated cost may propagate. The complexity related to computing such a bound shrinks down to that of an OT. To observe that, by denoting $(d \vee d')\lambda^{2} = \tilde{\lambda}$, let us write
    \begin{equation*}
        \bar{F}_{2}(\mu, \nu; \tilde{\lambda}) = \inf_{\textbf{A} \in \mathcal{D}_{M^{\lambda}_{\mu,\nu}}} 8 \norm{\textbf{A}}^{2}_{F} + \textrm{OT}_{c^{\lambda}_{\textbf{A}}}(\mu, \nu) = \inf_{\textbf{A} \in \mathcal{D}_{M^{\lambda}_{\mu,\nu}}} U^{\mu, \nu}_{\lambda}(\textbf{A}).
    \end{equation*}
Now, given any other $\tilde{\mu} \in \mathcal{P}_{4}(\mathbb{R}^{d}_{\geq 0})$ and $\tilde{\nu} \in \mathcal{P}_{4}(\mathbb{R}^{d'}_{\geq 0})$, Proposition \ref{propIGW} ensures the existence of $\textbf{A}, \tilde{\textbf{A}} \in \mathcal{D}_{M^{\lambda}_{\mu,\nu}}$ such that $\bar{F}_{2}(\mu, \nu; \tilde{\lambda}) = U^{\mu, \nu}_{\lambda}(\textbf{A})$ and $\bar{F}_{2}(\tilde{\mu}, \tilde{\nu}; \tilde{\lambda}) = U^{\mu, \nu}_{\lambda}(\tilde{\textbf{A}})$. As such, by optimality
    \begin{align}
        \abs{\bar{F}_{2}(\mu, \nu; \tilde{\lambda}) - \bar{F}_{2}(\tilde{\mu}, \tilde{\nu}; \tilde{\lambda})} &\leq \abs{U^{\mu, \nu}_{\lambda}(\textbf{A}) - U^{\tilde{\mu}, \tilde{\nu}}_{\lambda}(\textbf{A})} + \abs{U^{\mu, \nu}_{\lambda}(\tilde{\textbf{A}}) - U^{\tilde{\mu}, \tilde{\nu}}_{\lambda}(\tilde{\textbf{A}})} \nonumber \\ & = \abs{\textrm{OT}_{c^{\lambda}_{\textbf{A}}}(\mu, \nu) - \textrm{OT}_{c^{\lambda}_{\textbf{A}}}(\tilde{\mu}, \tilde{\nu})} + \abs{\textrm{OT}_{c^{\lambda}_{\tilde{\textbf{A}}}}(\mu, \nu) - \textrm{OT}_{c^{\lambda}_{\tilde{\textbf{A}}}}(\tilde{\mu}, \tilde{\nu})} \nonumber \\ &\leq 2 \sup_{\textbf{A} \in \mathcal{D}_{M^{\lambda}_{\mu,\nu}}} \abs{\textrm{OT}_{c^{\lambda}_{\textbf{A}}}(\mu, \nu) - \textrm{OT}_{c^{\lambda}_{\textbf{A}}}(\tilde{\mu}, \tilde{\nu})}. \label{stab}
    \end{align}
Since $\textrm{OT}_{c^{\lambda}_{\textbf{A}}}$ essentially calculates the transportation cost between truncated observations from $\mu$ and $\nu$ --- which offers finer control over extreme values --- LR turns out to achieve arbitrary accuracy in finding a robust surrogate to the GW cost. By plugging in the empirical distributions $\tilde{\mu} = \hat{\mu}_{n}$ and $\tilde{\nu} = \hat{\nu}_{n}$ in the stability bound (\ref{stab}), one can also comment on the sample complexity of $\bar{F}_{2}$. First, observe that $M^{\lambda}_{\mu,\nu} \lesssim \lambda^{2}$ and based on the truncation, the measurable cost $c^{\lambda}_{\textbf{A}}$ is absolutely bounded. This narrows down the search for dual potentials ($\phi : \mathbb{R}^{d}_{\geq 0} \rightarrow \mathbb{R}$) corresponding to $\textrm{OT}_{c^{\lambda}_{\textbf{A}}}$ to the class $\mathcal{F}_{\lambda} \coloneqq \bigcup_{\textbf{A} \in \mathcal{D}_{M^{\lambda}_{\mu,\nu}}} \mathcal{F}_{\textbf{A},\lambda}$ such that
\begin{equation*}
    \mathcal{F}_{\textbf{A},\lambda} \coloneqq \Big\{\phi \:|\: \exists\: \psi : \mathbb{R}^{d'}_{\geq 0} \rightarrow \mathbb{R} \:\textrm{such that}\: \phi = \psi^{c}; \norm{\phi}_{\infty}, \norm{\psi}_{\infty} \lesssim \lambda\Big\},
\end{equation*}
where $\psi^{c}$ (the $c$-transform of $\psi : \mathbb{R}^{d'}_{\geq 0} \rightarrow \mathbb{R}$ w.r.t. $c^{\lambda}_{\textbf{A}}$) is given by $\psi^{c} = \inf_{y}c^{\lambda}_{\textbf{A}}(\cdot, y) - \psi(y)$. As such, we can further upper bound (\ref{stab}) to obtain
\begin{equation*}
    \abs{\bar{F}_{2}(\mu, \nu; \tilde{\lambda}) - \bar{F}_{2}(\hat{\mu}_{n}, \hat{\nu}_{n}; \tilde{\lambda})} \leq 4 \sup_{\phi \in \mathcal{F}_{\lambda}} \abs{\int \phi\: d[\mu - \hat{\mu}_{n}]} + 4 \sup_{\psi \in \mathcal{F}^{c}_{\lambda}} \abs{\int \psi\: d[\nu - \hat{\nu}_{n}]},
\end{equation*}
where $\mathcal{F}^{c}_{\lambda} \coloneqq \bigcup_{\textbf{A} \in \mathcal{D}_{M^{\lambda}_{\mu,\nu}}} \mathcal{F}^{c}_{\textbf{A},\lambda}$ \citep{groppe2023lower}. This reduces the problem into controlling the two empirical processes over $\mathcal{F}_{\lambda}$ and $\mathcal{F}^{c}_{\lambda}$. The involvement of raw empirical measures $(\hat{\mu}_{n}, \hat{\nu}_{n})$, susceptible to outliers, makes further upper bounding the individual errors in terms of entropy only feasible in an $\mathcal{O} \cup \mathcal{I}$ framework. We identify this as a potential future work since it does not follow directly by adopting the approach of \citet{ma2023inference} into the framework of \citet{zhang2024gromov} [Theorem 3]. However, the trivial upper bound (see, Remark \ref{remROB}) along with properties such as resilience (Corollary \ref{corrniet}) and robust estimation of corresponding GW (Proposition \ref{Prop_up}) hold as a result of lemma \ref{ineqtukloc}. Nonetheless, it is not apparent how a penalization as $c^{\lambda}_{\textbf{A}}$ addresses shifts in mass allocation due to contamination. In this line, to motivate our upcoming formulation, let us first unify the underlying spaces under the general framework of \textit{probabilistic mm spaces}.

Probabilistic metric spaces $(\mathcal{X}, p_{X})$ are generalizations over typical metric spaces based on the deterministic `distance' $p_{X}$ following a modified triangle inequality 
\begin{equation*}
    T\{p_{X}(x,x')[0,s], p_{X}(x',x'')[0,t]\} \leq p_{X}(x,x'')[0,s+t],
\end{equation*}
where $T: \mathbb{R}_{+} \times \mathbb{R}_{+} \rightarrow \mathbb{R}_{+}$ \citep{bauer2024z}, for $x,x',x'' \in \mathcal{X}$ and $s,t \geq 0$. In case $p_{X}$ is replaced by the distribution function, specific choices of $T$ equate them to Menger spaces (e.g. taking $T= \min$ or $\max$) or Wald spaces ($T=\ast$, convolution) \citep{schweizer1960statistical}. Defining a measure on this collection $\mathcal{X}$ completes the triplet $(\mathcal{X}, p_{X}, \mu)$, which we call a probabilistic mm (pmm) space. In our setup, we particularly choose the pair $(\mathcal{X}=\bar{\mathcal{P}}_{p}(X), W_{p})$, where $\bar{\mathcal{P}}_{p} \subset \mathcal{P}_{p}$ is the collection of probability measures with full support on $X \subseteq \mathbb{R}^{d}$ having finite $p$-moments, $p \geq 1$. This reinforces the notion of generalization based on the fact that given $x,x' \in X$, we have $W_{1}(\delta_{x}, \delta_{x'}) = \textrm{OT}_{d_{X}}(\delta_{x}, \delta_{x'}) = d_{X}(x,x')$. The idea can similarly be extended to an alignment problem between distinct pmm spaces, which brings us back to GW. Observe that, considering a distance $\textrm{OT}_{l_{\lambda}(d_{X})}$ recovers our earlier LR formulation, as in lemma \ref{ineqtukloc}. We may alternatively choose the L\'{e}vy-Prokhorov (LP) metric ($\hat{\rho}$) to construct our pmm space since it ensures that $(\mathcal{X}, \hat{\rho})$ remains Polish, given that $(X,d_{X})$ is Polish (\citet{fristedt2013modern}, Chapter 18.7).

\begin{remark}[Localization leading to pmm spaces]
    Though Dirac measures are the easiest choice to show that individual $x$'s are represented in the pmm space, it is only a special case of a localized measure. Given $\alpha \in \bar{\mathcal{P}}(X)$, a localized measure $m^{L}_{\alpha}(x) \in \mathcal{P}(X)$ is tasked with preserving information about the neighborhood of $x \in X$ under the localization operator $L$\footnote{$L$ maps $\bar{\mathcal{P}}(X)$ to Markov kernels over $X$ \citep{memoli2019wasserstein}.}. Given the choice of the metric as $W_{p}$ in the pmm space, it implies that $W_{p}(m^{L}_{\alpha}(x), m^{L}_{\alpha}(x')) = d^{L}_{\alpha}(x,x')$: a generalization over the metric $d_{X}$. This is the precise reason we name our proposal of a robust GW based on robust $d_{X}$ `local robustification'.
\end{remark}

Given two of such pmm spaces $(\mathcal{X}, W_{p}, \mu)$ and $(\mathcal{Y}, W_{p}, \nu)$, the GW distance ($\mathcal{Z}$-GW according to \citet{bauer2024z}, Section 3.2.5) between them turns out as
\begin{equation} \label{pmmGW}
     d_{\textrm{GW}}(\mu, \nu) \coloneqq \Big(\inf_{\pi \in \Pi(\mu, \nu)} \int_{\mathcal{X} \times \mathcal{Y}} \int_{\mathcal{X} \times \mathcal{Y}} \abs{W_{p}(x,x') - W_{p}(y,y')}^{p'} d \pi \otimes \pi (x,y,x',y')\Big)^{\frac{1}{p'}}.
\end{equation}
For simplicity, we will always assume $p=p'$. However, it is not straightforward to imagine the ambient measures $\mu, \nu$ and the feasible couplings $\Pi$ they form. Let us look at an example that puts the problem in context.

\begin{example}[Gromov-Wasserstein between mixture of Gaussians]
    Consider the mm spaces $(\mathcal{N}(\mathbb{R}^{d}), W_{2}, \mu)$ and $(\mathcal{N}(\mathbb{R}^{d'}), W_{2}, \nu)$, where $\mathcal{N}(\mathbb{R}^{d})$ is the space of $d$-variate Gaussian distributions. Observe that
    \begin{enumerate}[(I)]
        \item since Gaussians are exactly identifiable based on their mean and covariance matrix, a finite Gaussian mixture $\in \mathcal{G}(\mathbb{R}^{d}) \coloneqq \bigcup_{k \geq 0}\mathcal{G}_{k}(\mathbb{R}^{d})$\footnote{$\mathcal{G}_{k}(\mathbb{R}^{d}) \coloneqq$ set of Gaussian mixtures with $\leq k$ components.} can be deemed a discrete probability distribution on $\mathcal{N}(\mathbb{R}^{d})$.
        \item Endowed with $W_{2}$, $\mathcal{N}(\mathbb{R}^{d})$ becomes Polish.
        \item Given $\alpha_{i} = \mathcal{N}(m_{i}, \Sigma_{i})$, $i=\{1,0\}$ due to \citet{dowson1982frechet}
        \begin{equation*}
            W^{2}_{2}(\alpha_{0}, \alpha_{1}) = \norm{m_{0} - m_{1}}^{2} + \textrm{tr}\Big[\Sigma_{0} + \Sigma_{1} - 2(\Sigma_{0}^{\frac{1}{2}}\Sigma_{1}\Sigma_{0}^{\frac{1}{2}})^{\frac{1}{2}}\Big].
        \end{equation*}
    \end{enumerate}
    Hence, for any $\alpha = \sum_{i=1}^{k}a_{i}\alpha_{i} \in \mathcal{G}_{k}(\mathbb{R}^{d})$ and $\beta = \sum_{j=1}^{l}b_{j}\beta_{j} \in \mathcal{G}_{l}(\mathbb{R}^{d'})$, there uniquely exist $\mu = \sum_{i=1}^{k}a_{i}\delta_{\alpha_{i}} \in \mathcal{P}(\mathcal{N}(\mathbb{R}^{d}))$ and $\nu = \sum_{j=1}^{l}b_{j}\delta_{\beta_{j}} \in \mathcal{P}(\mathcal{N}(\mathbb{R}^{d'}))$ respectively, where $a \coloneqq \{a_{i}\}^{k}_{i=1} \in \Delta_{k}$ and $b \coloneqq \{b_{j}\}^{l}_{j=1} \in \Delta_{l}$ (using (I)). (III) implies that, under such a $\mu$, $(\mathcal{N}(\mathbb{R}^{d}), W_{2}, \mu)$ has finite $2$-diameter, i.e. $\int_{\mathcal{N}(\mathbb{R}^{d}) \times \mathcal{N}(\mathbb{R}^{d})} W^{2}_{2}(x,x') d\mu(x) d\mu(x') < \infty$ (same holds for $(\mathcal{N}(\mathbb{R}^{d'}), W_{2}, \nu)$). As such, the corresponding $d_{\textrm{GW}}(\mu, \nu)$ (as in (\ref{pmmGW}), given $p=2$) exists and follows all metric properties of GW. \citet{salmona2024gromovwassersteinlike} show that solving the problem (\ref{pmmGW}) essentially boils down to 
    \begin{equation} \label{MGW2}
        \Big(\inf_{\pi \in \Pi(a,b)} \sum_{i,j,s,t} \abs{W_{2}(\alpha_{i},\alpha_{s}) - W_{2}(\beta_{j}, \beta_{t})}^{2} \pi_{i,j} \pi_{s,t}\Big)^{\frac{1}{2}},
    \end{equation}
    where $\Pi(a,b)$ is a subset of the simplex $\Delta_{k \times l}$ with marginals $a$ and $b$. This example can be further extended to general distribution classes $\subset \bar{\mathcal{P}}_{p}(\mathbb{R}^{d})$ owing to the fact that $\mathcal{G}(\mathbb{R}^{d})$ is dense in $\bar{\mathcal{P}}_{p}(\mathbb{R}^{d})$ for $W_{p}$, as long as they are complete and separable.
\end{example}

Evidently, observations from the distributions $\mu$ and $\nu$ can suffer arbitrary corruptions as before. In cases such as Example 1, one or more contaminated individual Gaussian components from either space may contribute to such corruption. To remedy $\epsilon$-contamination in the components, we replace $W_{p}$ with the smallest cost achieved by `optimally' removing $\epsilon$-mass from them. This extends our LR framework to alignment models concerning pmm spaces, endowed with $W_{p}$. Given $\alpha, \beta \in \mathcal{P}(X)$, it is defined as
\begin{equation} \label{robW}
    W^{\epsilon}_{p}(\alpha,\beta) = \inf_{{\alpha', \beta' \in \mathcal{P}(X):} \atop {\alpha \in \mathcal{B}_{\epsilon}(\alpha'), \beta \in \mathcal{B}_{\epsilon}(\beta')}}W_{p}(\alpha', \beta'),
\end{equation}
where $\mathcal{B}_{\epsilon}(\alpha) \coloneqq \{(1-\epsilon)\alpha + \epsilon\gamma: \gamma \in \mathcal{P}(X)\}$ denotes the $\epsilon$-Huber ball centered at $\alpha$ \citep{nietert2022outlier}. In this context, $\epsilon \in [0,1]$ signifies the radius of robustness, which when chosen distinctly for the two distributions generalizes the notion (i.e. $W^{\epsilon,\epsilon'}_{p}$). Remarkably, the dual formulation of $\textrm{OT}_{l_{\lambda}(d_{X})}$ can be derived as a special case of that of (\ref{robW}) (see, Appendix \ref{App1}). It also ties the threshold leading to truncation to the underlying optimization. Based on (\ref{robW}), we define the locally robust GW distance between pmm spaces (\ref{pmmGW}) as
\begin{equation} \label{RobpmmGW}
     d_{\textrm{LRGW}}(\mu, \nu; \epsilon) \coloneqq \Big(\inf_{\pi \in \Pi(\mu, \nu)} \int_{\mathcal{X} \times \mathcal{Y}} \int_{\mathcal{X} \times \mathcal{Y}} \abs{W^{\epsilon}_{p}(x,x') - W_{p}(y,y')}^{p} d \pi \otimes \pi (x,y,x',y')\Big)^{\frac{1}{p}},
\end{equation}
which also serves as a robust proxy to $\textrm{MGW}_{2}$ \citep{salmona2024gromovwassersteinlike} or the $\mathcal{Z}$-GW distance. While a generalization is imminent, the asymmetric robustness (only on space $\mathcal{X}$) in (\ref{RobpmmGW}) is specifically useful in cross-domain generative tasks, e.g. unpaired image-to-image translation. 

\begin{proposition}[Dependence on robustness radius] \label{rob_rad}
    For any $p \in [1,\infty)$ and $0 \leq \epsilon \leq \epsilon' \leq 1$ we have
    \begin{enumerate}[(i)]
        \item $d_{\textrm{LRGW}}(\mu, \nu; 0) = d_{\textrm{GW}}(\mu, \nu)$,
        \item $\abs{d_{\textrm{LRGW}}(\mu, \nu; \epsilon) - d_{\textrm{LRGW}}(\mu, \nu; \epsilon')} \lesssim \Big(\frac{\epsilon'-\epsilon}{1-\epsilon}\Big)^{\frac{1}{p}}$.
    \end{enumerate}
\end{proposition}

\begin{remark}[Local robustness based on L\'{e}vy-Prokhorov metric]
    The LP distance between $\alpha, \beta \in \mathcal{P}(X)$ is defined as 
    \begin{equation*}
        \hat{\rho}(\alpha, \beta) \coloneqq \inf\{\epsilon>0: \beta(A) \leq \alpha(A_{\epsilon})+\epsilon, \forall \:\textrm{Borel}\; A \subseteq X\},
    \end{equation*}
    where $A_{\epsilon} = \{x \in X: d_{X}(x,A) \leq \epsilon\}$ is the closed $\epsilon$-ball around $A$. It inherently carries mass allocation robustness, allowing free movement of an $\epsilon$-fraction mass between $\alpha$ and $\beta$. To observe the same, let us write LP in its alternative characterization due to Strassen's theorem (\citet{villani2021topics}, Section 1.4)
    \begin{equation}
        \inf_{\pi \in \Pi(\alpha,\beta)}\Big\{\inf\{\epsilon>0:\pi(\{(x,y):d_{X}(x,y)>\epsilon\})\leq \epsilon\}\Big\}.
    \end{equation}
    As such, it follows that $\abs{\hat{\rho}(\alpha, \beta) - \inf\{\epsilon>0:W^{1-\epsilon}_{\infty}(\alpha, \beta)\leq \epsilon\}}$ becomes arbitrarily small, given $X$ has unit diameter \citep{raghvendra2024robpar}. As a result, the feasibility of LR formulations equivalent to (\ref{RobpmmGW}) based on LP metrics is guaranteed. We show that local robustification using truncation, i.e. $l_{\lambda}(d_{X})$ also extends to pmm spaces under LP metric. It becomes evident due to the relation between $W_{p,\lambda}$ (ROBOT) and the modified LP.
    \begin{proposition} \label{levy}
        Define $\hat{\rho}_{\lambda}(\alpha, \beta)=\inf_{\pi \in \Pi(\alpha,\beta)}\Big\{\inf\{\epsilon>0:\pi(\{l_{\lambda}(d_{X}(x,y))>\epsilon\})\leq \epsilon\}\Big\}$, for $\lambda>0$. Then
        \begin{equation*}
            \frac{1}{1+\lambda}W_{1,\lambda} \leq \hat{\rho}_{\lambda} \leq \sqrt{W_{1,\lambda}}.
        \end{equation*}
    \end{proposition}
\end{remark}

\subsubsection{Sturm's GW and robust image-to-image translation}
Instilling intrinsic robustness to outliers in a measurable map (induced by neural networks) $\mathcal{X} \mapsto \mathcal{Y}$, learned based on contaminated data requires additional regularization. While introducing trimming methods (as in HGW or LRGW) in a GM setup may lead to denoised translations, the approximation capability of the maps thus produced remains shrouded. In this section, we rather connect natural upper bounds of GW to losses that fuel existing I2I models. This way, we ensure robustness guarantees without compromising complexity in an I2I translator. 

In our pursuit, let us first define Sturm's GW distance \citep{sturm2006geometry} between the altered spaces $(\mathcal{X}, l_{\lambda}(d_{X}), \mu)$ and $(\mathcal{Y}, l_{\lambda}(d_{Y}), \nu)$ as $\inf_{\tilde{d},\pi}||\tilde{d}||_{L^{p}(\pi)}$. Here, the infimum is over $\pi \in \Pi(\mu,\nu)$ and $\tilde{d} \in \mathscr{D}(l_{\lambda}(d_{X}), l_{\lambda}(d_{Y}))$, the set of \textit{metric couplings}\footnote{$\mathscr{D}(d_{X},d_{Y}) \coloneqq$ the set of metrics on $\mathcal{X} \sqcup \mathcal{Y}$ that extend $d_{X}$ and $d_{Y}$ \citep{sturm2006geometry}.}. Observe that, for $\lambda > 0$, $\mathscr{D}^{\lambda} \coloneqq l_{\lambda}(\mathscr{D}(d_{X},d_{Y})) \subset \mathscr{D}(l_{\lambda}(d_{X}), l_{\lambda}(d_{Y}))$. As such,
\begin{equation}\label{stur}
    \inf_{\tilde{d} \in \mathscr{D}(l_{\lambda}(d_{X}), l_{\lambda}(d_{Y})),\pi}||\tilde{d}||_{L^{p}(\pi)} \leq \inf_{\tilde{d} \in \mathscr{D}^{\lambda},\pi}||\tilde{d}||_{L^{p}(\pi)} \eqqcolon d_{RSGW}(\mu,\nu),
\end{equation}
which we call the $(p,\lambda)$-\textit{Robust Sturm's GW}, $p\in[1,\infty)$. The distance essentially embodies a locally robust formulation based on the couplings between $d_{X}$ and $d_{Y}$. We refer to the lower bound of (\ref{stur}) as the \textit{lower} RSGW. Complementing the relationship between GW and Sturm's GW, the respective robust formulations follow a similar inequality.

\begin{proposition}[Upper bound to LRGW] \label{Sturmupper}
    Given $p\in [1,\infty)$ and $\lambda \geq 0$,
    \begin{equation*}
        (p,\lambda)\textrm{-LRGW} \leq 2(p,\lambda)\textrm{-}l\textrm{RSGW}.
    \end{equation*}
    Also, for $\delta \in (0,\frac{1}{2}]$, whenever $(p,\lambda)$-LRGW$ \leq \delta^{5}$, we have $(p,\lambda)\textrm{-} l\textrm{RSGW} \lesssim \lambda(4\lambda+\delta)^{\frac{1}{p}}$.
\end{proposition}
The immediate benefit of such a result is that minimizing a realized loss of the RSGW-type arbitrarily in an I2I setup establishes a near-isometric relation between the two image spaces. Intuitively, this should produce robust translations that preserve geometry. To invoke the notion of an actual architecture, we recall the equivalent formulation of Sturm's GW. In our setup,
\begin{equation}\label{WtoGP}
    d_{RSGW}(\mu,\nu) = \inf_{d \in \mathscr{D}(d_{X},d_{Y}),\pi}\norm{l_{\lambda}(d)}_{L^{p}(\pi)} = \inf_{\mathcal{Z},\phi_{X},\phi_{Y}} W_{p,\lambda}({\phi_{X}}_{\#}\mu, {\phi_{Y}}_{\#}\nu),
\end{equation}
where the infimum is over all isometric embeddings $\phi_{X} : \mathcal{X} \rightarrow \mathcal{Z}$ and $\phi_{Y} : \mathcal{Y} \rightarrow \mathcal{Z}$ into a \textit{latent space} $\mathcal{Z}$, endowed with the metric $d$ (\citet{sturm2006geometry}, lemma 3.3). We deliberately bring on the term `latent space' to emphasize the connection to I2I architectures. The formulation also makes it sufficient to embed observations from both spaces into an optimal $\mathcal{Z}$ prior to truncation. Observe that if instead of $W_{p,\lambda}$, we deploy $\hat{\rho}_{\lambda}$ based on the metric $d$ in (\ref{WtoGP}), we obtain the robust Gromov-Prokhorov (RGP) metric (\citet{blumberg2014robust}, Section 2.5). By definition, RGP $< \epsilon$ implies the existence of a metric space $\mathcal{Z}$ with embeddings $\phi_{X}, \phi_{Y}$ into it, that satisfy $\hat{\rho}_{\lambda}({\phi_{X}}_{\#}\mu, {\phi_{Y}}_{\#}\nu) < \epsilon$.

\textbf{Equivalence of losses:} With the foundation in place, we explore the similarity between the loss (\ref{WtoGP}) and that of I2I translation models such as UNIT \citep{liu2017unsupervised} and GcGAN \citep{fu2019geometry}. We choose the two models based on their sustained relevance in the domain. However, the equivalence about to be shown can be extended to models that recognize the role of a latent space or deploy a cycle-consistency (CC) loss, such as DistanceGAN \citep{benaim2017one}, StarGAN \citep{Choi2018stargan} or MUNIT \citep{huang2018multimodal}. The cornerstone of successful I2I learning is inarguably the CC loss. In the population regime, it can be expressed as $W_{1}(\mu, {G \circ F}_{\#}\mu)$ for the space $\mathcal{X}$, where $F,G$ are optimized over measure-preserving (transport) maps parametrized using neural networks (NN). It becomes equivalent to optimizing the commonly used $L^{1}$ norm if $\mu$ possesses a H\"{o}lder smooth density \citep{chakrabarty2022translation}. 
\begin{equation}\label{diag}
	\begin{tikzcd}[row sep = huge]
		(\mathcal{X},\mu) \arrow[rr, "F", shift left] \arrow[dr, "\phi_{X}", swap, shift right]
		&&
		(\mathcal{Y},\nu) \arrow[ll, "G", shift left] \arrow[dl, "\phi_{Y}", swap, shift right]\\
		&
		(\mathcal{Z},\omega) \arrow[ur, "\phi'_{Y}", swap, shift right, dashed] \arrow[ul, "\phi''_{X}", swap, shift right, dashed]
	\end{tikzcd}
\end{equation}
Now, recognizing the existence of a shared latent space, we may construct $G = \phi''_{X} \circ \phi_{Y}$ and $F = \phi'_{Y} \circ \phi_{X}$, where $\phi'_{Y}: \mathcal{Z} \rightarrow \mathcal{Y}$ is the left-inverse of $\phi_{Y}$, and $\phi''_{X}: \mathcal{Z} \rightarrow \mathcal{X}$ is the right-inverse of $\phi_{X}$. We can assume them to be full functional inverses, as the same applies to isomorphic embeddings, in which case CC is achieved a.s. However, the maps $\phi''_{X}$, $\phi'_{Y}$ may not be measure-preserving in general. Therefore,
\begin{align}
    W_{1}(\mu, {G \circ F}_{\#}\mu) &= \inf_{\pi \in \Pi(\mu, F_{\#}\mu)} \int d_{X}(x,\phi''_{X} \circ \phi_{Y}(y)) \:d\pi (x,y) \nonumber \\ &= \inf_{\pi \in \Pi(\mu, F_{\#}\mu)} \int d\Big(\phi_{X}(x),(\phi_{X}\circ \phi''_{X}) \circ \phi_{Y}(y)\Big) \:d\pi (x,y) \label{SGW_e} \\ &= \inf_{\pi \in \Pi({\phi_{X}}_{\#}\mu, (\phi_{Y} \circ \phi'_{Y}) \circ {\phi_{X}}_{\#}\mu)} \int d(x, y) \:d\pi (x,y) \nonumber \\ &= \inf_{\pi \in \Pi(\omega, (\phi_{Y} \circ \phi'_{Y})_{\#}\omega)} \int d(x, y) \:d\pi (x,y) = W_{1}(\omega, {\phi_{Y} \circ \phi'_{Y}}_{\#}\omega) \label{auto},
\end{align}
where $d \in \mathscr{D}(d_{X},d_{Y})$\footnote{To avoid complications, we do not differentiate the two $W_{1}$ metrics in terms of notations, which are indeed calculated based on $d_{X}$ and $d$ respectively.}. We list out some immediate observations from the upper derivation. Firstly, constructing such a chaining ($\mathcal{X} \leftarrow \mathcal{Z} \leftarrow \mathcal{Y}$) reduces the problem of achieving CC in $\mathcal{X}$ to that of ensuring accurate autoencoding of $\omega$ based on the contextual latent law $\nu$ (\ref{auto}). The same choice of $F,G$ also guarantees CC in $\mathcal{Y}$ a.s. Observe that, for any $F$ satisfying $F_{\#}\mu = \nu$, given an optimal $\mathcal{Z}$ and the pair of embeddings into it, (\ref{SGW_e}) equates to SGW. As such, SGW is an upper bound to the optimal CC loss $\inf_{F,G}W_{1}(\mu, {G \circ F}_{\#}\mu)$ when $G$ follows our construction optimally, which is rather intuitive. It becomes much simpler if $\mu,\nu \in \mathcal{P}^{\textrm{ac}}_{2}(\mathbb{R}^{d})$, in which the construction can be made uniquely (see, Appendix \ref{chain}).

The second common loss component between UNIT and GcGAN is the constraint that ensures $F_{\#}\mu = \nu$ and $G_{\#}\nu = \mu$. Typically, the imposition is done using a GAN or WGAN objective. In our framework, a WGAN loss under $1$-Lipschitz critics turns out as
\begin{equation*}
    W_{1}(\mu, G_{\#}\nu) = \inf_{\pi \in \Pi(\mu,\nu)} \int d_{X}(x,G(y)) \:d\pi(x,y) = W_{1}({\phi_{X}}_{\#}\mu, {\phi_{Y}}_{\#}\nu),
\end{equation*}
which again at an optimal latent space equals SGW. Similarly, the loss $W_{1}(F_{\#}\mu,\nu)$ boils down to solving (\ref{auto}). As such, it is sufficient to optimize SGW between $\mu$ and $\nu$ subject to the autoencoding constraint (\ref{auto}) to solve the UNIT problem. 

The only additional term GcGAN employs is namely the geometric-consistency (GC) loss. In the population regime, a $\mathcal{X} \xrightarrow{F} \mathcal{Y}$ translation model incurs a GC loss
\begin{equation*}
    W_{1}({F \circ s_{X}}_{\#}\mu, s_{Y}\circ F_{\#}\mu),
\end{equation*}
where $s_{X}$ and $s_{Y}$ are automorphisms in $\mathcal{X}$ and $\mathcal{Y}$ respectively, e.g. rotation. Based on our construction, considering $s_{X} = \phi''_{X} \circ \phi_{X}$ and $s_{Y} = \phi'_{Y} \circ \phi_{Y} = \textrm{Id}_{Y}$ meets the constraint. Combining all the above observations gives the clear impression that effectively choosing a latent space $\mathcal{Z}$ --- in turn, enabling appropriate construction of $F$ and $G$ --- implies consistent I2I translation in UNIT and GcGAN. Remarkably, all the results hold exactly under the altered metric $W_{1,\lambda}$ (also, $W^{\epsilon}_{1}$) since the constructions remain same for $l_{\lambda}(d_{X})$ and $l_{\lambda}(d_{Y})$ (\ref{WtoGP}). As such, robustifying UNIT or GcGAN only requires updating their dependence on SGW to one with RSGW. 
    


\subsection*{Experiment: Style transfer with noise}
The first experiment we conduct tests the denoising capability of a robust GcGAN deploying (\ref{robW}) during I2I style transfer. Despite an overhaul in the optimization, we call our proposed model `robust GcGAN' for simplicity. Notably, this is the first outlier-robust cross-domain generative model to our knowledge. Based on the dataset `Apples-Oranges' \citep{zhu2017unpaired}, the underlying task is to translate the visual style of oranges onto apples that are contaminated. Unlike the Huber setup here, standard Gaussian noise is added to the RGB channels of each target sample (apple). The mixing intensity $\alpha$ is kept $0.2$. We present a detailed discussion on the experimental setup in Appendix \ref{Ablation_LR}. As discussed in the previous section, it is sufficient to optimize the RSGW loss for a suitable $\mathcal{Z}$, which in this case is the image space itself. As a regularizer, we add the GC loss taking $s_{X}, s_{Y}$ as $90^{\circ}$ clockwise rotations in their respective spaces. Model architecture and the choice of the Lagrangian parameter remain similar to that directed by the GcGAN authors.
\begin{figure}[H]
    \centering
    \includegraphics[width=\linewidth]{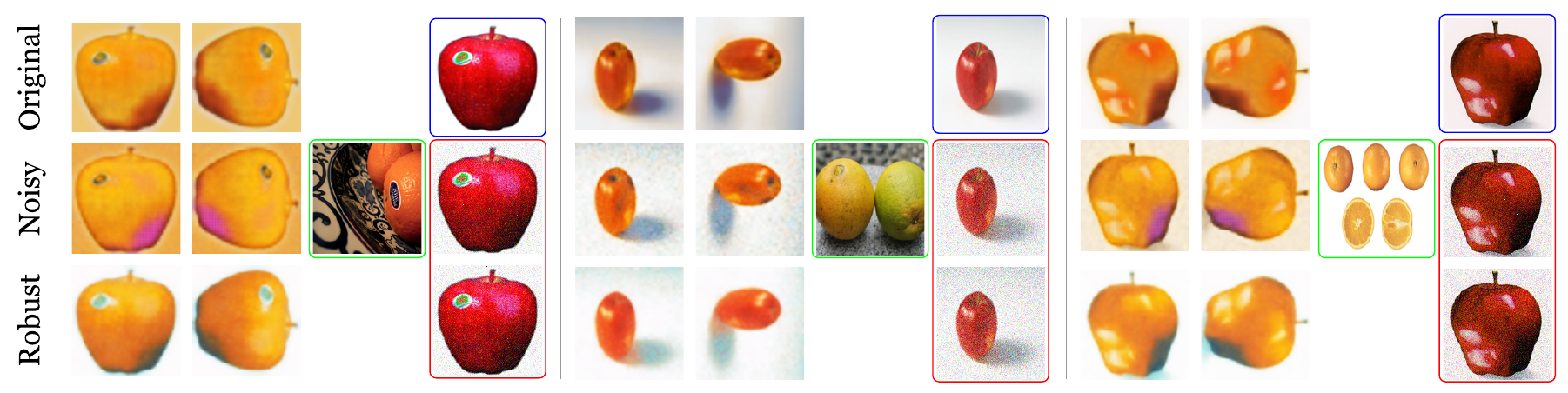}
    \caption{Style transfer performance of robust GcGAN under contamination ($\alpha = 0.3$). Images encircled in `\textcolor{blue}{blue}' represent clean target samples, in `\textcolor{red}{red}' are noisy versions of them and the ones in `\textcolor{green}{green}' act as sources of the style to be transferred. At $\epsilon = 0.5$, the robust translations (third row) maintain sharpness and prevent artifacts from appearing, improving the FID score to $152.65$ (compared to $154.74$ in the noiseless setting: first row).}
    \label{fig:RUNIT_AO}
\end{figure}
For comparison, the experiment contains three phases. As in the first row of Figure \ref{fig:RUNIT_AO}, we generate samples using the original GcGAN (without any modifications) on clean observations (control). The second row shows the degradation in translation once noise is added. Finally, applying our robust formulation at $\epsilon=0.5$ we observe a significant improvement in images, both qualitative and quantitative. We present our parameter selection scheme in Appendix \ref{Ablation_LR} in the form of an ablation study. 

\begin{figure}[H]
  \begin{minipage}[c]{0.4\textwidth}
    \includegraphics[width=\textwidth]{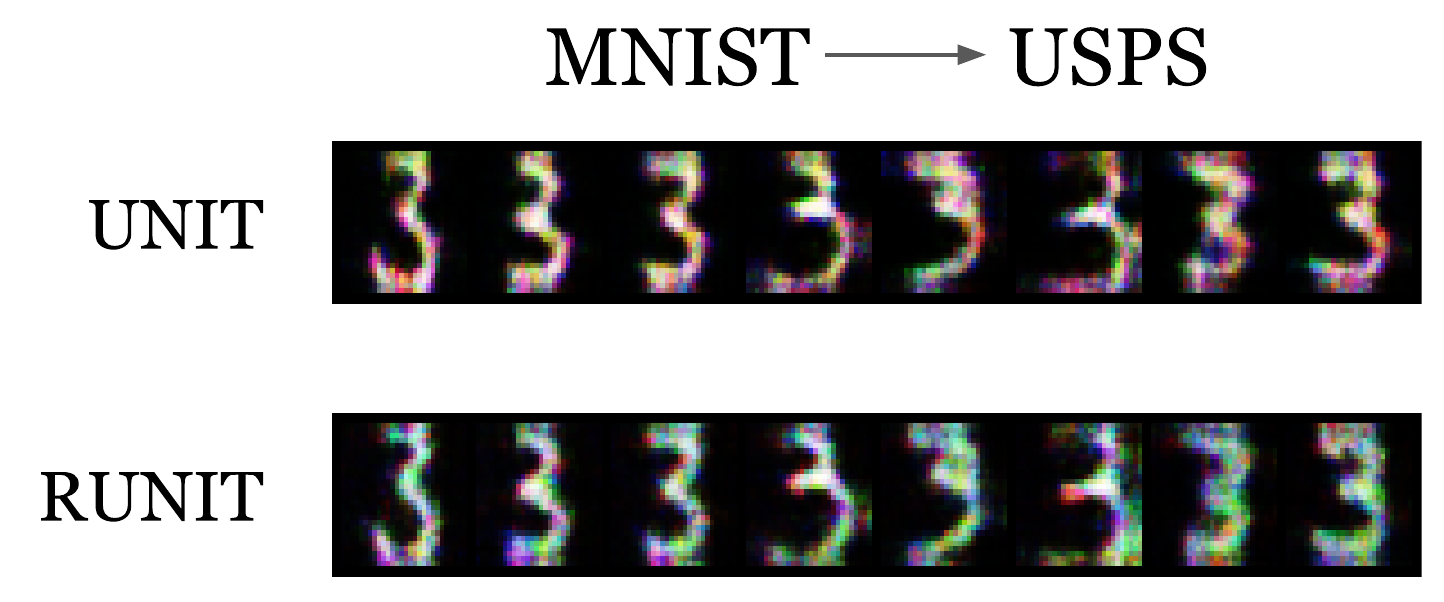}
  \end{minipage}\hfill
  \begin{minipage}[c]{0.58\textwidth}
    \caption{\hspace{4pt}Unpaired translation under contamination ($\alpha = 0.4$) using robust UNIT. At $\epsilon=0.5$, RUNIT recovers the visual quality of generated USPS samples (FID = $262.48$, compared to $304.39$ in case of UNIT under pixel noise).} 
    \label{Fig:RUNIT_MU}
  \end{minipage}
\end{figure}
The second experiment is of a different spirit in terms of the dissimilarity between dimensions of the two spaces, namely handwritten digit datasets MNIST ($28 \times 28$) and USPS ($16 \times 16$). Our goal lies in checking the robust domain translation capacity of a UNIT architecture reinforced with RSGW. Unlike style transfer, here, samples from MNIST (base distribution) are subjected to Gaussian noise. Keeping the robustness radius at $0.5$, the robust UNIT model produces USPS samples with improved FID score (see, Figure \ref{Fig:RUNIT_MU}) compared to vanilla UNIT. Besides the expected denoising, the heightened image quality is a result of employing WGAN instead of vanilla GAN regularization.

\subsection{Plan Robustification} \label{Part3}

All existing efforts to make GW robust to outliers rely on penalizing the plan $\pi$ based on partial mass transport. Relieving the constraint from aligning all points enables the optimization to filter out outliers as only their contribution to the total mass is ignored. In a GW setup, unbalancing \citep{sejourne2021unbalanced, tran2023unbalanced} in principle achieves the same by relaxing $\Pi(\mu,\nu)$ to $\mathcal{M}_{+}(\mathcal{X} \times \mathcal{Y})$, i.e. the set of positive Radon measures. However, it does not restrict the amount of mass to be pruned in its imposition of marginal constraints under quadratic $\phi$-divergences. Moreover, it fails to guarantee an optimal plan in the exact sense. While a rescaling afterward produces a joint probability distribution, it only redistributes the leftover mass to inlying points from both spaces uniformly. Using the TV metric instead connects the problem to PGW \citep{bai2024partialgromovwassersteinmetric}. It readily carries out the redistribution, which ties the idea to our previous truncation methods (TGW and LRGW). While the redistribution pathway in TGW is not uniform, it is directed to the points, which through their interactions with the other space, result in $\tau > 0$ distortion. In LR, points almost $\lambda>0$ apart receive the mass.

Though imposed on top of an unbalanced GW between surrogates of $\mu$ and $\nu$, \citet{kong2024outlier} is the only method to date that employs a direct penalization to robustify in the spirit of robust OT \citep{balaji2020robust}. Since our motivation lies in constructing a robust I2I translation architecture, we rather prioritize a balanced formulation of the same kind that results in an optimal plan, and eventually maps. As encouragement, we draw an immediate connection between the robust penalization and ROT \citep{le2021robust}. Observe that, the $(4,2)$-GW distance between Euclidean mm spaces can be fragmented as $d^{2}_{\textrm{GW}}(\mu, \nu) = S_{1} + S_{2}$, where
\begin{align*}
    S_{2}(\mu,\nu) = \inf_{\pi \in \Pi(\mu,\nu)}\Gamma(\pi) \coloneqq \inf_{\pi \in \Pi(\mu,\nu)}\int -\norm{x}^{2}\norm{y}^{2} d\pi(x,y) -2\sum_{{1\leq i \leq d} \atop {1\leq j \leq d'}} \Big(\int x_{i}y_{j} d\pi(x,y)\Big)^{2},
\end{align*}
and $S_{1}$ depends solely on the marginals $\mu, \nu$ \citep{zhang2024gromov}. It enables us to define
\begin{align}\label{S_2}
    \tilde{S}_{2}(\mu,\nu) = \inf_{{\tilde{\mu} \in \mathcal{P}(\mathcal{X})} \atop {\tilde{\nu} \in \mathcal{P}(\mathcal{Y})}} \min_{\pi \in \Pi(\tilde{\mu} ,\tilde{\nu})} \Gamma(\pi) + \lambda_{1} D_{f}(\tilde{\mu} | \mu) + \lambda_{2} D_{f}(\tilde{\nu} | \nu),
\end{align}
where $D_{f}$ is some $f$-divergence and $\lambda_{1}, \lambda_{2} > 0$.

\begin{lemma}[Duality] \label{dual_3}
    Given $(\mu,\nu) \in \mathcal{P}_{4}(\mathbb{R}^{d}) \times \mathcal{P}_{4}(\mathbb{R}^{d'})$, if $D_{f} \equiv d_{\textrm{KL}}$ we have 
    \begin{equation*}
        \tilde{S}_{2}(\mu,\nu) = \inf_{{\tilde{\mu} \in \mathcal{P}(\mathcal{X})} \atop {\tilde{\nu} \in \mathcal{P}(\mathcal{Y})}} \inf_{\textbf{A} \in \mathcal{D}_{M_{\tilde{\mu},\tilde{\nu}}}} 8 \norm{\textbf{A}}^{2}_{F} + \textrm{ROT}_{c_{\textbf{A}}}(\mu, \nu),
    \end{equation*}
    where $\mathcal{D}_{M_{\tilde{\mu},\tilde{\nu}}} = [-M^{\infty}_{\tilde{\mu},\tilde{\nu}}/2, M^{\infty}_{\tilde{\mu},\tilde{\nu}}/2]^{d \times d'}$ (see, Proposition \ref{propIGW}) and the cost $c_{\textbf{A}}$ maps $(x,y) \mapsto -\norm{x}^{2}\norm{y}^{2} -8x^{T}\textbf{A}y$.
\end{lemma}
As such, the optimization underlying robust alignment is essentially a moment-constrained robust OT. If we only regularize based on one marginal (e.g. $\mu$ only), the optimization boils down to solving a semi-constrained problem, namely RSOT. The GW estimate corresponding to such a $\tilde{S}_{2}(\mu,\nu)$ can be made robust by plugging in the optimal marginals so obtained, $\tilde{\mu},\tilde{\nu}$ in the functional $S_{1}$. This marks the potential the robust penalization (\ref{S_2}) has. In search of learnable maps between the spaces, we may narrow down the feasible set of couplings following \citet{hur2024reversible}. Instead of $\Pi(\tilde{\mu} ,\tilde{\nu})$, we may choose the path-restricted distributions that follow the \textit{binding constraint}
\begin{equation*}
    \{\pi : \pi = (\textrm{Id},F)_{\#}\tilde{\mu} = (G,\textrm{Id})_{\#}\tilde{\nu}\} \subset \Pi(\tilde{\mu} ,\tilde{\nu}),
\end{equation*}
where $F,G$ are measurable maps between $\mathcal{X}, \mathcal{Y}$ (see, illustration \ref{diag}) and $\tilde{\mu} ,\tilde{\nu}$ are supposedly `clean' marginals. One can also relax this by choosing a larger subclass of $\Pi$ that imposes only $F_{\#}\tilde{\mu}= \tilde{\nu}$ and $G_{\#}\tilde{\nu}= \tilde{\mu}$. In any case, the resultant robust distance--- an upper bound to GW under a similar penalization (\citet{hur2024reversible}, Proposition 5.4)--- only inculcates maps that transport inlying marginals to that in the other space. In a Huber contamination model, this readily implies that the corresponding set of couplings is non-empty. However, it is in principle different from learning a map possessing a denoising ability in the sense $F_{\#}\mu = \tilde{\nu}$. Maps of the latter kind promote carrying out the optimization over couplings given as
\begin{equation*}
    \{\pi : \pi = (\tilde{\textrm{Id}},F)_{\#}\mu = (G,\tilde{\textrm{Id}})_{\#}\nu\},
\end{equation*}
where $\tilde{\textrm{Id}}$ denotes the denoising operation that pushes forward $\mu$ to $\tilde{\mu}$ (also $\nu$ in its ambient space). The set containing such couplings is also non-empty since partial mass transport guarantees the existence of such `robustifiers' $\tilde{\textrm{Id}}$, and hence a pair of amenable maps $(F,G)$. Given $\epsilon \in [0,1)$, let us define the partial couplings
\begin{equation} \label{par}
    \Pi_{\epsilon}(\mu ,\nu) = \{\pi : \pi = (\tilde{\textrm{Id}}^{\epsilon},F)_{\#}\mu = (G,\tilde{\textrm{Id}}^{\epsilon})_{\#}\nu\},  
\end{equation}
where $\tilde{\textrm{Id}}^{\epsilon}_{\#}\alpha \leq (1-\epsilon)\alpha$, for $\alpha \in \mathcal{P}(\mathcal{X})$. The inequality should be understood setwise. We can also generalize the notion based on distinct mass fractions to be clipped in the two spaces. It is (\ref{par}) that we base our final proposition on constructing a robust I2I translation model. We call the term $\inf_{\pi \in \Pi_{\epsilon}(\mu ,\nu)} ||d_{X} - d_{Y}||_{L^{p}(\pi \otimes \pi)}$, the \textit{robust reversible Gromov-Monge} (RRGM) distance. The formulation essentially is a robust surrogate to the RGM distance due to \citet{hur2024reversible} based on a partial alignment. To favor comparative analysis, we only consider $p=2$ in our experiments. During transform sampling, RGM uses additional penalization imposing the constraints of measure preservation. The preferred metric for the same is often chosen as Maximum Mean Discrepancy (MMD). To eliminate the critical question of the ideal kernel given an empirical problem, we employ instead $W_{1}$. Under the same, the RRGM loss\footnote{The loss (\ref{RRGM_lag}) relaxes the binding constraint (as in (\ref{par})) and only imposes robust measure preservation. As such, it is essentially a lower bound to RRGM.} can be written in a Lagrangian form given as
\begin{align} \label{RRGM_lag}
    \inf_{F,G} \Big[\int \Big(d_{X}(\tilde{\textrm{Id}}^{\epsilon}(x), G(y)) - d_{Y}(\tilde{\textrm{Id}}^{\epsilon}(y),F(x))\Big)^2 d\mu \otimes \nu\Big]^{\frac{1}{2}} + \lambda_{1}W_{1}(\tilde{\textrm{Id}}^{\epsilon}_{\#}\mu, G_{\#}\nu) + \lambda_{2}W_{1}(F_{\#}\mu, \tilde{\textrm{Id}}^{\epsilon}_{\#}\nu),
\end{align}
where the infimum is over measurable maps and $\lambda_{1}, \lambda_{2} > 0$. We may find a further lower bound to the loss owing to the fact that
\begin{equation} \label{furth_low}
    W_{1}(\tilde{\textrm{Id}}^{\epsilon}_{\#}\mu, G_{\#}\nu) \geq W^{\epsilon}_{1}(\mu, G_{\#}\nu),
\end{equation}
where $W^{\epsilon}_{1}$ only carries out a partial transport of $\mu$ asymmetrically (see, (\ref{robW})). The same argument holds for the other term, given an asymmetric $W^{\epsilon}_{1}$ employment on the other space. During demanding I2I translations, it is often beneficial to have learnable discriminators over $1$-Lipschitz dual maps. The usage of $W_{1}$ is also advantageous since it enables deploying a larger class of neural network-induced critics. As such, the two added losses can be optimized using WGAN-GP \citep{gulrajani2017improved} architectures. Despite promoting a different redistribution pathway of inlying mass, the sample complexity of transform sampling under $W^{\epsilon}_{1}$ should be of a similar order to LR constraints (\citet{nietert2023robust}, Theorem 4). We note that the convergence rate corresponding to the latter depends only on the inlying sample size (see, Appendix \ref{sample_comp}) if outliers remain bounded in number.

\subsection*{Experiment: Image-to-image translation with noise}
We test RRGM in a noisy MNIST$\leftrightarrow$USPS domain translation experiment. The contamination regime for the experiment remains the same as that in RUNIT. Maintaining $\alpha = 0.4$, we randomly select pixel locations following a Gaussian law and set their values to $1.0$ (bright white) for all channels, adding visible outlier points to the image. Since handwritten digit images have all information regarding the numerical in the shape of white pixels, this contamination becomes quite challenging for an I2I model. For example, CycleGAN \citep{zhu2017unpaired} performs poorly despite employing a cycle-consistency component and generative losses in both directions. In contrast, RRGM (based on (\ref{furth_low}) with $\epsilon = 0.5$) under $\lambda_{i} = 0.2$; $i=1,2$ generates significantly sharper and denoised samples (Figure \ref{fig:RRGM}). For a fair comparison, we also present both clean and noisy samples to the discriminators in CycleGAN. Even if the discriminators are shown noisy observations only, the generation quality of RRGM surpasses that of CycleGAN. As a reference, we maintain a similar parameter selection for RGM \citep{hur2024reversible}, which deploys an additional MMD loss to impose the binding constraint. However, in the absence of dedicated critic modules, it lags behind. Our model outperforms both techniques by a significant margin, in terms of both quantitative and qualitative measures. 

\begin{figure}[H]
    \centering
    \begin{subfigure}{0.48\linewidth}
        \includegraphics[width=\linewidth]{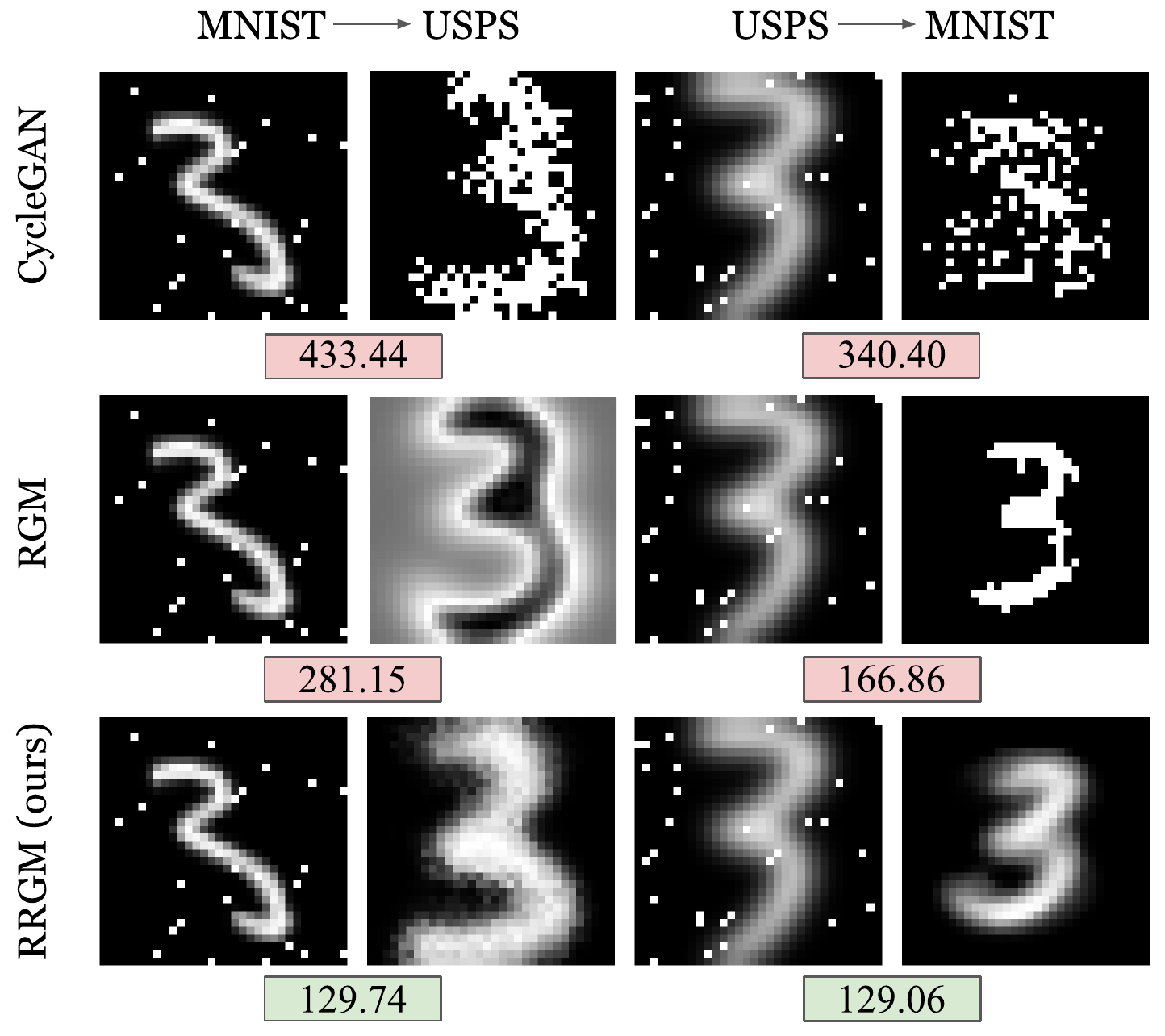}
        \caption{}
        \label{fig:RRGM}
    \end{subfigure}
    \hspace{8pt}
    \begin{subfigure}{0.42\linewidth}
      \includegraphics[width=\linewidth]{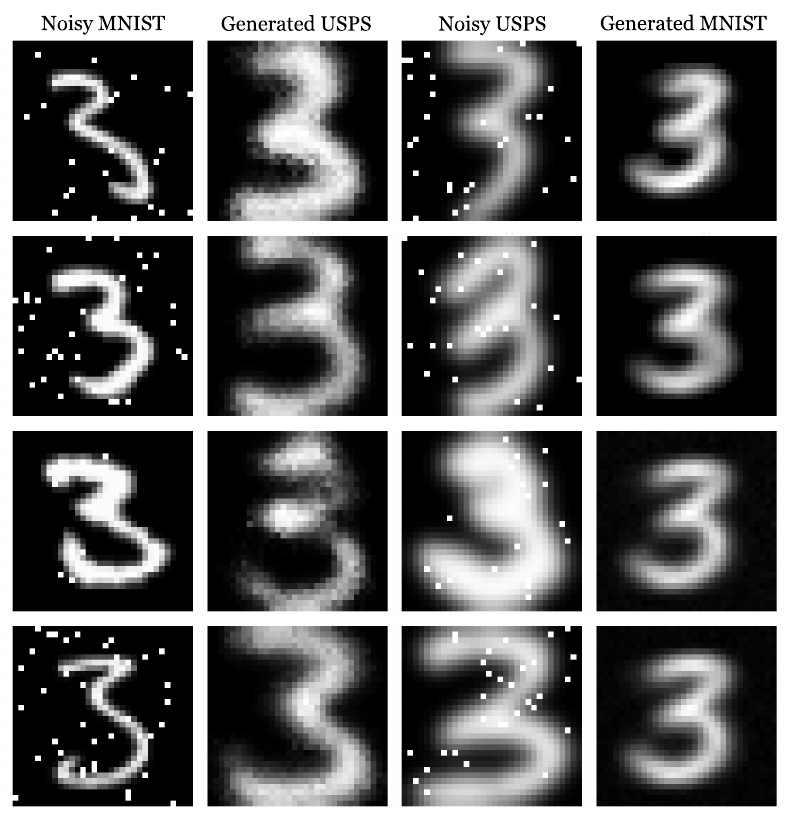}
      \caption{}
      \label{fig:RRGM_samples}
    \end{subfigure}
    \caption{(a) FID scores corresponding to robust cross-domain generations between USPS and MNIST data under Gaussian contamination, using CycleGAN, RGM, and RRGM (ours). (b) Denoised generated samples using RRGM in both domains.} 
    \label{fig:RRGM_comp}
\end{figure}

\section{Discussion}
We explore three major possibilities for the robustification problem in a GW setup. Drawing from classical techniques in robust statistics, we propose novel GW surrogate distances TGW (and HGW) and LRGW that limit contamination due to outliers. We study their interrelation and their respective dependence on the truncation parameters. For TGW, we comment on its population-level robust guarantees and the resilience it offers to underlying distributions. Based on a data-dependent parameter selection scheme, we present a working algorithm to solve the HGW distance, which exhibits superior protection against outliers compared to existing methods in shape-matching tasks. On the other hand, solving LRGW-type measures boils down to calculating an OT loss between trimmed samples from the underlying distributions. It also hints at the effective level of trimming necessary ($\lambda$) given a sample problem. We generalize our notion to probabilistic mm spaces, which allows one to define LR alignment between mixtures of distributions of distinct dimensions. We extend this setup to introduce robust image translation networks that surpass existing benchmarks. We also propose in RRGM a cross domain transform sampling framework that is robust to outliers. It promotes robust concentration and uniform deviation bounds besides significantly improving image quality in I2I translations.

\hspace{8pt}
\appendix
\section{Technical details and proofs} \label{App1}

\subsection*{Proof of Proposition \ref{Prop_up}}
    Triangle inequality of $d_{X}$ implies that given $(x,y), (x',y') \sim \mu' \otimes \nu$
    \begin{equation*}
        2\Lambda_{1}(d_{X},d_{X}) = \abs{d_{X}(x,x') - d_{X}(y,y')} \leq d_{X}(x,y) + d_{X}(x',y').
    \end{equation*}
    Thus, for a coupling $\pi \in \Pi(\mu', \nu)$, the triangle inequality of the norm $\norm{\cdot}_{\mathcal{T}_{p}}$ implies
    \begin{equation} \label{tri_norm}
        \norm{2\Lambda_{1}}_{\mathcal{T}_{p}(\pi \otimes \pi)} \leq 2\norm{d_{X}}_{\mathcal{T}_{p}(\pi)}.
    \end{equation}
    Now,
    \begin{align}
        \inf_{\pi \in \Pi} \norm{d_{X}}_{\mathcal{T}_{p}(\pi)} &= \inf_{\pi \in \Pi} \left(\int \mathcal{T}_{p}(d_{X}(x,y)) d\pi\right)^{\frac{1}{p}} = W_{\mathcal{T}_{p}}(\mu', \nu) \nonumber \\ & \leq W_{\mathcal{T}_{p}}((1-\epsilon)\mu + \epsilon \mu_{c}, \mu) + W_{\mathcal{T}_{p}}(\mu, \nu) \label{tri_tuk} \\ &\leq W_{\mathcal{T}_{p}}(\epsilon\mu, \epsilon\mu_{c}) + W_{\mathcal{T}_{p}}(\mu, \nu) \label{eps} \\ &\leq \epsilon^{\frac{1}{p}} W_{\mathcal{T}_{p}}(\mu, \mu_{c}) + W_{\mathcal{T}_{p}}(\mu, \nu) \nonumber \\ &\leq \tau\epsilon^{\frac{1}{p}} + W_{\mathcal{T}_{p}}(\mu, \nu), \label{evi}
    \end{align}
    where (\ref{eps}) is due to the fact that for any $\alpha, \beta \in \mathcal{P}(\mathcal{X})$, we have $\textrm{OT}_{d_{X}}(\alpha, \beta) \leq \textrm{OT}_{d_{X}}(\alpha - \alpha \wedge \beta, \beta - \alpha \wedge \beta)$. The triangle inequality of $W_{\mathcal{T}_{p}}$ (\ref{tri_tuk}) follows a similar proof to that of the $p$-Wasserstein distance. Inequality (\ref{evi}) is due to the trivial upper bound of the Tukey function. This along with the previous observation completes the proof.

\subsection{Relation between $W^{\epsilon}_{p}$ and truncated OT} \label{duality}
Given $\alpha, \beta \in \mathcal{P}(X)$ such that $X$ is compact, the dual formulation of $W^{\epsilon}_{p}$ for $p \in [1,\infty)$ becomes (\citet{nietert2022outlier}, Theorem 2)
\begin{equation} \label{dualWeps}
    (1-\epsilon)W^{\epsilon}_{p}(\alpha, \beta)^{p} = \sup_{\phi \in C_{b}(X)} \int \phi \:d \alpha - \int \phi^{c} \:d \beta - \epsilon\: \textrm{Range}(\phi),
\end{equation}
where $C_{b}(X) \coloneqq \{f: X \rightarrow \mathbb{R} : f \;\textrm{is continuous} , \:\norm{f}_{\infty} < \infty\}$ and $\phi^{c}$ denotes the $c$-transform of $\phi$ w.r.t. the cost $d_{X}(\cdot,\cdot)^{p}$. On the other hand, the Kantorovich potential $\phi$ that solves $\textrm{OT}_{l_{\lambda}(d_{X})}(\alpha, \beta) \coloneqq \inf_{\pi \in \Pi(\alpha, \beta)} \int_{X \times X} \min\{d_{X}(x,y), \lambda\}^{p} d\pi(x,y)$ is a solution to the dual
\begin{equation} \label{constdual}
    \sup_{{\phi \in C_{b}(X):} \atop {\textrm{Range}(\phi) \leq \lambda}} \int \phi \:d \alpha - \int \phi^{c} \:d \beta,
\end{equation}
such that $\phi^{c} = \phi$ (\citet{ma2023inference}, Theorem 2.1), $\lambda > 0$. The latter is a constrained formulation of the regularized dual (\ref{dualWeps}). Given any tolerable margin $\lambda <\infty$ on the range of potentials, the solution to (\ref{constdual}) satisfies (\ref{dualWeps}).

\subsection{Existence of optimal couplings in LRGW} \label{exist}
Given Polish mm spaces $(\mathcal{X}, d_{X}, \mu)$, $(\mathcal{Y}, d_{Y}, \nu)$ define the locally robust $p$-distortion realized by a coupling $\pi \in \Pi(\mu, \nu)$ as, $p \in [0,\infty]$
\begin{equation*}
    J^{\lambda}_{p}(\pi) \coloneqq \norm{l_{\lambda}(d_{X})-l_{\lambda}(d_{Y})}_{L^{p}(\pi \otimes \pi)},
\end{equation*}
where $\lambda >0$. Then,
\begin{lemma}
    There exists a coupling $\pi^{\lambda} \in \Pi(\mu, \nu)$ such that, $(p,\lambda)$-LRGW $=J^{\lambda}_{p}(\pi)$.
\end{lemma}
The lemma can be proved by extending Corollary 10.1 of \citet{memoli2011gromov} for the mm spaces with modified metrics $l_{\lambda}(d_{X})$ and $l_{\lambda}(d_{Y})$. We give a version of the proof for completeness.

\begin{proof}
    First, observe that the set of couplings $\Pi(\mu, \nu)$ is sequentially compact in $\mathcal{P}(\mathcal{X} \times \mathcal{Y})$ (\citet{sturm2023space}, lemma 1.2). Hence, for $1\leq p<\infty$, it suffices to show that $J^{\lambda}_{p}(\pi)$ is continuous on $\Pi(\mu, \nu)$. Let us choose the metric $d^{\lambda}((x,y),(x',y')) = l_{\lambda}(d_{X}(x,x')) + l_{\lambda}(d_{Y}(y,y'))$ mapping $\mathcal{X} \times \mathcal{Y}$ to $[0,2\lambda]$. Given $((x_{i},y_{i}), (x'_{i},y'_{i})) \sim \pi \otimes \pi$ for $i=1,2$, observe that
    \begin{align}
        &\abs{\abs{l_{\lambda}(d_{X}(x_{1},x'_{1})) - l_{\lambda}(d_{Y}(y_{1},y'_{1}))}-\abs{l_{\lambda}(d_{X}(x_{2},x'_{2})) - l_{\lambda}(d_{Y}(y_{2},y'_{2}))}} \nonumber \\ &\leq \abs{l_{\lambda}(d_{X}(x_{1},x'_{1})) - l_{\lambda}(d_{X}(x_{2},x'_{2})) + l_{\lambda}(d_{Y}(y_{2},y'_{2})) - l_{\lambda}(d_{Y}(y_{1},y'_{1}))} \nonumber \\ &\leq \abs{l_{\lambda}(d_{X}(x_{1},x'_{1})) - l_{\lambda}(d_{X}(x_{2},x'_{2}))} + \abs{l_{\lambda}(d_{Y}(y_{2},y'_{2})) - l_{\lambda}(d_{Y}(y_{1},y'_{1}))} \nonumber \\ &\leq [l_{\lambda}(d_{X}(x_{1},x_{2})) + l_{\lambda}(d_{Y}(y_{1},y_{2}))] + [l_{\lambda}(d_{X}(x'_{1},x'_{2})) + l_{\lambda}(d_{Y}(y'_{1},y'_{2}))] \label{trilast} \\ &= d^{\lambda}((x_{1},y_{1}),(x_{2},y_{2})) + d^{\lambda}((x'_{1},y'_{1}),(x'_{2},y'_{2})) \label{metri},
    \end{align}
    which implies the Lipschitz continuity of $f_{1} \coloneqq 2\Lambda_{1}(l_{\lambda}(d_{X}),l_{\lambda}(d_{Y}))$ for the metric given by (\ref{metri}). The step (\ref{trilast}) follows from the triangle inequalities of $l_{\lambda}(d_{X})$ and $l_{\lambda}(d_{Y})$. Now, consider a function $f_{2} : [0,2\lambda] \rightarrow \mathbb{R}_{+}$ mapping $t \mapsto t^{p}$, which in turn implies that $f_{2} \circ f_{1}$ is Lipschitz with constant $\leq p(2\lambda)^{p-1}$. Thus, given a sequence $\{\pi_{n}\}_{n \in \mathbb{N}} \subset \mathcal{P}(\mathcal{X} \times \mathcal{Y})$ such that $\pi_{n} \xrightarrow{w} \pi$, we have
    \begin{align*}
        &\Big|\int \underbrace{\int f_{2} \circ f_{1}\:\pi_{n}(d(x,y))}_{\coloneqq f_{\pi_{n}}(x',y')}\pi_{n}(d(x',y')) - \int \underbrace{\int f_{2} \circ f_{1}\:\pi(d(x,y))}_{\coloneqq f_{\pi}(x',y')}\pi(d(x',y')\Big| \\ &\leq \Big|\int (f_{\pi_{n}} - f_{\pi})\: \pi_{n}(d(x',y'))\Big| + \Big|\int f_{\pi}\: \pi_{n}(d(x',y')) - f_{\pi}\: \pi(d(x',y'))\Big| \\ &\leq \max_{(x',y')}|f_{\pi_{n}} - f_{\pi}| + \Big|\int f_{\pi}\: d\pi_{n} - f_{\pi}\: d\pi\Big|,
    \end{align*}
    where the first term on the right-hand side vanishes as $n \rightarrow \infty$ due to the uniformly convergent $f_{\pi_{n}} \rightarrow f_{\pi}$ (based on the point-wise convergence of $f_{\pi_{n}}$ and the Lipschitz continuity of $f_{2} \circ f_{1}$). The weak convergence of $\pi_{n}$ ensures that the second term also vanishes. As such, $J^{\lambda}_{p}(\pi_{n}) \rightarrow J^{\lambda}_{p}(\pi)$ as $n \rightarrow \infty$. Hence, the proof.
    
    To show the result for $p=\infty$, first observe that given $\pi$, the sequence $\{J^{\lambda}_{l}(\pi)\}$ is non-decreasing in $1\leq l \leq \infty$ (using Jensen's inequality) and $\lim_{p \rightarrow \infty}J^{\lambda}_{p}(\pi) \rightarrow J^{\lambda}_{\infty}(\pi) = \sup\{J^{\lambda}_{p}: 1\leq p < \infty\}$. As such, $J^{\lambda}_{\infty}(\pi)$ is lower semi-continuous. Now, invoking the compactness argument proves the infimum is achieved in $\Pi(\mu , \nu)$.
\end{proof}

\subsection*{Proof of lemma \ref{properLR}}
    \textit{(i)} Given $\lambda' > \lambda \geq 0$, observe that
    \begin{equation*}
        \abs{l_{\lambda}(d_{X}) - l_{\lambda}(d_{Y})} \leq \abs{l_{\lambda'}(d_{X}) - l_{\lambda'}(d_{Y})},
    \end{equation*}
    a.s. $(\mu \otimes \nu)^{\otimes 2}$, which proves the first claim. Also, for all $a \in \mathbb{R}$ the following equality holds
    \begin{align*}
        l_{\lambda'}(a) - l_{\lambda}(a) = \min\{\lambda', \max\{a,\lambda\}\} - \lambda = l_{\lambda'}(\bar{l}_{\lambda}(a))-\lambda.
    \end{align*}
    We point out that, even a stronger equality holds in general almost surely
    \begin{equation} \label{lpineq}
        \abs{l_{\lambda'}(d_{X}) - l_{\lambda'}(d_{Y})} = \abs{l_{\lambda}(d_{X}) - l_{\lambda}(d_{Y})} + \abs{l_{\lambda'}(\bar{l}_{\lambda}(d_{X}))-l_{\lambda'}(\bar{l}_{\lambda}(d_{Y}))}.
    \end{equation}
    Now, given any $\pi$ that solves the $(p,\lambda)$-LRGW($\mu,\nu$) problem
    \begin{align*}
        (p,\lambda')\textrm{-LRGW}(\mu,\nu) - (p,\lambda)\textrm{-LRGW}(\mu,\nu) &\leq \norm{l_{\lambda'}(d_{X}) - l_{\lambda'}(d_{Y})}_{L^{p}(\pi \otimes \pi)}- (p,\lambda)\textrm{-LRGW}(\mu,\nu) \\ &\leq \norm{l_{\lambda'}(\bar{l}_{\lambda}(d_{X}))-l_{\lambda'}(\bar{l}_{\lambda}(d_{Y}))}_{L^{p}(\pi \otimes \pi)},
    \end{align*}
    where the second inequality is due to (\ref{lpineq}) and the the triangle inequality of $L^{p}$ norms. The proof follows since the choice of $\pi$ is arbitrary in $\Pi^{\lambda}$. We also present an upper bound that depends on the spaces' regulated $p$-diameters. First, $\forall a \in \mathbb{R}$ let us write
    \begin{equation*}
        \abs{l_{\lambda+\lambda'}(a)-l_{\lambda'}(a)} = \bar{l}_{0}(l_{\lambda}(a-\lambda')).
    \end{equation*}
    Now, for any optimal coupling $\pi \in \Pi(\mu, \nu)$ for the $(p,\lambda')\textrm{-LRGW}$ problem, we get
    \begin{align}
        (p,\lambda+\lambda')\textrm{-LRGW}(\mu,\nu) &\leq \norm{l_{\lambda+\lambda'}(d_{X})-l_{\lambda+\lambda'}(d_{Y})}_{L^{p}(\pi \otimes \pi)} \nonumber \\ &\leq \norm{l_{\lambda+\lambda'}(d_{X})-l_{\lambda'}(d_{X})}_{L^{p}(\mu \otimes \mu)} + \norm{l_{\lambda'}(d_{X})-l_{\lambda'}(d_{Y})}_{L^{p}(\pi \otimes \pi)} \nonumber \\ &\qquad \qquad+ \norm{l_{\lambda'}(d_{Y})-l_{\lambda+\lambda'}(d_{Y})}_{L^{p}(\nu \otimes \nu)} \nonumber \\ &\leq \norm{\bar{l}_{0}(l_{\lambda}(d_{X}(x,x')-\lambda'))}_{L^{p}(\mu \otimes \mu)} + (p,\lambda')\textrm{-LRGW}(\mu,\nu) \nonumber \\&\qquad \qquad +\norm{\bar{l}_{0}(l_{\lambda}(d_{Y}(y,y')-\lambda'))}_{L^{p}(\nu \otimes \nu)} \nonumber \\ &= d^{\lambda,\lambda'}_{\mathcal{X},p} + d^{\lambda,\lambda'}_{\mathcal{Y},p} + (p,\lambda')\textrm{-LRGW}(\mu,\nu),
    \end{align}
    where 
    \begin{align*}
        \norm{l_{\lambda}(d_{X})}_{L^{p}(\mu \otimes \mu)} &\geq d^{\lambda,\lambda'}_{\mathcal{X},p} \\ &= \norm{\bar{l}_{0}(l_{\lambda}(d_{X}(x,x')-\lambda'))}_{L^{p}(\mu \otimes \mu)} \\ &= \norm{\bar{l}_{\lambda'}(l_{\lambda+\lambda'}(d_{X}(x,x')))-\lambda'}_{L^{p}(\mu \otimes \mu)} \\ &\geq \bar{l}_{\lambda'}\big(\norm{l_{\lambda+\lambda'}(d_{X}(x,x'))}_{L^{p}(\mu \otimes \mu)}\big)-\lambda' \\ &= \bar{l}_{0}\big(\norm{l_{\lambda+\lambda'}(d_{X})}_{L^{p}(\mu \otimes \mu)} -\lambda'\big).
    \end{align*}
    As such, $(p,\lambda+\lambda')\textrm{-LRGW}(\mu,\nu) - (p,\lambda')\textrm{-LRGW}(\mu,\nu) \leq 2\max_{\mathcal{Z} \in (\mathcal{X},\mathcal{Y})}d^{\lambda,\lambda'}_{\mathcal{Z},p}$. \\
    \textit{(ii)} Choosing $\lambda'=\infty$ in the first argument of \textit{(i)} proves the result.

\subsection*{Proof of Proposition \ref{rob_rad}}
    \textit{(i)} The result is a direct consequence of the fact $W^{0}_{p} = W_{p}$. \\
    \textit{(ii)} For $0 \leq \epsilon \leq \epsilon' \leq 1$, 
    \begin{equation}
        d_{\textrm{LRGW}}(\mu, \nu; \epsilon) = \inf_{\pi \in \Pi(\mu, \nu)} \norm{W^{\epsilon}_{p}(x,x') - W_{p}(y,y')}_{L^{p}(\pi \otimes \pi)}.
    \end{equation}
    Now, for $(x,y), (x',y') \sim \mu \otimes \nu$
    \begin{align} \label{inewtri}
        \abs{W^{\epsilon}_{p}(x,x') - W_{p}(y,y')} \leq \abs{W^{\epsilon}_{p}(x,x') - W^{\epsilon'}_{p}(x,x')} + \abs{W^{\epsilon'}_{p}(x,x') - W_{p}(y,y')}.
    \end{align}
    Due to the dual form (see, Appendix \ref{duality}), given any $\alpha, \beta \in \mathcal{P}(X)$ we may write
    \begin{align*}
        (1-\epsilon)W^{\epsilon}_{p}(\alpha, \beta)^{p} &= \sup_{\phi \in C_{b}(X)} \int \phi \:d \alpha - \int \phi^{c} \:d \beta - \epsilon\: \textrm{Range}(\phi) \\ &\leq \sup_{\phi \in C_{b}(X)} \int \phi \:d \alpha - \int \phi^{c} \:d \beta - \epsilon'\: \textrm{Range}(\phi) + 2(\epsilon'-\epsilon) \norm{\phi}_{\infty}. 
    \end{align*}
    Since $\phi \in C_{b}(X)$, $\exists K>0$ such that $W^{\epsilon}_{p}(\alpha, \beta) - \big(\frac{1-\epsilon'}{1-\epsilon}\big)^{\frac{1}{p}}W^{\epsilon'}_{p}(\alpha, \beta) \leq K\big(\frac{\epsilon'-\epsilon}{1-\epsilon}\big)^{\frac{1}{p}}$. As such, 
    \begin{align}
        0 &\leq W^{\epsilon}_{p}(x,x') - W^{\epsilon'}_{p}(x,x') \nonumber \\ &\leq \Big[\Big(\frac{1-\epsilon'}{1-\epsilon}\Big)^{\frac{1}{p}}-1\Big]W^{\epsilon'}_{p}(x, x') + K\Big(\frac{\epsilon'-\epsilon}{1-\epsilon}\Big)^{\frac{1}{p}} \nonumber \\ &\leq  \Big[\Big(\frac{1-\epsilon'}{1-\epsilon}\Big)^{\frac{1}{p}}-1\Big]\underline{W}^{\epsilon'}_{p,\mu} + K\Big(\frac{\epsilon'-\epsilon}{1-\epsilon}\Big)^{\frac{1}{p}} \label{triv},
    \end{align}
    where $\underline{W}^{\epsilon'}_{p,\mu} = \inf_{x,x'\sim \mu}W^{\epsilon'}_{p}(x, x')$, which may only be non-zero given a sample problem. The last inequality, along with (\ref{inewtri}) and the triangle inequality of $L^{p}$ norms implies that
    \begin{equation*}
        d_{\textrm{LRGW}}(\mu, \nu; \epsilon) - d_{\textrm{LRGW}}(\mu, \nu; \epsilon') \leq  \Big[\Big(\frac{1-\epsilon'}{1-\epsilon}\Big)^{\frac{1}{p}}-1\Big]\underline{W}^{\epsilon'}_{p,\mu} + K\Big(\frac{\epsilon'-\epsilon}{1-\epsilon}\Big)^{\frac{1}{p}}.
    \end{equation*}
    Since (\ref{inewtri}) holds both ways (in $\epsilon,\epsilon'$), invoking the trivial upper bound to (\ref{triv}) we obtain
    \begin{equation*}
        \abs{d_{\textrm{LRGW}}(\mu, \nu; \epsilon) - d_{\textrm{LRGW}}(\mu, \nu; \epsilon')} \lesssim \Big(\frac{\epsilon'-\epsilon}{1-\epsilon}\Big)^{\frac{1}{p}}. 
    \end{equation*} 

\subsection*{Proof of Proposition \ref{levy}}
    The proof follows as a modification to \citet{huber1981robust}, Corollary 4.3. Given any $\pi \in \Pi(\alpha, \beta)$, observe that
    \begin{align*}
        \mathbb{E}_{\pi}[l_{\lambda}(d_{X}(x,y))] &\leq \varepsilon\:\mathbb{P}(l_{\lambda}(d_{X}(x,y)) \leq \varepsilon) + \lambda\:\mathbb{P}(l_{\lambda}(d_{X}(x,y)) > \varepsilon) \\ &= \varepsilon + (\lambda-\varepsilon)\:\mathbb{P}(l_{\lambda}(d_{X}(x,y)) > \varepsilon).
    \end{align*}
    Now, consider $\varepsilon = \hat{\rho}_{\lambda}(\alpha, \beta)$. As such there exists $\pi$ such that $\pi(\{(x,y):l_{\lambda}(d_{X}(x,y))>\varepsilon\})\leq \varepsilon$. Thus
    \begin{equation*}
        \mathbb{E}_{\pi}[l_{\lambda}(d_{X}(x,y))] \leq \varepsilon + (\lambda-\varepsilon)\varepsilon \leq (1+\lambda)\varepsilon.
    \end{equation*}
    Taking infimum over all couplings result in $\frac{1}{1+\lambda}W_{1,\lambda} \leq \hat{\rho}_{\lambda}$.
    To show the upper bound, by using Markov's inequality we write
    \begin{equation*}
        \mathbb{P}(l_{\lambda}(d_{X}(x,y)) > \varepsilon) \leq \frac{\mathbb{E}_{\pi}[l_{\lambda}(d_{X}(x,y))]}{\varepsilon},
    \end{equation*}
    where $\pi \in \Pi(\alpha, \beta)$ is a feasible optimal coupling for $W_{1,\lambda}$. Observe that we can always choose $\varepsilon^{2} = W_{1,\lambda}$. As such, $\hat{\rho}_{\lambda} \leq \sqrt{W_{1,\lambda}}$.

\subsection*{Proof of Proposition \ref{Sturmupper}}
    Before proving the first inequality, recall that given two mm spaces $(\mathcal{X}, d_{X}, \mu)$ and $(\mathcal{Y}, d_{Y}, \nu)$, a metric $\tilde{d}$ on $\mathcal{X} \sqcup \mathcal{Y}$ (disjoint union) is said to be a coupling of $d_{X}$ and $d_{Y}$ if and only if $\tilde{d}(x,x') = d_{X}(x,x')$ and $\tilde{d}(y,y') = d_{Y}(y,y')$ hold for all $x,x' \in \mathcal{X}$, $y,y' \in \mathcal{Y}$. Let us denote by $\mathscr{D}(d_{X},d_{Y})$ the collection of all such couplings. 

    Now, given $(x,y),(x',y') \sim \mu \otimes \nu$ and $\lambda>0$, we have $\abs{l_{\lambda}(d_{X}(x,x')) - l_{\lambda}(d_{Y}(y,y'))} \leq l_{\lambda}(\tilde{d}(x,y)) + l_{\lambda}(\tilde{d}(x',y'))$. Hence, for $p \in [1,\infty)$
    \begin{equation*}
        J^{\lambda}_{p}(\pi) \coloneqq \norm{l_{\lambda}(d_{X})-l_{\lambda}(d_{Y})}_{L^{p}(\pi \otimes \pi)} \leq 2{||l_{\lambda}(\tilde{d})||}_{L^{p}(\pi)}
    \end{equation*}
    hold for all $\pi \in \Pi(\mu,\nu)$. Hence, the inequality. To show the upper bound, we require some additional definition.
    \begin{definition}[Modulus of trimmed mass distribution]
        For $\delta \geq 0$, the modulus of $\lambda$-trimmed mass distribution of $\mu$, having full support is defined as 
        \begin{equation*}
            v^{\lambda}_{\delta}(\mu) \coloneqq \inf\{\varepsilon > 0 : \mu(\{x\in\mathcal{X} : \mu(B^{\lambda}_{X}(x,\varepsilon))\leq \delta\})\leq \varepsilon\},
        \end{equation*}
        where $B^{\lambda}_{X}(x,\varepsilon) = \{y\in\mathcal{X} : l_{\lambda}(d_{X}(x,y)) < \varepsilon\}$ is the open ball of $\lambda$-trimmed radius $\varepsilon>0$ around $x \in \mathcal{X}$. 
    \end{definition}
    Here, we uniquely consider $\mu$ to be a probability measure. Observe that only when $\varepsilon < \lambda$, we get $B^{\lambda}_{X}(x,\varepsilon) = B_{X}(x,\varepsilon)$: the usual open ball around $x$. Otherwise, the ball becomes the entire $\mathcal{X}$. As such, we only account for the `thin points' residing in the trimmed support. $v^{\lambda}_{\delta}(\mu)$ is essentially equal to $\min\{\lambda, v_{\delta}(\mu)\}$, where $v_{\delta}$ is the modulus under the metric $d_{X}$. This preserves the continuity in the sense that $v^{\lambda}_{\delta}(\mu) \xrightarrow{\delta \rightarrow 0} 0$ (\citet{greven2009convergence}, lemma 6.5). The relation also implies that an effective trimming requires $\lambda \leq 1$ for probability measures. The proof for the upper bound requires showing a similar statement to lemma 10.3 under the altered metrics $l_{\lambda}(d_{X})$ and $l_{\lambda}(d_{Y})$.
    
    \textit{\textbf{Step 1} (Construction of $\varepsilon$-nets):} Let $\delta \in (0,\frac{1}{2})$, and $\pi \in \Pi(\mu,\nu)$ such that $J^{\lambda}_{p}(\pi) < \delta^{5}$. Since the altered metric space $(\mathcal{X}, l_{\lambda}(d_{X}))$ also contains a maximal $2\varepsilon$-separated net for any $\varepsilon \geq 0$, we have the following statement.
    \begin{lemma}[\citet{greven2009convergence}, lemma 6.9]
        Given $\delta>0$ and $v^{\lambda}_{\delta}(\mu) < \varepsilon$, there exist points $\{x_{i}\}^{N}_{i=1} \in \mathcal{X}$ with $N \leq \floor{\frac{1}{\delta}}$ such that $\mu(B^{\lambda}_{X}(x_{i},\varepsilon)) > \delta$, and $\mu(\bigcup^{N}_{i=1}B^{\lambda}_{X}(x,2\varepsilon)) > 1-\varepsilon$, $\forall\: i = 1, \cdots, N$. Also, for all $i \neq j = 1,\cdots, N$, $l_{\lambda}(d_{X}(x_{i},x_{j})) > \varepsilon$.
    \end{lemma}
    Observe that, since the effective range of permissible $\varepsilon$ remains $(0,\lambda)$, a maximal net of $(\mathcal{X}, d_{X})$ may be a feasible candidate satisfying the lemma. Beyond the range, the argument becomes trivial.

    Now, set $\varepsilon = 4v^{\lambda}_{\delta}(\mu)$. As such, we can find a set of points $\{x_{i}\}^{N}_{i=1} \in \mathcal{X}$ which ensures $\mu(\bigcup^{N}_{i=1}B^{\lambda}_{X}(x,\varepsilon)) > 1-\varepsilon$ with $l_{\lambda}(d_{X}(x_{i},x_{j})) > \varepsilon/2$ for all $i \neq j = 1,\cdots, N$. The set may contain arbitrarily far lying observations, yet the argument holds until $\lambda > \varepsilon/2$. Thus, following \citet{memoli2011gromov}, Claim 10.2 we argue that $\forall \: i = 1,\cdots, N$ $\exists\: y_{i} \in \mathcal{Y}$ such that 
    \begin{equation} \label{Si}
        \pi\Big(B^{\lambda}_{X}(x_{i},\varepsilon) \times B^{\lambda}_{Y}(y_{i},2(\varepsilon+\delta))\Big) \geq (1-\delta^2)\mu(B^{\lambda}_{X}(x_{i},\varepsilon))> \delta(1-\delta^2).
    \end{equation}
    Observe that if $\lambda < 2(\varepsilon+\delta)$, the first inequality holds trivially. We denote by $S=\{(x_{i},y_{i}), i=1,\cdots, N\} \subset \mathcal{X} \times \mathcal{Y}$ the set of points constructing the nets.
    
    \textit{\textbf{Step 2} (Bounding locally robust distortions):} Consider $\{x_{i}\}^{N}_{i=1}$ and $\{y_{i}\}^{N}_{i=1}$ that satisfy (\ref{Si}) with $\mu(B^{\lambda}_{X}(x_{i},\varepsilon)) > \delta$. Then, for all $i,j = 1,\cdots, N$
    \begin{equation*}
        \abs{l_{\lambda}(d_{X}(x_{i},x_{j})) - l_{\lambda}(d_{Y}(y_{i},y_{j}))} \leq 6(\varepsilon + \delta).
    \end{equation*}
    To prove the claim, let us first assume that there exists a pair $(i,j)$ for which it does not hold. As such, for all $x' \in B^{\lambda}_{X}(x_{i},\varepsilon)$, $x'' \in B^{\lambda}_{X}(x_{j},\varepsilon)$, $y' \in B^{\lambda}_{Y}(y_{i},2(\varepsilon+\delta))$, and $y'' \in B^{\lambda}_{Y}(y_{j},2(\varepsilon+\delta))$ we have
    \begin{align*}
        &\abs{l_{\lambda}(d_{X}(x',x'')) - l_{\lambda}(d_{Y}(y',y''))} \\ & \geq \abs{l_{\lambda}(d_{X}(x_{i},x_{j})) - l_{\lambda}(d_{Y}(y_{i},y_{j}))} - \abs{l_{\lambda}(d_{X}(x_{i},x_{j})) - l_{\lambda}(d_{X}(x',x''))} \\ &\qquad \qquad - \abs{l_{\lambda}(d_{Y}(y',y'')) - l_{\lambda}(d_{Y}(y_{i},y_{j}))} \\ &\geq 6(\varepsilon + \delta) - 3\varepsilon - 4(\varepsilon+\delta) =2\delta.
    \end{align*}
    Then,
    \begin{align*}
        J^{\lambda}_{1}(\pi) &\geq 2\delta \:\pi\Big(B^{\lambda}_{X}(x_{i},\varepsilon) \times B^{\lambda}_{Y}(y_{i},2(\varepsilon+\delta))\Big) \pi\Big(B^{\lambda}_{X}(x_{j},\varepsilon) \times B^{\lambda}_{Y}(y_{j},2(\varepsilon+\delta))\Big) \\ &\geq 2\delta^{3}(1-\delta^2)^{2} > 2\delta^{5},
    \end{align*}
    since $\delta \leq \frac{1}{2}$. This contradicts our initial assumption.

    \textit{\textbf{Step 3} (Constructing a suitable metric $S$):}
    Define $\tilde{d}^{\lambda}_{S}$ on $\mathcal{X} \sqcup \mathcal{Y}$ as
    \begin{equation*}
        (x,y) \mapsto \inf_{(x',y') \in S} \Big[l_{\lambda}(d_{X}(x,x')) + \norm{l_{\lambda}(d_{X}) - l_{\lambda}(d_{Y})}_{L^{\infty}(S \times S)} + l_{\lambda}(d_{Y}(y,y'))\Big],
    \end{equation*}
    also assuming $\tilde{d}^{\lambda}_{S} = l_{\lambda}(d_{X})$ on $\mathcal{X} \times \mathcal{X}$ and $\tilde{d}^{\lambda}_{S} = l_{\lambda}(d_{Y})$ on $\mathcal{Y} \times \mathcal{Y}$. Using Step 2, we get
    \begin{equation}\label{M}
        \tilde{d}^{\lambda}_{S}(x,y) \leq 2\lambda + 6(\varepsilon + \delta).
    \end{equation}
    However, for $i = 1, \cdots, N$, given $(x,y) \in B^{\lambda}_{X}(x_{i},\varepsilon) \times B^{\lambda}_{Y}(y_{i},2(\varepsilon+\delta))$ we have
    \begin{align}
        \tilde{d}^{\lambda}_{S}(x,y) &\leq \varepsilon + 2(\varepsilon+\delta) + \tilde{d}^{\lambda}_{S}(x_{i},y_{i}) \nonumber\\ &\leq \varepsilon + 8(\varepsilon+\delta), \label{L}
    \end{align}
    where the last inequality is due to \citet{memoli2011gromov}, lemma 10.1. Now, in pursuit of fragmenting $\mathcal{X} \times \mathcal{Y}$ based on balls around the points that constitute the maximal net, define $L = \bigcup^{N}_{i=1} B^{\lambda}_{X}(x_{i},\varepsilon) \times B^{\lambda}_{Y}(y_{i},2(\varepsilon+\delta))$. Hence, 
    \begin{align*}
        \int_{\mathcal{X} \times \mathcal{Y}} [\tilde{d}^{\lambda}_{S}(x,y)]^{p} d\pi(x,y) &=  \int_{L} [\tilde{d}^{\lambda}_{S}(x,y)]^{p} d\pi(x,y) + \int_{\mathcal{X} \times \mathcal{Y} \backslash L} [\tilde{d}^{\lambda}_{S}(x,y)]^{p} d\pi(x,y) \\ & \leq (9(\varepsilon+\delta))^{p} + (\varepsilon+\delta)[2\lambda + 6(\varepsilon + \delta)]^{p},
    \end{align*}
    where the last inequality is due to (\ref{M}), (\ref{L}) and the fact that $\pi(\mathcal{X} \times \mathcal{Y} \backslash L) \leq \varepsilon+\delta$ (\citet{memoli2011gromov}, Claim 10.5), which is obvious if $\lambda < 2(\varepsilon+\delta)$. As such, for $p \in [1,\infty)$
    \begin{equation*}
        d_{l\textrm{RSGW}}(\mu,\nu) \leq (4\lambda+\delta)^{\frac{1}{p}}(62\lambda + \frac{15}{2}),
    \end{equation*}
    since $\varepsilon \leq 4\lambda$ and $\delta\leq\frac{1}{2}$.

\subsection{Sample Complexity of Transform Sampling Using Tukey and LR-guided RGM Under Contamination} \label{sample_comp}
    While (\ref{RRGM_lag}) proposes altering the set of amenable couplings, based on our discussion in the first two sections, it is quite intuitive to think of a formulation that rather penalizes the norms. For example, we may construct a robust transform sampler drawing from both LR and Tukey's robustification techniques as follows:
    \begin{equation} \label{LamTuk}
        \inf_{F,G} \int \mathcal{T}_{2}\Big(d_{X}(x, G(y)) - d_{Y}(y,F(x))\Big)d\mu \otimes \nu(x,y) + \lambda_{1}W^{\lambda}_{1}(\mu, G_{\#}\nu) + \lambda_{2}W^{\lambda}_{1}(F_{\#}\mu, \nu),
    \end{equation}
    where the parameter underlying $\mathcal{T}_{2}$ is $\tau \geq 0$ and $\lambda \geq 0$. Typically, in an empirical problem both $\tau, \lambda$ need to be tuned, and the infimum is taken over arbitrary measurable maps $F,G$ between spaces $\mathcal{X} \rightleftarrows \mathcal{Y}$. They need not follow the binding constraint, as the Lagrangian conditions ensure their measure preservation only. We assume that they are continuous and component-wise uniformly bounded in the following sense.

    \begin{assumption} \label{ass1}
        Given that $\mathcal{X} \subseteq \mathbb{R}^{d}$ and $\mathcal{Y} \subseteq \mathbb{R}^{d'}$, let $F_{k}(\cdot) : \mathcal{X} \rightarrow \mathbb{R}$ (similarly, $G_{l}(\cdot) : \mathcal{Y} \rightarrow \mathbb{R}$) denote the $k$th (and $l$th, $l = 1, \cdots,d$) coordinate of $F$, which is continuous (similarly, $G$), $k=1,\cdots,d'$. There exists $b > 0$ such that for all $k=1,\cdots,d'$, and $l = 1, \cdots,d$
            \begin{equation*}
            \sup_{F} \norm{F_{k}}_{\infty}, \sup_{G} \norm{G_{l}}_{\infty} \leq b.
            \end{equation*}
    \end{assumption}
    We denote such classes of functions as $\mathcal{F}_{b}^{\mathcal{X} \rightarrow \mathcal{Y}}$ and $\mathcal{F}_{b}^{\mathcal{Y} \rightarrow \mathcal{X}}$ respectively.
    
    \subsubsection{Robust Concentration Inequalities}
    Given a pair of feasible maps $(F,G) \in \mathcal{F}_{b}^{\mathcal{X} \rightarrow \mathcal{Y}} \times \mathcal{F}_{b}^{\mathcal{Y} \rightarrow \mathcal{X}}$ we observe the non-asymptotic deviation of realized values of (\ref{LamTuk}) from a robust population benchmark. We assume, without loss of generality, that $\lambda_{1} = \lambda_{2} = 1$. Also, let $x_{1},x_{2},\cdots,x_{m}$ are sampled following the $\mathcal{O} \cup \mathcal{I}$ framework with $\mu$ being the inlier distribution. Similarly, we have $y_{1},y_{2},\cdots,y_{n}$ from the other space with inliers drawn i.i.d from $\nu$. The inlying (outlying) set of samples are indexed using $\mathcal{I}^{X}$ and $\mathcal{I}^{Y}$ ($\mathcal{O}^{X}$ and $\mathcal{O}^{Y}$) respectively. Let us denote
        \begin{align*}
            &T(\mu, \nu,F,G) \coloneqq \int \mathcal{T}_{2}\Big(d_{X}(x, G(y)) - d_{Y}(y,F(x))\Big)d\mu \otimes \nu(x,y), \\
            &L(\mu, \nu,F,G) \coloneqq W^{\lambda}_{1}(\mu, G_{\#}\nu) + W^{\lambda}_{1}(F_{\#}\mu, \nu).
        \end{align*}
        As such, the population loss function in (\ref{LamTuk}) can be written as $C(\mu, \nu, F,G) = T(F,G) + L(F,G)$. The empirical loss under contaminated measures $\hat{\mu}_{m}, \hat{\nu}_{n}$ is given as $T(\hat{\mu}_{m}, \hat{\nu}_{n},F,G)$. However, to emphasize the robust translations we write the empirical version of $L(\mu, \nu,F,G)$ as $L(\hat{\mu}_{m}, \hat{\nu}_{n},F,G) \coloneqq W^{\lambda}_{1}(\mu, G_{\#}\hat{\nu}_{n}) + W^{\lambda}_{1}(F_{\#}\hat{\mu}_{m}, \nu)$. The following two results combined, present the concentration of $C(\hat{\mu}_{m}, \hat{\nu}_{n}, F,G)$ around $\frac{|\mathcal{I}^{X}||\mathcal{I}^{Y}|}{mn} \mathbb{E}_{\mu \otimes \nu}[T(\hat{\mu}^{\mathcal{I}}_{m}, \hat{\nu}^{\mathcal{I}}_{n},F,G)] + \frac{|\mathcal{I}^{Y}|}{n}\mathbb{E}_{\nu}[W^{\lambda}_{1}(\mu, G_{\#}\hat{\nu}^{\mathcal{I}}_{n})] + \frac{|\mathcal{I}^{X}|}{m}\mathbb{E}_{\mu}[W^{\lambda}_{1}(F_{\#}\hat{\mu}^{\mathcal{I}}_{m}, \nu)]$.

        \begin{proposition} \label{conc_1}
        There exists a constant $K>0$ depending on $\tau^2$ such that for $\delta>0$
            \begin{equation*}
            \abs{T(\hat{\mu}_{m}, \hat{\nu}_{n},F,G) - \frac{|\mathcal{I}^{X}||\mathcal{I}^{Y}|}{mn} \mathbb{E}_{\mu \otimes \nu}[T(\hat{\mu}^{\mathcal{I}}_{m}, \hat{\nu}^{\mathcal{I}}_{n},F,G)]} \leq K\frac{|\mathcal{I}^{X}||\mathcal{I}^{Y}|}{mn}\sqrt{\frac{\ln\Big(\frac{4(|\mathcal{I}^{X}| \vee |\mathcal{I}^{Y}|)}{\delta}\Big)}{|\mathcal{I}^{X}| \wedge |\mathcal{I}^{Y}|}} + \tau^{2} \Big(\frac{|\mathcal{O}^{X}|}{m} + \frac{|\mathcal{O}^{Y}|}{n}\Big)
        \end{equation*}
        holds with probability at least $1-\delta$.
        \end{proposition}
        \textbf{Proof of Proposition \ref{conc_1}}\\
        Let us denote $t_{F,G}(x,y) \coloneqq \mathcal{T}_{2}\Big(d_{X}(x, G(y)) - d_{Y}(y,F(x))\Big)$. Now, following the $\mathcal{O} \cup \mathcal{I}$ setup
        \begin{align*}
            T(\hat{\mu}_{m}, \hat{\nu}_{n},F,G) &= \frac{1}{mn}\sum^{m}_{i=1}\sum^{n}_{j=1} \mathcal{T}_{2}\Big(d_{X}(x_{i}, G(y_{j})) - d_{Y}(y_{j},F(x_{i}))\Big) \\ & \leq  \frac{1}{mn}\sum_{i \in \mathcal{I}^{X}}\sum_{j \in \mathcal{I}^{Y}} \mathcal{T}_{2}\Big(d_{X}(x_{i}, G(y_{j})) - d_{Y}(y_{j},F(x_{i}))\Big) + \frac{\tau^{2}}{mn} \Big(|\mathcal{O}^{X}||\mathcal{O}^{Y}| + |\mathcal{O}^{X}||\mathcal{I}^{Y}| + |\mathcal{I}^{X}||\mathcal{O}^{Y}|\Big) \\ &= \frac{|\mathcal{I}^{X}||\mathcal{I}^{Y}|}{mn} T(\hat{\mu}^{\mathcal{I}}_{m}, \hat{\nu}^{\mathcal{I}}_{n},F,G) + \tau^{2} \Big(\frac{|\mathcal{O}^{X}|}{m} + \frac{|\mathcal{O}^{Y}|}{n} - \frac{|\mathcal{O}^{X}||\mathcal{O}^{Y}|}{mn}\Big).
        \end{align*}
        Also, 
        \begin{equation*}
            \frac{|\mathcal{I}^{X}||\mathcal{I}^{Y}|}{mn} T(\hat{\mu}^{\mathcal{I}}_{m}, \hat{\nu}^{\mathcal{I}}_{n},F,G) - \tau^{2} \Big(\frac{|\mathcal{O}^{X}|}{m} + \frac{|\mathcal{O}^{Y}|}{n} - \frac{|\mathcal{O}^{X}||\mathcal{O}^{Y}|}{mn}\Big) \leq T(\hat{\mu}_{m}, \hat{\nu}_{n},F,G).
        \end{equation*}
        As such, combining the two inequalities we get
        \begin{align}
            &\abs{T(\hat{\mu}_{m}, \hat{\nu}_{n},F,G) - \frac{|\mathcal{I}^{X}||\mathcal{I}^{Y}|}{mn} \mathbb{E}_{\mu \otimes \nu}[T(\hat{\mu}^{\mathcal{I}}_{m}, \hat{\nu}^{\mathcal{I}}_{n},F,G)]} \nonumber \\ &\leq \frac{|\mathcal{I}^{X}||\mathcal{I}^{Y}|}{mn} \abs{T(\hat{\mu}^{\mathcal{I}}_{m}, \hat{\nu}^{\mathcal{I}}_{n},F,G) - \mathbb{E}_{\mu \otimes \nu}[T(\hat{\mu}^{\mathcal{I}}_{m}, \hat{\nu}^{\mathcal{I}}_{n},F,G)]} + \tau^{2} \Big(\frac{|\mathcal{O}^{X}|}{m} + \frac{|\mathcal{O}^{Y}|}{n} - \frac{|\mathcal{O}^{X}||\mathcal{O}^{Y}|}{mn}\Big). \label{ineq_end}
        \end{align}
        Now, observe that
        \begin{align}
            &\abs{T(\hat{\mu}^{\mathcal{I}}_{m}, \hat{\nu}^{\mathcal{I}}_{n},F,G) - \mathbb{E}_{\mu \otimes \nu}[T(\hat{\mu}^{\mathcal{I}}_{m}, \hat{\nu}^{\mathcal{I}}_{n},F,G)]} \nonumber \\ &= \abs{\frac{1}{|\mathcal{I}^{X}||\mathcal{I}^{Y}|}\sum_{i \in \mathcal{I}^{X}}\sum_{j \in \mathcal{I}^{Y}} t_{F,G}(x_{i},y_{j}) - \mathbb{E}_{\mu \otimes \nu}[t_{F,G}(x,y)]} \nonumber \\ &\leq \abs{\frac{1}{|\mathcal{I}^{X}|}\sum_{i \in \mathcal{I}^{X}} \left(\frac{1}{|\mathcal{I}^{Y}|}\sum_{j \in \mathcal{I}^{Y}}t_{F,G}(x_{i},y_{j}) -  \mathbb{E}_{\nu}[t_{F,G}(x_{i},y)]\right)} + \abs{\frac{1}{|\mathcal{I}^{X}|}\sum_{i \in \mathcal{I}^{X}}\mathbb{E}_{\nu}[t_{F,G}(x_{i},y)] - \mathbb{E}_{\mu \otimes \nu}[t_{F,G}(x,y)]}. \label{two_ter}
        \end{align}
        Recall that $M = \textrm{diam}(\mathcal{X}) \vee \textrm{diam}(\mathcal{Y})$. The function $|\mathcal{I}^{X}|^{-1}\sum_{i \in \mathcal{I}^{X}}t_{F,G}(x_{i},y)$ satisfies the bounded difference inequality with parameter $|\mathcal{I}^{X}|^{-1}(4M^{2} \wedge \tau^{2})$. Hence, due to McDiarmid's inequality 
        \begin{align*}
            &\abs{\frac{1}{|\mathcal{I}^{X}|}\sum_{i \in \mathcal{I}^{X}}\mathbb{E}_{\nu}[t_{F,G}(x_{i},y)] - \mathbb{E}_{\mu \otimes \nu}[t_{F,G}(x,y)]} \\ &\leq \mathbb{E}_{\nu}\abs{\frac{1}{|\mathcal{I}^{X}|}\sum_{i \in \mathcal{I}^{X}}t_{F,G}(x_{i},y) - \mathbb{E}_{\mu}[t_{F,G}(x,y)]} \leq \sqrt{\frac{(4M^{2} \wedge \tau^{2})^{2} \ln(2/\delta)}{2|\mathcal{I}^{X}|}}
        \end{align*}
        holds with probability at least $1-\delta$, where the first inequality follows from Jensen's inequality. For the first term in (\ref{two_ter}), using the union bound over a similar argument we get
        \begin{align*}
            \abs{\frac{1}{|\mathcal{I}^{X}|}\sum_{i \in \mathcal{I}^{X}} \left(\frac{1}{|\mathcal{I}^{Y}|}\sum_{j \in \mathcal{I}^{Y}}t_{F,G}(x_{i},y_{j}) -  \mathbb{E}_{\nu}[t_{F,G}(x_{i},y)]\right)} \leq \sqrt{\frac{(4M^{2} \wedge \tau^{2})^{2} \ln(2|\mathcal{I}^{X}|/\delta)}{2|\mathcal{I}^{Y}|}}
        \end{align*}
        with probability at least $1-\delta$. As such, 
        \begin{equation*}
            \abs{T(\hat{\mu}^{\mathcal{I}}_{m}, \hat{\nu}^{\mathcal{I}}_{n},F,G) - \mathbb{E}_{\mu \otimes \nu}[T(\hat{\mu}^{\mathcal{I}}_{m}, \hat{\nu}^{\mathcal{I}}_{n},F,G)]} \lesssim \sqrt{\frac{\ln\Big(\frac{2(|\mathcal{I}^{X}| \vee |\mathcal{I}^{Y}|)}{\delta}\Big)}{|\mathcal{I}^{X}| \wedge |\mathcal{I}^{Y}|}}
        \end{equation*}
        holds with probability $\geq 1-2\delta$. Hence, putting this back to (\ref{ineq_end}) proves the result.

        The proof also shows that it is always possible to replace the term $\frac{|\mathcal{I}^{X}||\mathcal{I}^{Y}|}{mn} \mathbb{E}_{\mu \otimes \nu}[T(\hat{\mu}^{\mathcal{I}}_{m}, \hat{\nu}^{\mathcal{I}}_{n},F,G)]$ by $T(\mu,\nu,F,G)$, only by incurring an additional term on the upper bound of $\mathcal{O}\Big(\frac{|\mathcal{O}^{X}|}{m} + \frac{|\mathcal{O}^{Y}|}{n}\Big)$.

        \begin{proposition} \label{conc_2}
            There exists a constant $\tilde{K}>0$ depending on $\lambda$ such that for $\delta>0$
            \begin{align*}
                \bigg|W^{\lambda}_{1}(\mu, G_{\#}\hat{\nu}_{n}) + W^{\lambda}_{1}(F_{\#}\hat{\mu}_{m}, \nu) &- \frac{|\mathcal{I}^{Y}|}{n}\mathbb{E}[W^{\lambda}_{1}(\mu, G_{\#}\hat{\nu}^{\mathcal{I}}_{n})] - \frac{|\mathcal{I}^{X}|}{m}\mathbb{E}[W^{\lambda}_{1}(F_{\#}\hat{\mu}^{\mathcal{I}}_{m}, \nu)]\bigg| \\ &\leq \tilde{K} \Big(\frac{\sqrt{|\mathcal{I}^{X}|}}{m} + \frac{\sqrt{|\mathcal{I}^{Y}|}}{n}\Big)\sqrt{\ln(4/\delta)} + \lambda\Big(\frac{|\mathcal{O}^{X}|}{m}+\frac{|\mathcal{O}^{Y}|}{n}\Big)
            \end{align*}
            holds with probability at least $1-\delta$.
        \end{proposition}
        \textbf{Proof of Proposition \ref{conc_2}} \\
        First, let us note that $F_{\#}\hat{\mu}_{m} = \frac{1}{m}\sum^{m}_{i=1}\delta_{F(x_{i})}$ based on the transformed observations $\{F(x_{i})\}^{m}_{i=1}$. Technically, this is rather the empirical distribution $\widehat{(F_{\#}\mu)}_{m}$. Our consideration remains valid if $F$ is taken as an information preserving transform (IPT), which are in abundance (e.g. Lipschitz maps) \citep{chakrabarty2023concurrent}. Hence, similar to the proof of Proposition \ref{conc_1}, we observe
        \begin{equation*}
            \abs{W^{\lambda}_{1}(F_{\#}\hat{\mu}_{m}, \nu) - \frac{|\mathcal{I}^{X}|}{m}W^{\lambda}_{1}(F_{\#}\hat{\mu}^{\mathcal{I}}_{m}, \nu)} \leq \frac{\lambda|\mathcal{O}^{X}|}{m}.
        \end{equation*}
        This implies that
        \begin{align}
            \abs{W^{\lambda}_{1}(F_{\#}\hat{\mu}_{m}, \nu) - \frac{|\mathcal{I}^{X}|}{m}\mathbb{E}[W^{\lambda}_{1}(F_{\#}\hat{\mu}^{\mathcal{I}}_{m}, \nu)]} \leq \frac{|\mathcal{I}^{X}|}{m}\abs{W^{\lambda}_{1}(F_{\#}\hat{\mu}^{\mathcal{I}}_{m}, \nu) - \mathbb{E}[W^{\lambda}_{1}(F_{\#}\hat{\mu}^{\mathcal{I}}_{m}, \nu)]} + \frac{\lambda|\mathcal{O}^{X}|}{m}. \label{anot}
        \end{align}
        Now, using the duality of $W^{\lambda}_{1}$ (see, Appendix \ref{duality}) we may write
        \begin{align*}
            W^{\lambda}_{1}(F_{\#}\hat{\mu}^{\mathcal{I}}_{m}, \nu) = \frac{1}{|\mathcal{I}^{X}|}\sup_{{\phi \in C_{b}(\mathcal{Y}):} \atop {\textrm{Range}(\phi) \leq \lambda}}\sum_{i \in \mathcal{I}^{X}} \phi(F(x_{i})) - \mathbb{E}_{\nu}\phi^{c}.
        \end{align*}
        Due to Assumption \ref{ass1}, the composition $\phi \circ F \in C_{b'}(\mathcal{X})$ for some $b'>0$. As such, $W^{\lambda}_{1}(F_{\#}\hat{\mu}_{m}, \nu)$ satisfies the bounded differences property with upper bound $\mathcal{O}(\frac{1}{|\mathcal{I}^{X}|})$. Hence,
        \begin{equation*}
            \abs{W^{\lambda}_{1}(F_{\#}\hat{\mu}^{\mathcal{I}}_{m}, \nu) - \mathbb{E}[W^{\lambda}_{1}(F_{\#}\hat{\mu}^{\mathcal{I}}_{m}, \nu)]} \lesssim \sqrt{\frac{\ln(2/\delta)}{|\mathcal{I}^{X}|}}
        \end{equation*}
        holds with probability at least $1-\delta$. Putting the bound back in (\ref{anot}) yields,
        \begin{equation*}
            \abs{W^{\lambda}_{1}(F_{\#}\hat{\mu}_{m}, \nu) - \frac{|\mathcal{I}^{X}|}{m}\mathbb{E}[W^{\lambda}_{1}(F_{\#}\hat{\mu}^{\mathcal{I}}_{m}, \nu)]} \leq K'\sqrt{\frac{|\mathcal{I}^{X}|}{m}} \sqrt{\frac{\ln(2/\delta)}{m}} + \frac{\lambda|\mathcal{O}^{X}|}{m},
        \end{equation*}
        that hold with probability at least $1-\delta$, where $K'>0$ depends on $\lambda$. Similarly, we can show that there exists $K''>0$ such that
        \begin{equation*}
            \abs{W^{\lambda}_{1}(\mu, G_{\#}\hat{\nu}_{n}) - \frac{|\mathcal{I}^{Y}|}{n}\mathbb{E}[W^{\lambda}_{1}(\mu, G_{\#}\hat{\nu}^{\mathcal{I}}_{n})]} \leq K''\sqrt{\frac{|\mathcal{I}^{Y}|}{n}} \sqrt{\frac{\ln(2/\delta)}{n}} + \frac{\lambda|\mathcal{O}^{Y}|}{n}
        \end{equation*}
        also holds with probability $\geq 1-\delta$. Combining the last two bounds proves the result.

    \begin{remark}[Uniform Deviations]
        The concentration inequalities make it easier to comment on the uniform deviation $\sup_{(F,G) \in \mathcal{F}_{b}^{\mathcal{X} \rightarrow \mathcal{Y}} \times \mathcal{F}_{b}^{\mathcal{Y} \rightarrow \mathcal{X}}}\abs{T(\hat{\mu}_{m}, \hat{\nu}_{n},F,G) - T(\mu,\nu,F,G)}$. Let us assume both $d_{X}$ and $d_{Y}$ to be the Euclidean metrics in their respective spaces. Due to (\ref{ineq_end}), the problem boils down to finding
        \begin{equation*}
            \sup_{(F,G) \in \mathcal{F}_{b}^{\mathcal{X} \rightarrow \mathcal{Y}} \times \mathcal{F}_{b}^{\mathcal{Y} \rightarrow \mathcal{X}}} \abs{T(\hat{\mu}^{\mathcal{I}}_{m}, \hat{\nu}^{\mathcal{I}}_{n},F,G) - \mathbb{E}_{\mu \otimes \nu}[T(\hat{\mu}^{\mathcal{I}}_{m}, \hat{\nu}^{\mathcal{I}}_{n},F,G)]} \eqqcolon T^{\mathcal{I}}_{m,n}.
        \end{equation*}
        Since the underlying loss depends solely on inlying observations, using Proposition 4.6 of \citet{hur2024reversible}, we obtain for any $\varepsilon > 0$
        \begin{align*}
            T^{\mathcal{I}}_{m,n} \lesssim \sqrt{\frac{\ln\Big(\frac{2(|\mathcal{I}^{X}| \vee |\mathcal{I}^{Y}|)}{\delta}\Big)}{|\mathcal{I}^{X}| \wedge |\mathcal{I}^{Y}|}} + \varepsilon + \sqrt{\frac{\sum^{d'}_{k=1} \log N_{\infty}(\varepsilon, \mathcal{F}_{k}, |\mathcal{I}^{X}|) + \sum^{d}_{l=1} \log N_{\infty}(\varepsilon, \mathcal{G}_{l}, |\mathcal{I}^{Y}|)}{|\mathcal{I}^{X}| \wedge |\mathcal{I}^{Y}|}},
        \end{align*}
        holds with probability $\geq 1-\delta$, where $\mathcal{F}_{k}$ and $\mathcal{G}_{l}$ are respectively the collections of amenable functions $F_{k}$ and $G_{l}$ satisfying Assumption \ref{ass1}. Classes $\mathcal{F}_{k}$ and $\mathcal{G}_{l}$ whose metric entropies scale according to $\mathcal{O}(1/\varepsilon)^{a}$, for some $a>0$ are abundant. For example, Sobolev or Lipschitz-smooth functions defined on the unit interval $[0,1]^{d}$. One can similarly derive uniform deviation bounds corresponding to $L(\hat{\mu}_{m}, \hat{\nu}_{n},F,G)$ and hence, eventually $C(\hat{\mu}_{m}, \hat{\nu}_{n},F,G)$.
    \end{remark}

\subsection{Existence of latent chaining} \label{chain}
    While it is difficult to characterize a suitable $\mathcal{Z}$ that follows the chaining argument given arbitrary $\mu$ and $\nu$, we give examples that conform to our experiments. Assume that $\mu,\nu \in \mathcal{P}^{\textrm{ac}}_{2}(\mathbb{R}^{d})$ are fully supported. Moreover, $\mathcal{Z}$ is convex, endowed with an absolutely continuous $\omega$ (e.g. Lebesgue). Then, due to Brenier's polar factorization \citep{brenier1991polar}, any transport map $T: \mathcal{Z} \rightarrow \mathcal{X}$ can be decomposed as $T = (\nabla \varphi) \circ s$ a.e., where $\varphi : \mathcal{Z} \rightarrow \mathbb{R}$ is convex and $s : \mathcal{Z} \rightarrow \mathcal{Z}$ is measure-preserving, both uniquely defined a.e. In fact, $(\nabla \varphi)$ is the unique optimal transport map between $\omega$ and $\mu$ under the Euclidean cost. Hence, we can construct $\phi_{X} \coloneqq [(\nabla \varphi) \circ s']^{-1}$, where $s'$ is also bijective (see, (\ref{diag})). Similarly, define $\phi'_{Y} \coloneqq (\nabla \varrho) \circ s''$, where $(\nabla \varrho)$ is the OT map between $\omega$ and $\nu$, and $s'' : \mathcal{Z} \rightarrow \mathcal{Z}$ preserves measure. As such, $F = (\nabla \varrho) \circ s'' \circ [(\nabla \varphi) \circ s']^{-1}$. One can similarly define $G$.

    The same can be extended to $\mathcal{Z}$ (Also, $\mathcal{X}$ and $\mathcal{Y}$) being a connected, compact, $C^{3}$-smooth Riemannian Manifold without boundary (\citet{mccann2001polar}, Theorem 11). Then, any volume-preserving transport $T$ is represented as $\textrm{exp}(-\nabla \varphi) \circ s$ a.e. Based on this, the rest of the construction follows exactly.

\section{Experiments and implementation details}
We refer to the repository {\small \url{https://anonymous.4open.science/r/RCDA/}} for all codes along with execution instructions. All experiments were carried out on an RTX 3090 GPU.
\subsection{Parameter selection in TGW and HGW} \label{param_sel}
As mentioned before, we select $\tau = \tilde{m} + 3\tilde{\sigma}$, where $\tilde{m}$ and $\tilde{\sigma}$ are respectively the median and the mean deviation about median of the deviation values $J_{X,Y} = \abs{d_{X}-d_{Y}}$. Observe that, for an univariate standard folded Normal random variable $Z$
\begin{align*}
    \mathbb{P}(Z \leq \tilde{m}) = 2\Phi(\tilde{m}) -1  = \frac{1}{2},
\end{align*}
where $\Phi(\cdot)$ is the distribution function of $N(0,1)$. As such, $\tilde{m} \approx 0.69$, the third quartile of $N(0,1)$. Now, given $a>0$,
\begin{align*}
    \tilde{\sigma}=\mathbb{E}\left[\:\abs{Z-a}\right] &= \int^{\infty}_{0} \abs{z-a} f(z) dz = \sqrt{\frac{2}{\pi}} \int^{\infty}_{0} \abs{z-a} e^{-\frac{z^{2}}{2}} dz \\ &= \sqrt{\frac{2}{\pi}} \Big(\int^{a}_{0} (a-z) e^{-\frac{z^{2}}{2}} dz + \int^{\infty}_{a} (z-a) e^{-\frac{z^{2}}{2}} dz\Big) \\ &= \sqrt{\frac{2}{\pi}}(2e^{-\frac{a^{2}}{2}}-1) + 4a\Phi(a) - 3a.
\end{align*}
As such at $a=0.69$, we have $\tilde{\sigma} \approx 0.46$. The corresponding estimate for $\tau$ turns out to be $\approx 2.07$, which coincides approximately with the $96$-percentile of $Z$. In our experiments, the deviation between pairwise distances under contamination does not follow such a law. Thus, calculating only a certain percentile becomes insufficient. Given a set of observations, we calculate the statistics independently and only use the percentiles as a reference.  
\textbf{}

\begin{figure}[H]
    \centering
    \begin{subfigure}{0.425\linewidth}
        \includegraphics[width=\linewidth]{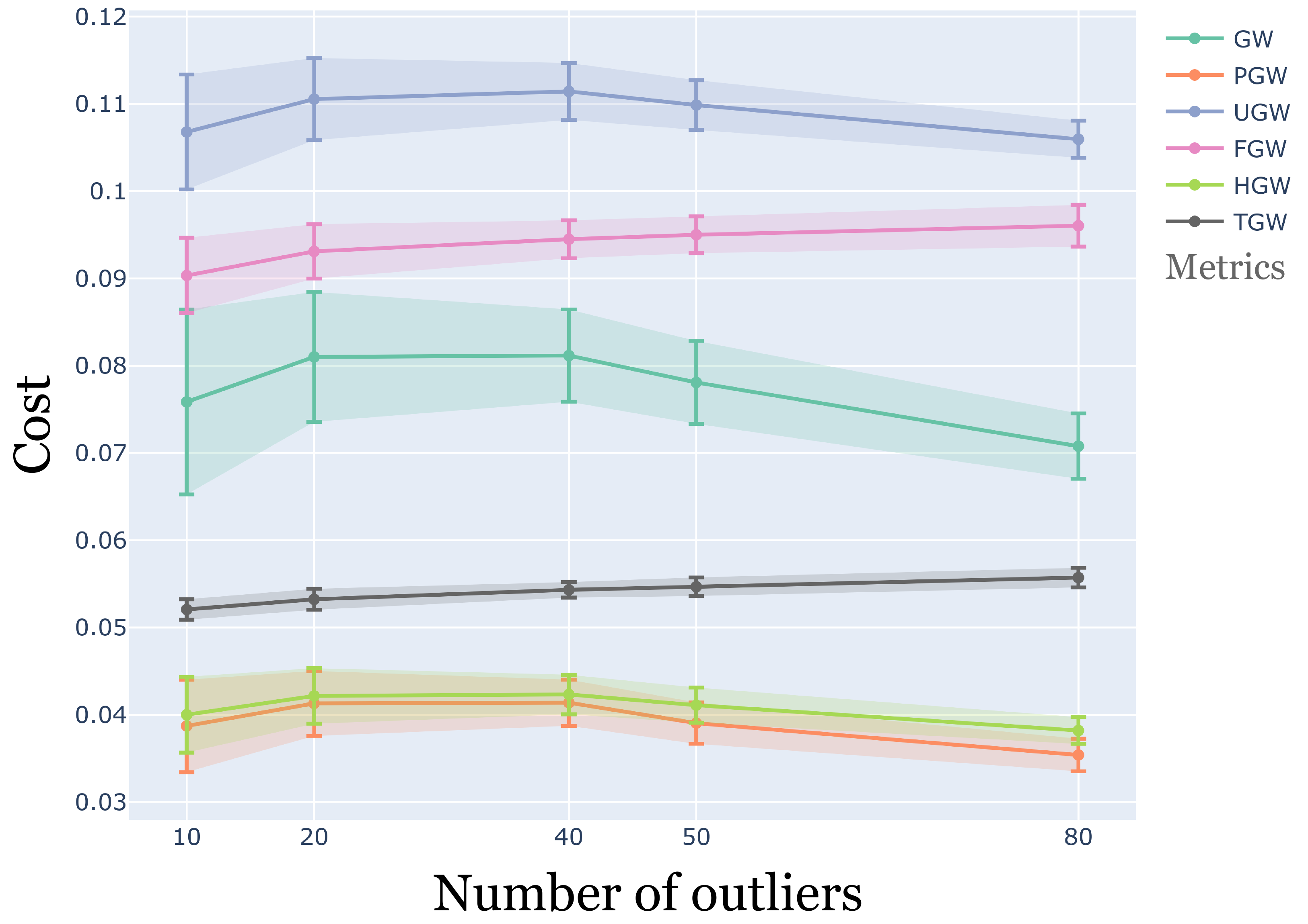}
        \caption{}
        \label{fig:gauss_plot}
    \end{subfigure}
    \begin{subfigure}{0.565\linewidth}
      \includegraphics[width=\linewidth]{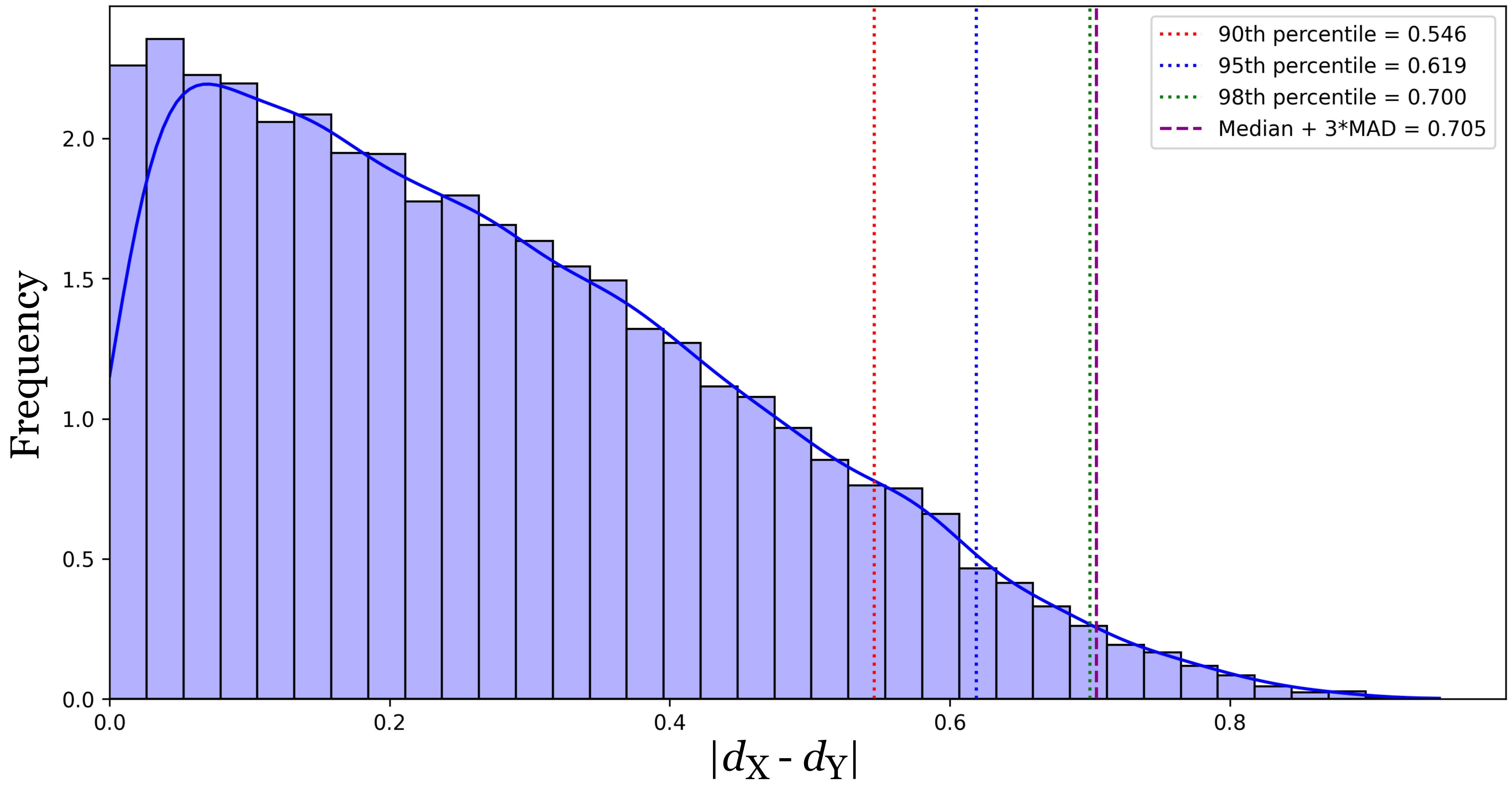}
      \caption{}
      \label{fig:gauss_dist}
    \end{subfigure}
    \caption{(a) Average losses under increasing proportion of bi-variate standard Gaussian outliers $(0.02,0.04,0.08,0.1,0.16)$ in source. (b) Empirical distribution of $J_{X,Y}$ under $80$ Gaussian outliers. Realized $95$-percentile and $\tilde{m} + 3\tilde{\sigma}$ are $0.619$ and $0.705$ respectively. TGW follows the $95$-percentile selection scheme while HGW is calculated based on $\tilde{m} + 3\tilde{\sigma}$.} 
    \label{fig:cat_heart_gauss}
\end{figure}

Performances of the competing methods alter under a contaminating distribution with a thinner tail. While Cauchy implants vastly outlying observations with higher frequency, Gaussian outliers flock to the immediate neighborhood (see, Figure \ref{fig:cat_heart_data}(b)). As a result, we observe a much more pronounced distorting of the shape instead of extremely large distortion values ($J_{X,Y}$). Moreover, as the number of outliers increases, only `moderately' high-valued distances ($d_{X}$) in the source increase (see, Figure \ref{fig:pairwise_cat}). This noising phenomenon is more difficult to eliminate using our thresholding scheme as `outlying' $d_{X}$ values are erroneously considered legitimate. As a result, the vanilla GW value, due to averaging, decreases even at elevated contamination levels. This is misleading since it does not reflect the distortion in local geometry. PGW and HGW perform well, and remarkably, TGW not only stays stable but also exhibits minute increments, indicating intensifying noise. The OT component's increase in FGW also demonstrates the same effect.  

\begin{figure}[H]
    \centering
    \begin{subfigure}{0.475\linewidth}
        \includegraphics[width=\linewidth]{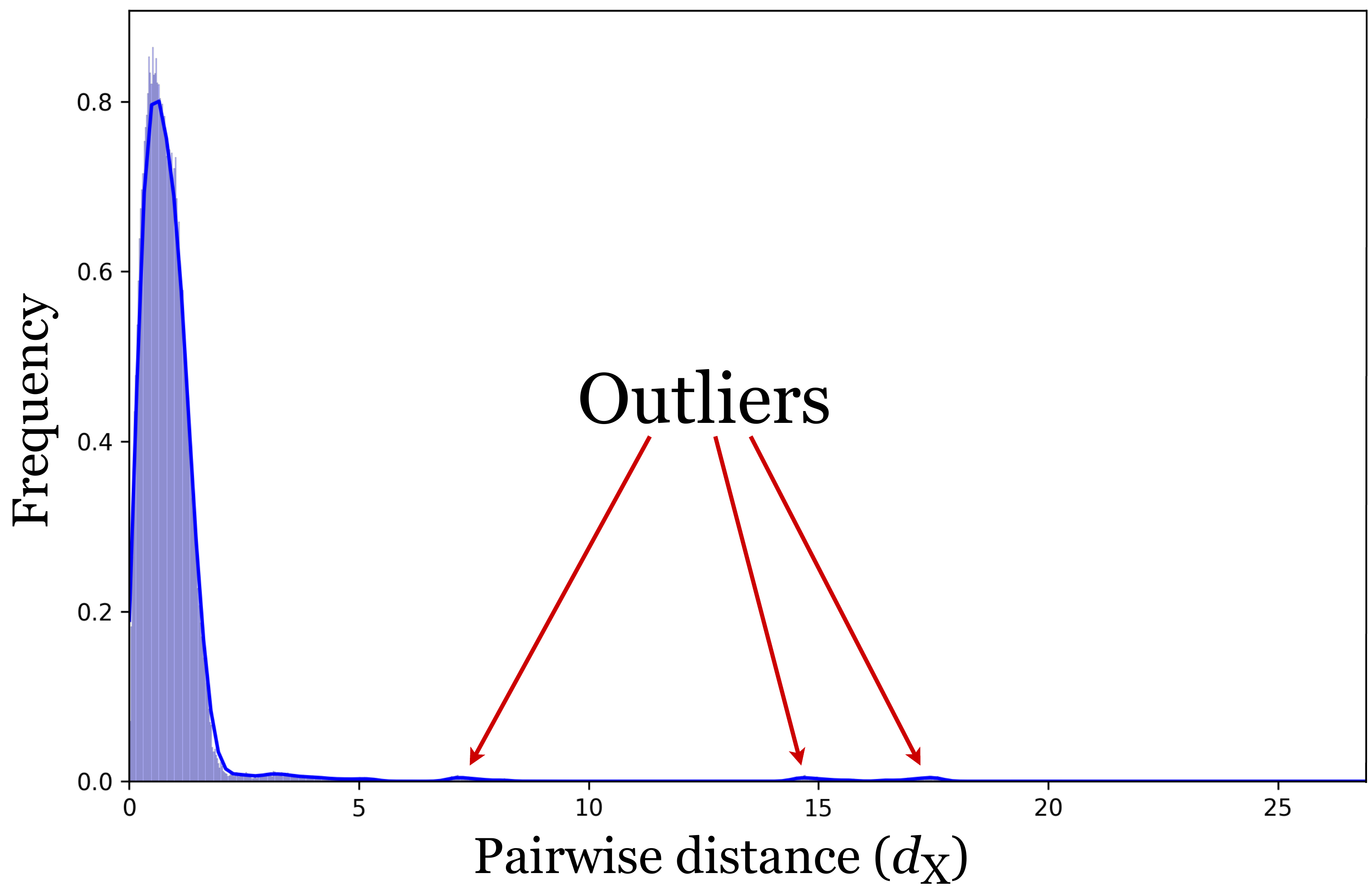}
        \caption{Cauchy noise}
        \label{fig:cauchy_pair}
    \end{subfigure}
    \hspace{2pt}
    \begin{subfigure}{0.48\linewidth}
      \includegraphics[width=\linewidth]{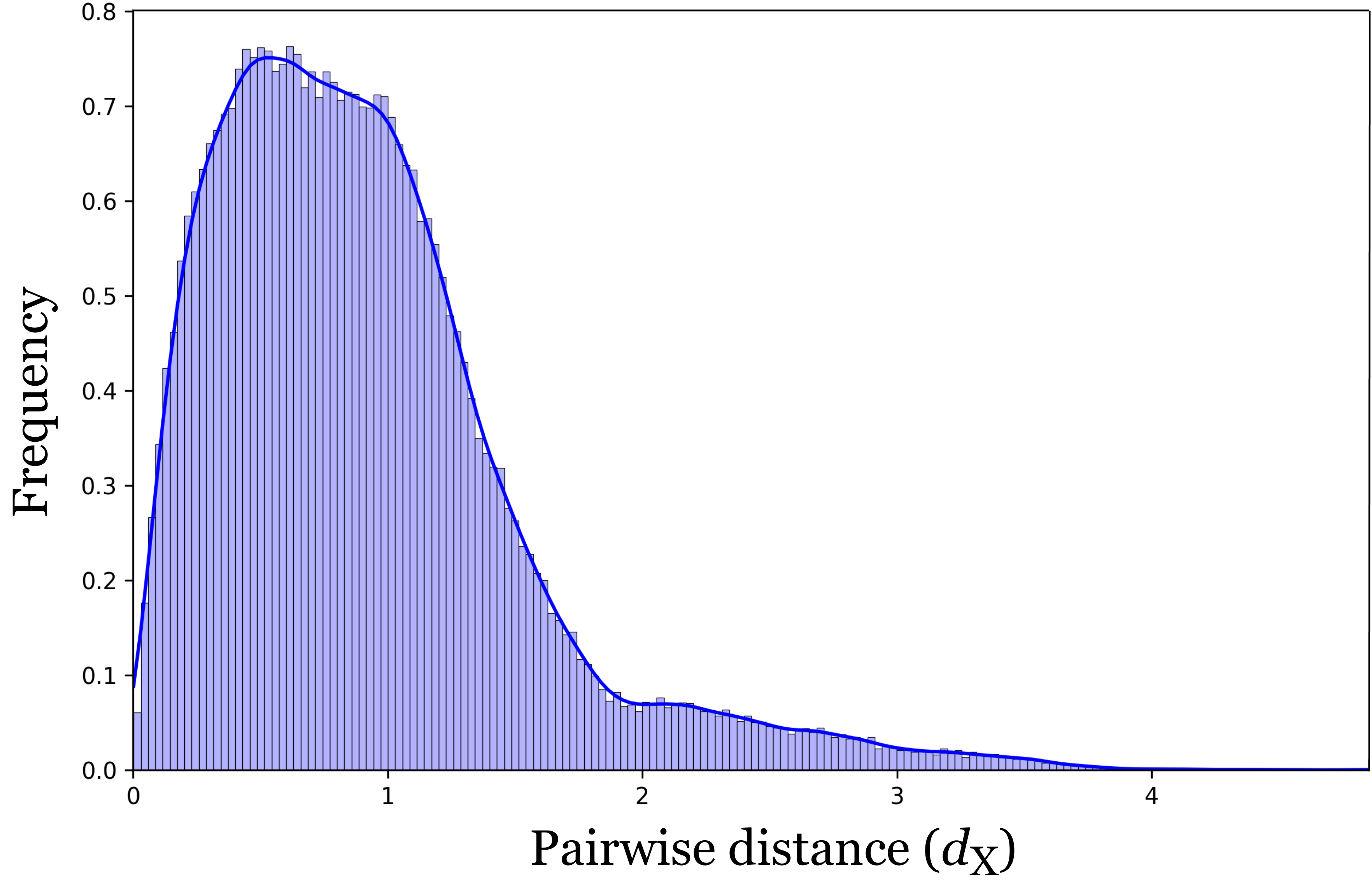}
      \caption{Gaussian noise}
      \label{fig:gauss_pair}
    \end{subfigure}
    \caption{(a) Empirical density of pairwise distances $d_{X}(x,x')$ in the source shape (cat) with 40 outliers.} 
    \label{fig:pairwise_cat}
\end{figure}

This motivates a rather finer thresholding technique we call local robustification. Instead of trimming extreme distortion values, we shift our focus to individual pairwise distances $d_{X}$ and $d_{Y}$. 

\subsection{LR translation using GcGAN and UNIT} \label{Ablation_LR}

For images, contamination regimes become more complex than point data. In case there are clear outliers ($n\epsilon$ out of $n$) from other image sources (e.g. MNIST samples in a pool of facial images \citep{nietert2023robust}), the discriminators still can distinguish between them. In contrast, if all images have noise injected in them in a predefined proportion ($\alpha$), the resultant generations get much more affected. This is mainly due to the discriminators misclassifying them. Observe that if $\alpha=1$, the noisy image from the second regime becomes an outlier from the first case. We maintain a flexible framework, striking a balance between the two. 

For the experiment under GcGAN, first, we create a copy of the original image tensor to ensure its data remains unchanged. Next, we generate standard Gaussian noise with the same shape as the tensor, scaling it by $\alpha = 0.2$ to control the noise intensity. Finally, we add this scaled noise to the copied tensor to produce a noisy version of the original.

In case of UNIT, we inject random bright pixels following a Gaussian law into the random images based on the image ratio. The wrapper supports datasets with or without labels by handling tuple or single image outputs, and seamlessly integrates into existing data loaders. \\

\textbf{Ablation study: Parameter selection}
Our experimental framework begins with the analysis of the generator and discriminator loss propagation (Figure \ref{fig:abl_RGc_loss}) under varying values of $\epsilon$. While smaller values imply weaker protection against noise, larger values tend to degrade discrimination performance (Figure \ref{fig:dis}). Coupled with the quantitative scrutiny, we introduce an additional experiment based on the qualitative outcomes as in Figure \ref{fig:abl_RGc}. Here, we observe increased meddling in background color and oversaturation as $\epsilon$ increases. On the other hand, a lower parameter value increases the likelihood of noise being manifested in the resulting images. Based on a trade-off between both examinations, we infer that $\epsilon=  0.5$ consistently demonstrates balanced performance during training, prompting us to an optimal value.
\begin{figure}[H]
    \centering
    \begin{subfigure}{0.49\linewidth}
        \includegraphics[width=\linewidth]{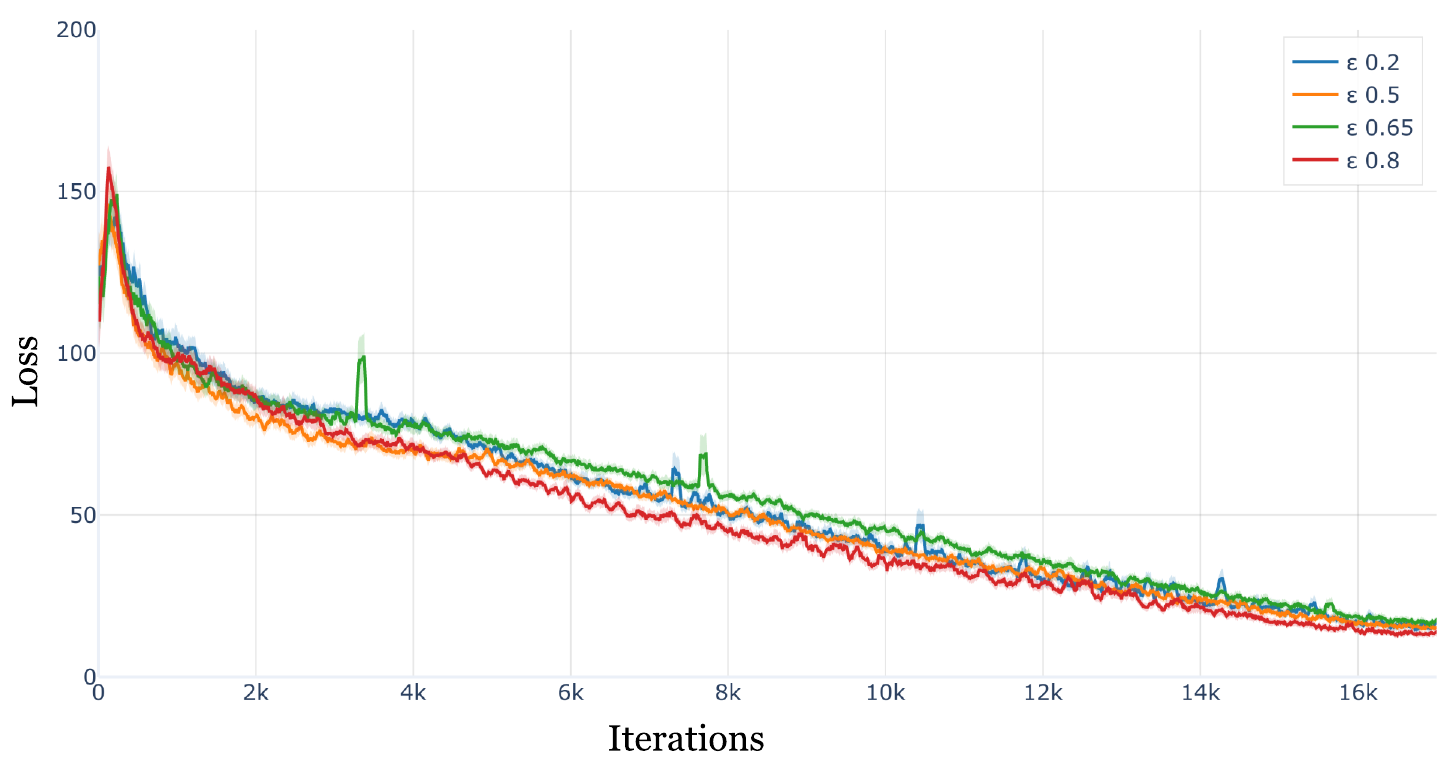}
        \caption{Convergence of loss}
        \label{fig:gen}
    \end{subfigure}
    \hspace{2pt}
    \begin{subfigure}{0.49\linewidth}
      \includegraphics[width=\linewidth]{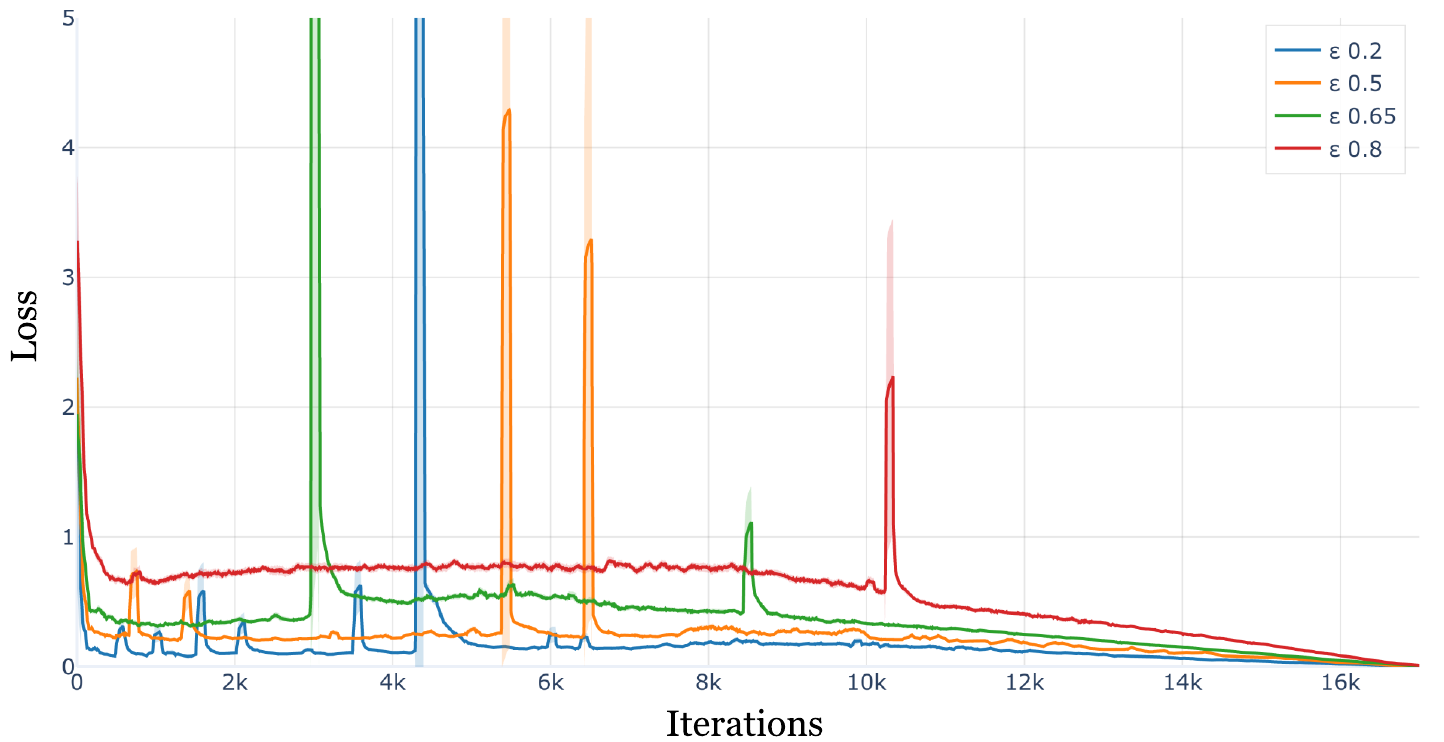}
      \caption{Discriminator performance}
      \label{fig:dis}
    \end{subfigure}
    \caption{(a) Realized robust GcGAN loss for varying $\epsilon$ under Gaussian noise ($\alpha = 0.2$). There is no perceptible difference between $\epsilon$ values in this regard. (b) The discriminators also eventually perform similarly.} 
    \label{fig:abl_RGc_loss}
\end{figure}

\begin{figure}[H]
    \centering
    \includegraphics[width=\linewidth]{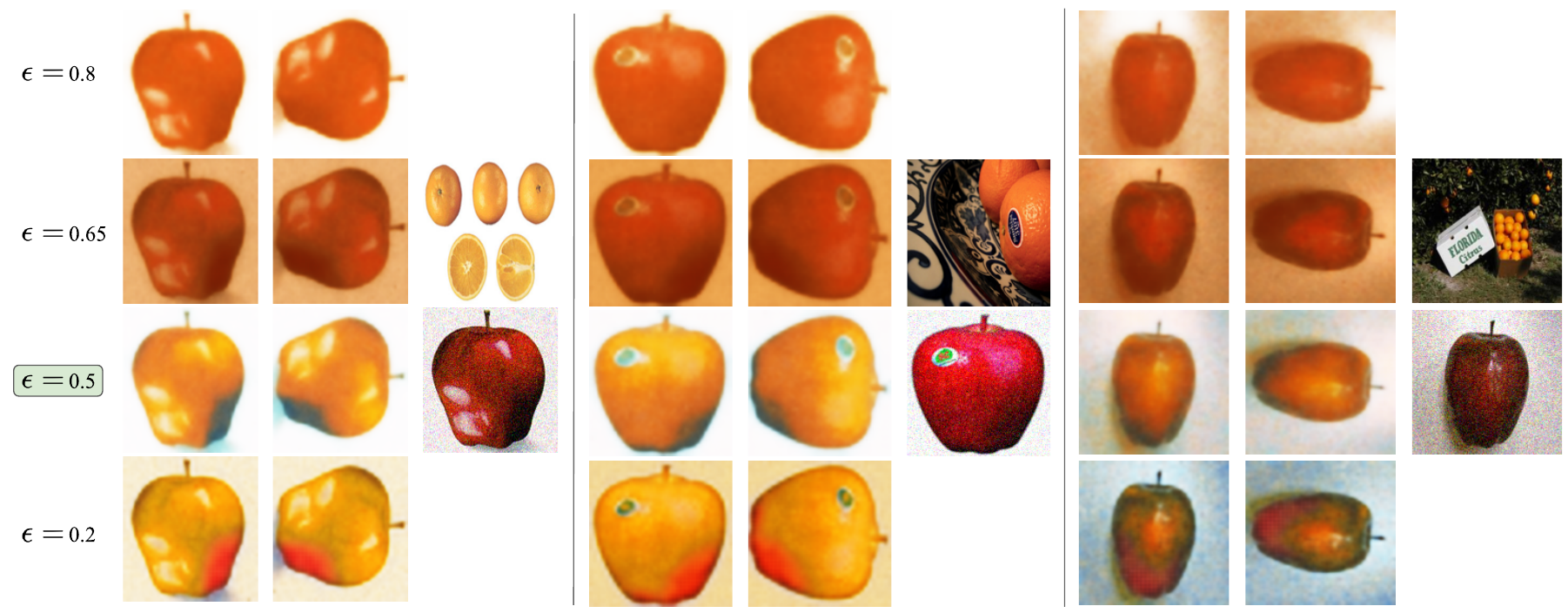}
    \caption{Style transfer performance of robust GcGAN for varying $\epsilon$. While small values ($\epsilon = 0.2$) produce inadequate denoising, high values ($\epsilon = 0.8$) distort the style and oversaturate images.}
    \label{fig:abl_RGc}
\end{figure}











\vskip 0.2in
\begin{thebibliography}{67}
\providecommand{\natexlab}[1]{#1}
\providecommand{\url}[1]{\texttt{#1}}
\expandafter\ifx\csname urlstyle\endcsname\relax
  \providecommand{\doi}[1]{doi: #1}\else
  \providecommand{\doi}{doi: \begingroup \urlstyle{rm}\Url}\fi

\bibitem[Alvarez-Esteban et~al.(2008)Alvarez-Esteban, Del~Barrio, Cuesta-Albertos, and Matran]{alvarez2008trimmed}
P.~C. Alvarez-Esteban, E.~Del~Barrio, J.~A. Cuesta-Albertos, and C.~Matran.
\newblock Trimmed comparison of distributions.
\newblock \emph{Journal of the American Statistical Association}, 103\penalty0 (482):\penalty0 697--704, 2008.

\bibitem[Arya et~al.(2024)Arya, Auddy, Clark, Lim, Memoli, and Packer]{arya2024gromov}
S.~Arya, A.~Auddy, R.~A. Clark, S.~Lim, F.~Memoli, and D.~Packer.
\newblock The gromov--wasserstein distance between spheres.
\newblock \emph{Foundations of Computational Mathematics}, pages 1--56, 2024.

\bibitem[Bai et~al.(2024)Bai, Martin, Kothapalli, Du, Liu, and Kolouri]{bai2024partialgromovwassersteinmetric}
Y.~Bai, R.~D. Martin, A.~Kothapalli, H.~Du, X.~Liu, and S.~Kolouri.
\newblock Partial gromov-wasserstein metric.
\newblock \emph{arXiv preprint arXiv:2402.03664}, 2024.

\bibitem[Balaji et~al.(2020)Balaji, Chellappa, and Feizi]{balaji2020robust}
Y.~Balaji, R.~Chellappa, and S.~Feizi.
\newblock Robust optimal transport with applications in generative modeling and domain adaptation.
\newblock \emph{Advances in Neural Information Processing Systems}, 33:\penalty0 12934--12944, 2020.

\bibitem[Bauer et~al.(2024)Bauer, M{\'e}moli, Needham, and Nishino]{bauer2024z}
M.~Bauer, F.~M{\'e}moli, T.~Needham, and M.~Nishino.
\newblock The z-gromov-wasserstein distance.
\newblock \emph{arXiv preprint arXiv:2408.08233}, 2024.

\bibitem[Benaim and Wolf(2017)]{benaim2017one}
S.~Benaim and L.~Wolf.
\newblock One-sided unsupervised domain mapping.
\newblock \emph{Advances in neural information processing systems}, 30, 2017.

\bibitem[Blanchet et~al.(2024)Blanchet, Jambulapati, Kent, and Sidford]{blanchet2024towards}
J.~Blanchet, A.~Jambulapati, C.~Kent, and A.~Sidford.
\newblock Towards optimal running timesfor optimal transport.
\newblock \emph{Operations Research Letters}, 52:\penalty0 107054, 2024.

\bibitem[Blumberg et~al.(2014)Blumberg, Gal, Mandell, and Pancia]{blumberg2014robust}
A.~J. Blumberg, I.~Gal, M.~A. Mandell, and M.~Pancia.
\newblock Robust statistics, hypothesis testing, and confidence intervals for persistent homology on metric measure spaces.
\newblock \emph{Foundations of Computational Mathematics}, 14:\penalty0 745--789, 2014.

\bibitem[Brenier(1991)]{brenier1991polar}
Y.~Brenier.
\newblock Polar factorization and monotone rearrangement of vector-valued functions.
\newblock \emph{Communications on pure and applied mathematics}, 44\penalty0 (4):\penalty0 375--417, 1991.

\bibitem[Chakrabarty and Das(2022)]{chakrabarty2022translation}
A.~Chakrabarty and S.~Das.
\newblock On translation and reconstruction guarantees of the cycle-consistent generative adversarial networks.
\newblock \emph{Advances in Neural Information Processing Systems}, 35:\penalty0 23607--23620, 2022.

\bibitem[Chakrabarty et~al.(2023)Chakrabarty, Basu, and Das]{chakrabarty2023concurrent}
A.~Chakrabarty, A.~Basu, and S.~Das.
\newblock Concurrent density estimation with wasserstein autoencoders: Some statistical insights.
\newblock \emph{arXiv preprint arXiv:2312.06591}, 2023.

\bibitem[Chapel et~al.(2020)Chapel, Alaya, and Gasso]{chapel2020partial}
L.~Chapel, M.~Z. Alaya, and G.~Gasso.
\newblock Partial optimal tranport with applications on positive-unlabeled learning.
\newblock \emph{Advances in Neural Information Processing Systems}, 33:\penalty0 2903--2913, 2020.

\bibitem[Choi et~al.(2018)Choi, Choi, Kim, Ha, Kim, and Choo]{Choi2018stargan}
Y.~Choi, M.~Choi, M.~Kim, J.-W. Ha, S.~Kim, and J.~Choo.
\newblock Stargan: Unified generative adversarial networks for multi-domain image-to-image translation.
\newblock In \emph{Proceedings of the IEEE Conference on Computer Vision and Pattern Recognition (CVPR)}, June 2018.

\bibitem[Chowdhury and Needham(2021)]{chowdhury2021generalized}
S.~Chowdhury and T.~Needham.
\newblock Generalized spectral clustering via gromov-wasserstein learning.
\newblock In \emph{International Conference on Artificial Intelligence and Statistics}, pages 712--720. PMLR, 2021.

\bibitem[Clark et~al.(2024)Clark, Needham, and Weighill]{clark2024generalized}
R.~A. Clark, T.~Needham, and T.~Weighill.
\newblock Generalized dimension reduction using semi-relaxed gromov-wasserstein distance.
\newblock \emph{arXiv preprint arXiv:2405.15959}, 2024.

\bibitem[Clarkson et~al.(2019)Clarkson, Wang, and Woodruff]{clarkson2019dimensionality}
K.~Clarkson, R.~Wang, and D.~Woodruff.
\newblock Dimensionality reduction for tukey regression.
\newblock In \emph{International Conference on Machine Learning}, pages 1262--1271. PMLR, 2019.

\bibitem[Clarkson and Woodruff(2014)]{clarkson2014sketching}
K.~L. Clarkson and D.~P. Woodruff.
\newblock Sketching for m-estimators: A unified approach to robust regression.
\newblock In \emph{Proceedings of the twenty-sixth annual ACM-SIAM symposium on Discrete algorithms}, pages 921--939. SIAM, 2014.

\bibitem[Cuturi(2013)]{cuturi2013sinkhorn}
M.~Cuturi.
\newblock Sinkhorn distances: Lightspeed computation of optimal transport.
\newblock \emph{Advances in neural information processing systems}, 26, 2013.

\bibitem[Czado and Munk(1998)]{czado1998assessing}
C.~Czado and A.~Munk.
\newblock Assessing the similarity of distributions-finite sample performance of the empirical mallows distance.
\newblock \emph{Journal of Statistical Computation and Simulation}, 60\penalty0 (4):\penalty0 319--346, 1998.

\bibitem[De~Ponti and Mondino(2022)]{de2022entropy}
N.~De~Ponti and A.~Mondino.
\newblock Entropy-transport distances between unbalanced metric measure spaces.
\newblock \emph{Probability Theory and Related Fields}, 184\penalty0 (1):\penalty0 159--208, 2022.

\bibitem[Demetci et~al.(2022)Demetci, Santorella, Sandstede, Noble, and Singh]{demetci2022scot}
P.~Demetci, R.~Santorella, B.~Sandstede, W.~S. Noble, and R.~Singh.
\newblock Scot: Single-cell multi-omics alignment with optimal transport.
\newblock \emph{Journal of Computational Biology}, 29\penalty0 (1):\penalty0 3--18, 2022.

\bibitem[Dowson and Landau(1982)]{dowson1982frechet}
D.~Dowson and B.~Landau.
\newblock The fr{\'e}chet distance between multivariate normal distributions.
\newblock \emph{Journal of multivariate analysis}, 12\penalty0 (3):\penalty0 450--455, 1982.

\bibitem[Dumont et~al.(2024)Dumont, Lacombe, and Vialard]{dumont2024existence}
T.~Dumont, T.~Lacombe, and F.-X. Vialard.
\newblock On the existence of monge maps for the gromov--wasserstein problem.
\newblock \emph{Foundations of Computational Mathematics}, pages 1--48, 2024.

\bibitem[Fristedt and Gray(2013)]{fristedt2013modern}
B.~E. Fristedt and L.~F. Gray.
\newblock \emph{A modern approach to probability theory}.
\newblock Springer Science \& Business Media, 2013.

\bibitem[Fu et~al.(2019)Fu, Gong, Wang, Batmanghelich, Zhang, and Tao]{fu2019geometry}
H.~Fu, M.~Gong, C.~Wang, K.~Batmanghelich, K.~Zhang, and D.~Tao.
\newblock Geometry-consistent generative adversarial networks for one-sided unsupervised domain mapping.
\newblock In \emph{Proceedings of the IEEE/CVF conference on computer vision and pattern recognition}, pages 2427--2436, 2019.

\bibitem[Gong et~al.(2022)Gong, Nie, and Xu]{gong2022gromov}
F.~Gong, Y.~Nie, and H.~Xu.
\newblock Gromov-wasserstein multi-modal alignment and clustering.
\newblock In \emph{Proceedings of the 31st ACM International Conference on Information \& Knowledge Management}, pages 603--613, 2022.

\bibitem[Greven et~al.(2009)Greven, Pfaffelhuber, and Winter]{greven2009convergence}
A.~Greven, P.~Pfaffelhuber, and A.~Winter.
\newblock Convergence in distribution of random metric measure spaces ($\lambda$-coalescent measure trees).
\newblock \emph{Probability Theory and Related Fields}, 145\penalty0 (1):\penalty0 285--322, 2009.

\bibitem[Groppe and Hundrieser(2023)]{groppe2023lower}
M.~Groppe and S.~Hundrieser.
\newblock Lower complexity adaptation for empirical entropic optimal transport.
\newblock \emph{arXiv preprint arXiv:2306.13580}, 2023.

\bibitem[Gulrajani et~al.(2017)Gulrajani, Ahmed, Arjovsky, Dumoulin, and Courville]{gulrajani2017improved}
I.~Gulrajani, F.~Ahmed, M.~Arjovsky, V.~Dumoulin, and A.~C. Courville.
\newblock Improved training of wasserstein gans.
\newblock \emph{Advances in neural information processing systems}, 30, 2017.

\bibitem[Huang et~al.(2018)Huang, Liu, Belongie, and Kautz]{huang2018multimodal}
X.~Huang, M.-Y. Liu, S.~Belongie, and J.~Kautz.
\newblock Multimodal unsupervised image-to-image translation.
\newblock In \emph{Proceedings of the European conference on computer vision (ECCV)}, pages 172--189, 2018.

\bibitem[Huber(1981)]{huber1981robust}
P.~Huber.
\newblock \emph{Robust Statistics}.
\newblock Wiley Series in Probability and Statistics. Wiley, 1981.

\bibitem[Huber(1964)]{huber1964robust}
P.~J. Huber.
\newblock {Robust Estimation of a Location Parameter}.
\newblock \emph{The Annals of Mathematical Statistics}, 35\penalty0 (1):\penalty0 73 -- 101, 1964.

\bibitem[Hur et~al.(2024)Hur, Guo, and Liang]{hur2024reversible}
Y.~Hur, W.~Guo, and T.~Liang.
\newblock Reversible gromov--monge sampler for simulation-based inference.
\newblock \emph{SIAM Journal on Mathematics of Data Science}, 6\penalty0 (2):\penalty0 283--310, 2024.

\bibitem[Kong et~al.(2024)Kong, Li, Tang, and So]{kong2024outlier}
L.~Kong, J.~Li, J.~Tang, and A.~M.-C. So.
\newblock Outlier-robust gromov-wasserstein for graph data.
\newblock \emph{Advances in Neural Information Processing Systems}, 36, 2024.

\bibitem[Le et~al.(2021)Le, Nguyen, Nguyen, Pham, Bui, and Ho]{le2021robust}
K.~Le, H.~Nguyen, Q.~M. Nguyen, T.~Pham, H.~Bui, and N.~Ho.
\newblock On robust optimal transport: Computational complexity and barycenter computation.
\newblock \emph{Advances in Neural Information Processing Systems}, 34:\penalty0 21947--21959, 2021.

\bibitem[Lecu{\'e} and Lerasle(2020)]{lecue2020mom}
G.~Lecu{\'e} and M.~Lerasle.
\newblock {Robust machine learning by median-of-means: Theory and practice}.
\newblock \emph{The Annals of Statistics}, 48\penalty0 (2):\penalty0 906 -- 931, 2020.

\bibitem[Li et~al.(2023)Li, Yu, Xu, and Meng]{li2023efficient}
M.~Li, J.~Yu, H.~Xu, and C.~Meng.
\newblock Efficient approximation of gromov-wasserstein distance using importance sparsification.
\newblock \emph{Journal of Computational and Graphical Statistics}, 32\penalty0 (4):\penalty0 1512--1523, 2023.

\bibitem[Liu et~al.(2017)Liu, Breuel, and Kautz]{liu2017unsupervised}
M.-Y. Liu, T.~Breuel, and J.~Kautz.
\newblock Unsupervised image-to-image translation networks.
\newblock \emph{Advances in neural information processing systems}, 30, 2017.

\bibitem[Loh(2017)]{loh2017robust}
P.-L. Loh.
\newblock {Statistical consistency and asymptotic normality for high-dimensional robust $M$-estimators}.
\newblock \emph{The Annals of Statistics}, 45\penalty0 (2):\penalty0 866 -- 896, 2017.

\bibitem[Ma et~al.(2023)Ma, Liu, La~Vecchia, and Lerasle]{ma2023inference}
Y.~Ma, H.~Liu, D.~La~Vecchia, and M.~Lerasle.
\newblock Inference via robust optimal transportation: theory and methods.
\newblock \emph{arXiv preprint arXiv:2301.06297}, 2023.

\bibitem[McCann(2001)]{mccann2001polar}
R.~J. McCann.
\newblock Polar factorization of maps on riemannian manifolds.
\newblock \emph{Geometric \& Functional Analysis GAFA}, 11\penalty0 (3):\penalty0 589--608, 2001.

\bibitem[M{\'e}moli(2011)]{memoli2011gromov}
F.~M{\'e}moli.
\newblock Gromov--wasserstein distances and the metric approach to object matching.
\newblock \emph{Foundations of computational mathematics}, 11:\penalty0 417--487, 2011.

\bibitem[Memoli et~al.(2019)Memoli, Smith, and Wan]{memoli2019wasserstein}
F.~Memoli, Z.~Smith, and Z.~Wan.
\newblock The wasserstein transform.
\newblock In \emph{International Conference on Machine Learning}, pages 4496--4504. PMLR, 2019.

\bibitem[Mena and Niles-Weed(2019)]{mena2019statistical}
G.~Mena and J.~Niles-Weed.
\newblock Statistical bounds for entropic optimal transport: sample complexity and the central limit theorem.
\newblock \emph{Advances in neural information processing systems}, 32, 2019.

\bibitem[Mroueh and Rigotti(2020)]{mroueh2020unbalanced}
Y.~Mroueh and M.~Rigotti.
\newblock Unbalanced sobolev descent.
\newblock \emph{Advances in Neural Information Processing Systems}, 33:\penalty0 17034--17043, 2020.

\bibitem[Mukherjee et~al.(2021)Mukherjee, Guha, Solomon, Sun, and Yurochkin]{mukherjee2021outlier}
D.~Mukherjee, A.~Guha, J.~M. Solomon, Y.~Sun, and M.~Yurochkin.
\newblock Outlier-robust optimal transport.
\newblock In \emph{International Conference on Machine Learning}, pages 7850--7860. PMLR, 2021.

\bibitem[Musco et~al.(2021)Musco, Musco, Woodruff, and Yasuda]{musco2021active}
C.~Musco, C.~Musco, D.~P. Woodruff, and T.~Yasuda.
\newblock Active linear regression for $l_p$ norms and beyond.
\newblock \emph{arXiv preprint arXiv:2111.04888}, 2021.

\bibitem[Nguyen et~al.(2023)Nguyen, Nguyen, Zhou, and Nguyen]{nguyen2023unbalanced}
Q.~M. Nguyen, H.~H. Nguyen, Y.~Zhou, and L.~M. Nguyen.
\newblock On unbalanced optimal transport: Gradient methods, sparsity and approximation error.
\newblock \emph{The Journal of Machine Learning Research}, 24\penalty0 (1):\penalty0 18390--18430, 2023.

\bibitem[Nietert et~al.(2022)Nietert, Goldfeld, and Cummings]{nietert2022outlier}
S.~Nietert, Z.~Goldfeld, and R.~Cummings.
\newblock Outlier-robust optimal transport: Duality, structure, and statistical analysis.
\newblock In \emph{International Conference on Artificial Intelligence and Statistics}, pages 11691--11719. PMLR, 2022.

\bibitem[Nietert et~al.(2023)Nietert, Cummings, and Goldfeld]{nietert2023robust}
S.~Nietert, R.~Cummings, and Z.~Goldfeld.
\newblock Robust estimation under the wasserstein distance.
\newblock \emph{arXiv preprint arXiv:2302.01237}, 2023.

\bibitem[Pele and Werman(2009)]{pele2009fast}
O.~Pele and M.~Werman.
\newblock Fast and robust earth mover's distances.
\newblock In \emph{2009 IEEE 12th international conference on computer vision}, pages 460--467. IEEE, 2009.

\bibitem[Peyr{\'e} et~al.(2016)Peyr{\'e}, Cuturi, and Solomon]{peyre2016gromov}
G.~Peyr{\'e}, M.~Cuturi, and J.~Solomon.
\newblock Gromov-wasserstein averaging of kernel and distance matrices.
\newblock In \emph{International conference on machine learning}, pages 2664--2672. PMLR, 2016.

\bibitem[Raghvendra et~al.(2024)Raghvendra, Shirzadian, and Zhang]{raghvendra2024robpar}
S.~Raghvendra, P.~Shirzadian, and K.~Zhang.
\newblock A new robust partial p-{W}asserstein-based metric for comparing distributions.
\newblock In \emph{Proceedings of the 41st International Conference on Machine Learning}, volume 235 of \emph{Proceedings of Machine Learning Research}, pages 41867--41885. PMLR, 21--27 Jul 2024.

\bibitem[Rioux et~al.(2023)Rioux, Goldfeld, and Kato]{rioux2023entropic}
G.~Rioux, Z.~Goldfeld, and K.~Kato.
\newblock Entropic gromov-wasserstein distances: Stability and algorithms.
\newblock \emph{arXiv preprint arXiv:2306.00182}, 2023.

\bibitem[Salmona et~al.(2024)Salmona, Desolneux, and Delon]{salmona2024gromovwassersteinlike}
A.~Salmona, A.~Desolneux, and J.~Delon.
\newblock Gromov-wasserstein-like distances in the gaussian mixture models space.
\newblock \emph{Transactions on Machine Learning Research}, 2024.
\newblock ISSN 2835-8856.

\bibitem[Scetbon et~al.(2022)Scetbon, Peyr{\'e}, and Cuturi]{scetbon2022linear}
M.~Scetbon, G.~Peyr{\'e}, and M.~Cuturi.
\newblock Linear-time gromov wasserstein distances using low rank couplings and costs.
\newblock In \emph{International Conference on Machine Learning}, pages 19347--19365. PMLR, 2022.

\bibitem[Schweizer et~al.(1960)Schweizer, Sklar, et~al.]{schweizer1960statistical}
B.~Schweizer, A.~Sklar, et~al.
\newblock Statistical metric spaces.
\newblock \emph{Pacific J. Math}, 10\penalty0 (1):\penalty0 313--334, 1960.

\bibitem[S{\'e}journ{\'e} et~al.(2021)S{\'e}journ{\'e}, Vialard, and Peyr{\'e}]{sejourne2021unbalanced}
T.~S{\'e}journ{\'e}, F.-X. Vialard, and G.~Peyr{\'e}.
\newblock The unbalanced gromov-wasserstein distance: Conic formulation and relaxation.
\newblock \emph{Advances in Neural Information Processing Systems}, 34:\penalty0 8766--8779, 2021.

\bibitem[Sturm(2006)]{sturm2006geometry}
K.-T. Sturm.
\newblock On the geometry of metric measure spaces.
\newblock \emph{Acta Mathematica}, 196\penalty0 (1):\penalty0 65 -- 131, 2006.

\bibitem[Sturm(2023)]{sturm2023space}
K.-T. Sturm.
\newblock \emph{The space of spaces: curvature bounds and gradient flows on the space of metric measure spaces}, volume 290.
\newblock American Mathematical Society, 2023.

\bibitem[Tran et~al.(2023)Tran, Janati, Courty, Flamary, Redko, Demetci, and Singh]{tran2023unbalanced}
Q.~H. Tran, H.~Janati, N.~Courty, R.~Flamary, I.~Redko, P.~Demetci, and R.~Singh.
\newblock Unbalanced co-optimal transport.
\newblock In \emph{Proceedings of the AAAI Conference on Artificial Intelligence}, volume~37, pages 10006--10016, 2023.

\bibitem[Vayer et~al.(2020)Vayer, Chapel, Flamary, Tavenard, and Courty]{vayer2020fused}
T.~Vayer, L.~Chapel, R.~Flamary, R.~Tavenard, and N.~Courty.
\newblock Fused gromov-wasserstein distance for structured objects.
\newblock \emph{Algorithms}, 13\penalty0 (9):\penalty0 212, 2020.

\bibitem[Villani(2021)]{villani2021topics}
C.~Villani.
\newblock \emph{Topics in optimal transportation}, volume~58.
\newblock American Mathematical Soc., 2021.

\bibitem[Villani et~al.(2009)]{villani2009optimal}
C.~Villani et~al.
\newblock \emph{Optimal transport: old and new}, volume 338.
\newblock Springer, 2009.

\bibitem[Zhang et~al.(2022)Zhang, Mroueh, Goldfeld, and Sriperumbudur]{zhang2022cycle}
Z.~Zhang, Y.~Mroueh, Z.~Goldfeld, and B.~Sriperumbudur.
\newblock Cycle consistent probability divergences across different spaces.
\newblock In \emph{International Conference on Artificial Intelligence and Statistics}, pages 7257--7285. PMLR, 2022.

\bibitem[Zhang et~al.(2024)Zhang, Goldfeld, Mroueh, and Sriperumbudur]{zhang2024gromov}
Z.~Zhang, Z.~Goldfeld, Y.~Mroueh, and B.~K. Sriperumbudur.
\newblock Gromov--wasserstein distances: Entropic regularization, duality and sample complexity.
\newblock \emph{The Annals of Statistics}, 52\penalty0 (4):\penalty0 1616--1645, 2024.

\bibitem[Zhu et~al.(2017)Zhu, Park, Isola, and Efros]{zhu2017unpaired}
J.-Y. Zhu, T.~Park, P.~Isola, and A.~A. Efros.
\newblock Unpaired image-to-image translation using cycle-consistent adversarial networks.
\newblock In \emph{Proceedings of the IEEE international conference on computer vision}, pages 2223--2232, 2017.

\end{thebibliography}
\end{document}